\documentclass{article}
\usepackage[utf8]{inputenc} 
\usepackage[T1]{fontenc}    
\usepackage{amsmath}
\usepackage{amsthm}
\usepackage{amsfonts}
\usepackage{amssymb}
\usepackage{latexsym} 
\usepackage{mathtools}
\usepackage{graphicx}
\usepackage{mathrsfs}
\usepackage{url}
\usepackage{booktabs}
\usepackage{threeparttable}
\usepackage{caption}
\usepackage{subfigure}
\usepackage{float}
\usepackage{array}
\usepackage{multirow}
\usepackage{appendix}
\usepackage{hyperref}
\usepackage{enumerate}
\usepackage{authblk}

\usepackage[margin = 1in]{geometry}
\usepackage{nicefrac}       
\usepackage{microtype}    

\usepackage{xcolor}
\usepackage{tikz}

\usepackage{makecell}
\usepackage{diagbox}

\usepackage{lipsum,etoolbox}

\usepackage[inline]{enumitem}
\makeatletter
\newcommand{\myitem}[1][]{%
	\item[#1]\protected@edef\@currentlabel{#1}\ignorespaces%
}
\makeatother
\usepackage{crossreftools}
\usepackage{calc}
\SetEnumitemKey{widestlabel}
{labelwidth = \widthof{\ref{enum-\EnumitemId}},
	after = \label{enum-\EnumitemId}}

\newcommand*\circled[1]{\tikz[baseline=(char.base)]{
    \node[shape=circle, draw, inner sep=1pt, 
        minimum height=12pt] (char) {\vphantom{1g}#1};}}

\newtheorem{theorem}{Theorem}[section]

\newtheorem{proposition}[theorem]{Proposition}
\newtheorem{lemma}[theorem]{Lemma}

\newtheorem{definition}[theorem]{Definition}

\newtheorem{assumption}{Assumption}

\theoremstyle{remark}
\newtheorem{remark}{Remark}[section]

\newtheorem{exmp}[theorem]{Example}

\newcommand{\E}{\mathbb{E}}

\newcommand{\R}{\mathbb{R}}

\newcommand{\sM}{\mathcal{M}}
\newcommand{\sL}{\mathcal{L}}
\newcommand{\calL}{\mathcal{L}}
\newcommand{\brho}{\bar{\rho}}
\newcommand{\hrho}{\hat{\rho}}
\newcommand{\trho}{\tilde{\rho}}

\newcommand{\un}{{\rm un}}
\newcommand{\rw}{{\rm rw}}

\newcommand{\diag}{\text{diag}}
\newcommand{\var}{\text{var}}
\newcommand{\Vol }{\text{Vol}}
\newcommand{\Hess}{\text{Hess}}

\newcommand{\rd}{\mathrm{d}}

\newcommand{\calM}{\mathcal{M}}

\usepackage[textsize=tiny]{todonotes}

\begin{document}

\title{Improved convergence rate of $k$NN graph Laplacians:\\
differentiable self-tuned affinity}

\author[1]{Xiuyuan Cheng\thanks{Email: xiuyuan.cheng@duke.edu. Authors listed alphabetically.}}
\author[1]{Yixuan Tan}
\author[2]{Nan Wu}

\affil[1]{{\small Department of Mathematics, Duke University}}
\affil[2]{\small Department of Mathematical Sciences, The University of Texas at Dallas}

\date{\vspace{-30pt}}

\maketitle

\begin{abstract}
In graph-based data analysis, $k$-nearest neighbor ($k$NN) graphs are widely used due to their adaptivity to local data densities. Allowing weighted edges in the graph, the kernelized graph affinity provides a more general type of $k$NN graph where the $k$NN distance is used to set the kernel bandwidth adaptively. In this work, we consider a general class of $k$NN graph where the graph affinity is $W_{ij} = \epsilon^{-d/2} k_0 ( \| x_i - x_j \|^2 / \epsilon \phi(  \hrho(x_i),  \hrho(x_j) )^2 ) $, with $\hat{\rho}(x)$ being the (rescaled) $k$NN distance at the point $x$, $\phi$ a symmetric bi-variate function, and $k_0$ a non-negative function on $[0,\infty)$. Under the manifold data setting, where $N$ i.i.d. samples $x_i$ are drawn from a density $p$ on a $d$-dimensional unknown manifold embedded in a high dimensional Euclidean space, we prove the operator pointwise convergence of the $k$NN graph Laplacian to the limiting manifold operator (depending on $p$) at the rate of $O(N^{-2/(d+6)})$, up to a log factor, when $k_0$ and $\phi$ have $C^3$ regularity and satisfy other technical conditions. This is obtained when $\epsilon \sim N^{-2/(d+6)}$ and $k \sim N^{6/(d+6)}$, both at the optimal order to balance the theoretical bias and variance errors. Our improved convergence rate is based on a refined analysis of the $k$NN estimator, which can be of independent interest. We validate our theory by numerical experiments on simulated data.
\end{abstract}

\section{Introduction}

The graph Laplacian has been an essential tool in data analysis, and the applications include dimension reduction \cite{tenenbaum2000global, belkin2003laplacian, coifman2006diffusion, talmon2013diffusion}, clustering analysis \cite{ng2001spectral, von2007tutorial}, and semi-supervised learning \cite{belkin2002semi, zhu2003semi, el2016asymptotic, slepcev2019analysis, flores2022analysis}, among others.
In this work, we consider graph Laplacians constructed from high-dimensional data vectors, and we further assume that they lie on an unknown low-dimensional manifold. 
Specifically, from a set of i.i.d. samples $X = \{x_i\}_{i=1}^N$ in $\R^m$, 
one constructs a (kernelized) {\it graph affinity} matrix as $W_{ij} = K_\epsilon(x_i, x_j)$, where $K_\epsilon(x,y)$ is a kernel function typically parametrized by a bandwidth parameter $\epsilon > 0$.
For example, $W_{ij} = k_0(\|x_i - x_j\|^2 / \epsilon)$ for some function $k_0$ on  $[0, \infty)$, and this is called a ``fixed-bandwidth'' kernelized affinity. When $k_0$ is the indicator function $\mathbf{1}_{[0,1]}$, $W$ will assign a weight one to an edge between $x_i$ and $x_j$ when $\|x_i - x_j\| \le \sqrt{\epsilon}$, and zero otherwise. 

Fixed-bandwidth graph affinity encounters challenges in practice when data samples are distributed unevenly, 
which results in smaller local mutual distances at some samples (in dense regions) and larger distances at others (in sparse regions). By using a universal bandwidth $\epsilon$, the points in dense regions may be over-connected while those in sparse regions are left with no neighbors. This often results in instability (large variance error) in the performance of graph Laplacian methods. 

To overcome the issue of fixed-bandwidth kernelized graph affinity, the construction of {\it adaptive} bandwidth has been developed and those based on $k$-nearest neighbor ($k$NN) distances are most often used. In a classical $k$NN graph, a point $x_j$ is connected to $x_i$ (i.e. the affinity $W_{ij}=1$) if it is within the $k$NN of $x_i$, and otherwise there is no edge ($W_{ij}=0$).
This construction is asymmetric with respect to $i$ and $j$, and one way to enforce symmetric $W$ is by letting 
\begin{equation}\label{eq:def-kNN-graph}
	W_{ij} = \mathbf{1}_{[0,1]} \left(  \frac{\| x_i - x_j \|^2}{ \max\{ \hat{R}(x_i),  \hat{R}(x_j)\}^2   } \right), \quad \text{or} \quad W_{ij} = \mathbf{1}_{[0,1]} \left(  \frac{\| x_i - x_j \|^2}{ \min\{\hat{R}(x_i),  \hat{R}(x_j)\}^2   } \right),
\end{equation}
where $\hat{R}(x)$ denotes the distance between $x$ and its $k$NN within the given dataset $X$. 
The convergence rate of graph Laplacians using $W$ as in \eqref{eq:def-kNN-graph} has been theoretically analyzed in \cite{calder2022improved} on manifold data.
We call the model \eqref{eq:def-kNN-graph} the ``classical $k$NN'' graph affinity. 

Another way to enforce symmetry is by multiplying $\hat{R}(x_i)$ and $\hat{R}(x_j)$ instead of taking the max/min of the two. Specifically, for some function $k_0$, e.g., $k_0(\eta) = \exp(-\eta)$, 
\begin{equation}\label{eq:def-self-tune-graph}
	W_{ij} = k_0 \left(   \frac{\| x_i - x_j \|^2}{  \hat{R}(x_i) \hat{R}(x_j) } \right).
\end{equation}
This model was utilized in \cite{zelnik2004self} for spectral clustering and called ``self-tuned'',
and the graph Laplacian convergence was previously analyzed in \cite{cheng2022convergence}.

In this work, we consider a more general class of $k$NN graph Laplacians which includes the classical $k$NN affinity \eqref{eq:def-kNN-graph} and the self-tuned affinity \eqref{eq:def-self-tune-graph} as special cases. Technically, we let
\begin{equation}\label{eq:def-W}
W_{ij} = \epsilon^{-d/2} k_0 \left(  \frac{\| x_i - x_j \|^2}{ \epsilon \phi(\hrho(x_i), \hrho(x_j))^2 } \right),
\end{equation}
where $\phi: (0, \infty) \times (0, \infty) \to (0, \infty)$ is symmetric, i.e., $\phi( u,v) = \phi(v, u)$,
$k_0$ is a non-negative function on $[0, \infty)$, 
and $\hrho$ is the so-called empirical bandwidth function rescaled from the $k$NN distance $\hat{R}$ (will be defined in \eqref{eq:def-hat-rho} below). 
The functions $\phi$ and $k_0$ will need to satisfy additional assumptions for our convergence results to hold. 
The general model \eqref{eq:def-W} recovers \eqref{eq:def-kNN-graph} and \eqref{eq:def-self-tune-graph} by choosing the suitable $\phi$ and $k_0$, and only $\hat R$ (instead of $\hat \rho$) is involved because $\phi$ is homogeneous in these cases, see the details in Section \ref{sec:prelim}.

To prove the convergence of graph Laplacians, we follow the {\it manifold data} setting,
which assumes that high dimensional data samples $x_i$ are drawn i.i.d. from a density $p$ supported on an (unknown) $d$-dimensional manifold $\calM$ embedded in $\R^m$ (Assumption \ref{assump:M}). 
Under additional technical conditions, one can show that a graph Laplacian matrix $L_N$ converges to a limiting manifold differential operator $\calL$.
We will see that the graph Laplacians associated with $W$ converge to the limiting operator $\sL_p$, to be defined in \eqref{eq:def-sL}.
We quantify the convergence speed in terms of the {\it pointwise convergence} rate, namely, the high-probability bound of $| L_N f(x) - \calL f(x)|$ under certain asymptotic regime (a joint limit of large $N$ and small bandwidth $\epsilon$),
where $f$ is a test function on $\calM$. 
Our theoretical results on the $k$NN graph Laplacian convergence rates are summarized as follows:

\begin{itemize}
    \item
When $L_N$ is the un-normalized graph Laplacian $L_\un$ or the random-walk graph Laplacian $L_\rw$
using $W$ as in \eqref{eq:def-W}, and $\calL = \sL_p$ as in \eqref{eq:def-sL},
and  $k_0$ and $\phi$ have sufficient regularity,
then, with properly chosen parameters $k$ and $\epsilon$, $| L_N f(x) - \calL f(x)|$ is bounded to be 
    $ O (N^{-2/(d+6)}\sqrt{\log N}).$

\item 
When $L_N$ is the un-normalized graph Laplacian $\widetilde L_\un$ or the random-walk graph Laplacian $\widetilde L_\rw$ using $\widetilde W$ as in \eqref{eq:def-Wtilde}, and $\calL = \Delta_p$ as in \eqref{eq:def-Delta-p}, we have the same convergence rates.
   \end{itemize}

For the constant in the big-$O$ notation, we primarily track the constant dependence on the manifold density $p$ and the test function $f$.
When discussing overall error rates, we omit the constant dependence for notation brevity. 
The detailed constant dependence can be found in the proofs. 
Our pointwise rate
matches the known result for fixed-bandwidth graph Laplacians using differentiable kernels \cite{singer2006graph, cheng2022eigen}.

Our graph Laplacian analysis is based on a refined analysis of the $k$NN estimator. 
Specifically, we uniformly bound the relative error between 
 the empirical bandwidth function $\hrho$ 
 and a population bandwidth function $\brho_{r_k}$ (see Definition \ref{def:bar-rho-epsilon})  on $\sM$, with high probability, to be
\[ O( (k / N)^{3/d} +  \,  \sqrt{\log N / k}). \]
This refined $k$NN estimation is due to a correction term that will be introduced in the construction of $\brho_{r_k}$ depending on $k$, and this analysis may be of independent interest. 
We further validate our theoretical findings using numerical experiments.

\textbf{Notation}
The notations in this paper are standard.
We use $\cdot$ to stand for the vector-vector inner-product, possibly between two vector fields.
We use $| \cdot  |$ for the absolute value, and $\| \cdot \|$ for vector or tensor norms.
Throughout the paper, $\R_+ =(0, \infty)$. 
We utilize the following asymptotic notations: 
suppose $g > 0$,
$f = O(g)$ indicates that there exists a constant $C > 0$ such that $|f| \leq C g$ in the limit. 
The notation $f \lesssim g$ means that there exists a constant $C > 0$ such that $0 \le f \leq C g$ in the limit.
$f = o(g)$ means that, $|f|/ g \to 0$ in the limit.
$f \ll g$ means that $f \ge 0$ and $f/g \to 0$.
$f = \Omega(g)$ implies that $f> 0$ and $f / g \to \infty $.
$f = \Theta(g)$ means that for $f, g > 0$, there exist constants $C_1, C_2 > 0$ such that $C_1 g\le f \le C_2 g$ in the limit.
We may use the superscript $^{[\; ]}$ to emphasize constant dependence, e.g., $O^{[\rm x]}(\cdot )$ means the constant in the big-$O$ notation depends on x.

\subsection{Related works}\label{sec:related-works}

\paragraph{$k$NN for density estimation.}
Our analysis of $k$NN graph will be based on the analysis of $k$NN distance to its (properly normalized) deterministic counterpart.
A closely related problem is the density estimation by $k$NN distance, 
known as nearest neighbor density estimation (NNDE),
the study of which dated back to the 70's \cite{moore1977consistency, devroye1977strong, mack1979multivariate, mack1983rate, zhao2022analysis}.
Early works demonstrated the strong uniform convergence of NNDE without specifying rates, see reference in \cite{biau2015lectures}. 
\cite{mack1979multivariate} established a point-wise convergence at a rate of $O((k/N)^{2/d} + \sqrt{\log N / k})$ where $N$ is the sample size, assuming that the density is lower and upper bounded and $C^2$ on $\R^d$.
More recently, \cite{zhao2022analysis} proved the minimax optimality of the NNDE estimator under certain cases.

These works considered densities on a Euclidean domain, and the analysis of $k$NN estimator applied to samples lying on a $d$-dimensional manifold embedded in a higher dimensional ambient Euclidean space appeared as intermediate results in the two works
	\cite{calder2022improved, cheng2022convergence} on the $k$NN graph Laplacians.
Specifically, 
\cite{calder2022improved} showed the convergence of a $k$NN distance estimator with a relative error of $O(\sqrt{\log N / k})$ given that $(k/N)^{2/d} \lesssim \sqrt{\log N / k}$;
\cite{cheng2022convergence} proved the uniform convergence of the $k$NN distance estimator with a relative error of $O((k/N)^{2/d} + \sqrt{\log N / k})$.
These rates echoed that in \cite{mack1979multivariate} by replacing the Euclidean dimension to be the manifold intrinsic dimension $d$.
In this work, we will refine the $k$NN analysis by introducing a correction term to the limiting deterministic bandwidth function, and we follow some techniques in \cite{cheng2022convergence}, see more in Remark \ref{rmk:comp-hrho-est}.

\paragraph{Convergence of graph Laplacians.}
The convergence of fixed-bandwidth graph Laplacians has been extensively studied
dating back to the 2000s \cite{hein2005graphs, coifman2006diffusion, singer2006graph, hein2007graph, belkin2008towards}.
The seminal work of Diffusion Map  \cite{coifman2006diffusion} showed
the convergence of the continuous kernel integral operator to the limiting manifold Laplacian.
Considering graph Laplacians constructed from $N$ i.i.d. samples using a fixed-bandwidth kernel, 
\cite{singer2006graph} established pointwise convergence at the rate  $O(N^{-2/(d+6)})$.
The same convergence rate was extended to the connection Laplacian in \cite{singer2017spectral}.
The present work considers adaptive bandwidths determined by empirical $k$NN distances and establishes the same rate of graph Laplacian pointwise convergence, 
despite the additional estimation error and statistical dependence introduced by the $k$NN bandwidth.
In addition to point-wise convergence, spectral convergence (eigen-convergence), namely the convergence of eigenvalues and eigenvectors of graph Laplacians to the spectra of limiting manifold operators has also been analyzed in the literature \cite{singer2017spectral,dunson2021spectral,garcia2020error, calder2022improved,calder2022lipschitz,cheng2022eigen}. 
We discuss the pointwise convergence and eigen-convergence in more detail in Section \ref{subsec:eigen-rate}.

The convergence of $k$NN graph Laplacians has also been analyzed in several places.
The earlier work of \cite{ting2010analysis} considered a class of graph Laplacians that include the $k$NN affinity as a special case, and the affinity is obtained using a compactly supported kernel. 
The analysis in \cite{ting2010analysis} can prove the point-wise convergence of the $k$NN graph Laplacian but without a convergence rate.
Analyzing both fixed-bandwidth compactly supported kernel and $k$NN graph construction (the affinity is as in  \eqref{eq:def-kNN-graph}), 
\cite{calder2022improved} proved both the point-wise and spectral convergence of the graph Laplacian with a rate of $O(N^{-1/(d+4)})$.
Utilizing empirical process techniques, \cite{guerin2022strong} analyzed $k$NN graph Laplacian operators and proved uniform point-wise convergence to the manifold limiting operator for both fixed-bandwidth kernel affinity and $k$NN kernel affinity (using an asymmetric version), 
achieving a rate of $O(N^{-1/(d+4)})$.
Considering differentiable kernel functions instead of compactly supported ones, 
\cite{cheng2022convergence} studied $k$NN graph Laplacians with a family of self-tuned kernels, possibly with normalization, and obtained a point-wise convergence rate. 
Our work is most closely related to \cite{calder2022improved, cheng2022convergence}, and we will further comment on the relationship of our result to these prior works later in the paper.

\paragraph{$k$NN graph applications.}
$k$NN graphs have been applied in various scenarios,
including computing manifold geometry-preserving low-dimensional embedding (using a self-tuned kernel) \cite{kohli2021ldle},
constructing the manifold scattering transform \cite{chew2022manifold}, 
and computing the persistent homology  \cite{damrich2023persistent}.
In particular, $k$NN graph distance can be used to approximate manifold distance, e.g., ISOMAP \cite{tenenbaum2000global} used it to compute manifold geodesic distance \cite{bernstein2000graph}, 
and more recently, $k$NN was used in the study of the so-called Fermat distances \cite{trillos2024fermat}.
Meanwhile, $k$NN has also been adopted for regression problems on manifold data. 
Regressors based on $k$NN have been adopted in \cite{goldberg2009multi} using Mahalanobis distance, and the statistical minimax optimality was studied in \cite{moscovich2017minimax,qiu2024semi}.
In the context of the Gaussian process on manifolds, \cite{fichera2024implicit} used eigen-pairs of the $k$NN graph Laplacian to approximate the Mat\'ern kernel,
and \cite{tang2024adaptive} used $k$NN distance to design the prior of kernel bandwidth of the Gaussian process in manifold Bayesian regression.
Focusing on semi-supervised learning on graphs, \cite{flores2022analysis} extended the pointwise convergence analysis of $k$NN graphs in \cite{calder2022improved} to $p$-Laplacian (by considering a convex combination of $p=2$ and $p=\infty$ Laplacians) and proposed efficient methods for 
semi-supervised learning with very few labels.
The analysis of $k$NN graphs in this work may be extended to study $k$NN methods in these applications.

\subsection{Preliminaries}  \label{sec:prelim}

\paragraph{Riemannian manifold.}
We briefly provide some background on differential geometry used in this work. For formal definitions and additional explanations, we refer readers to standard Riemannian geometry textbooks \cite{do1992riemannian, petersen2006riemannian}.

We consider a connected compact smooth manifold $\sM$ embedded in a high dimensional Euclidean space $\R^m$. 
The manifold Riemannian metric is denoted as $g$.
We further assume that the embedding in $\R^m$ is isometric, and denote by 
$\iota$ the embedding mapping from $\sM$ to $\R^m$, $\iota \in C^\infty(\calM)$.
When there is no danger of confusion, we use the same notation $x$ to denote $x\in {\sM}$ and $\iota(x)\in \mathbb{R}^m$.

For any two points $x,y \in \sM$, the {\it Riemannian distance} between $x$ and $y$ is the infimum of the lengths of all piece-wise regular curves on $\calM$ connecting $x$ and $y$, denoted as $d_{\sM}(x,y)$.
When the manifold is connected and compact, the metric space $(\calM, d_{\sM})$ is complete, and 
for any two points $x,y$ on $\sM$ there exists a length-minimizing geodesic $\gamma$ joining from $x$ to $y$. 
That is to say, the length of the geodesic is equal to $d_{\sM}(x,y)$.
We also call $d_{\sM}(x,y)$ the manifold geodesic distance.

Let $dV$ represent the (local) volume form on $\sM$ induced by $g$, then $(\calM, dV)$ gives a measure space. 
For a probability distribution $P$ on $\calM$, we can use $dV$ as the base measure and denote $P$ as $dP = p dV$, whenever $P$ has a density  $p: \sM \to [0, \infty)$.

\paragraph{Manifold derivatives.}
To differentiate a function $f$ on the manifold $\calM$, 
one can compute the derivatives of $f$ using local charts. Specifically, for $f \in C^k(\calM)$, 
given $ x\in \sM$ and a local chart $(U, \phi)$ around $x$, $\phi: U \to \R^d$,
$f \circ \phi^{-1}$ is $C^k$ on  $\phi(U) \subset \R^d$. 
This way of defining function differentiation depends on the choice of charts. 
The mapped coordinates in $\R^m$ by  $\iota$ are called extrinsic coordinates. 
In contrast, a construction is called {\it intrinsic} if it only depends on the Riemannian tensor $g$, not the embedding mapping $\iota$.

The intrinsic manifold derivative (covariant derivative) of a function $f \in C^k(\calM)$ is defined using the Riemannian (Levi-Civita) connection $\nabla$ on $\sM$.
Specifically, 
the first derivative 
$\nabla f$ at $x\in \sM$  is defined 
s.t.  $\forall v \in T_x \sM$, 
$\nabla f\big|_x (v) = \nabla_v f\big|_x := \frac{d}{dt} f( \gamma (t))|_0$ where $\gamma(t)$ is a differentiable curve on $\calM $ s.t. $\gamma (0)= x$ and $\gamma'(0) = v$. 
For a smooth vector field $U$, we denote by $\nabla f(U) = U (f) $ defined as $U (f)|_x = \nabla_{U(x)} f|_x$.
Higher-order derivatives can also be defined using the connection.
Formally, for smooth vector fields $V_1,\dots, V_l, U$ on $\calM$,
$\nabla^{l+1} f(U, V_1,\dots, V_l) 
= \nabla_U (\nabla^l f) (V_1, \dots, V_l)  = \nabla_U ( \nabla^l f (V_1, \dots, V_l) ) - \sum_{j=1}^l \nabla^l f (V_1, \dots,  \nabla_U  V_j  , \dots, V_l) $.

It is useful to introduce normal coordinates for the representation of manifold derivatives at a point.
For any $x \in \calM$, 
let $ \{ \partial_i|_x \}_{i=1}^d$ be an orthonormal basis of $T_x \calM$
and $ s = \{ s_i \}_{i=1}^d$ be the normal coordinates associated with $\{\partial_i|_x\}$. 
We denote the {\it exponential map} at $x$ by $\exp_x: T_x \sM \to \sM$.
For any $f \in C^k(\sM)$, we define $\tilde f \coloneqq f \circ \exp_x: T_x \sM\cong \R^d \to \R$. Then, for any $v \in T_x \sM$, 
denoting the $k$-th derivative in $\R^d$ by $D^k$,  we have 
\begin{equation} \label{eq:nabla^k-f-norm-coord-v}
	\nabla^k f\big|_x (v, \cdots, v)   = D^k \tilde{f}(0)(v, \dots, v).
\end{equation}
In \eqref{eq:nabla^k-f-norm-coord-v}, we identify $v = \sum_{j=1}^d v_{j}  \partial_j | _x \in T_x \sM$ with $v = (v_{1}, \dots, v_{d}) \in \R^d$, which we also denote by $v$,
and $D^k \tilde{f}(0)(v, \dots, v) = \sum_{j_1,\dots,j_k = 1}^{d} v_{j_1}\cdots v_{j_k}  \frac{ \partial^k \tilde{f}}{\partial s_{j_1} \cdots \partial s_{j_k}}(0)$. 

In this work, we use the operator norm of the $k$-th manifold derivative of a manifold function $f$, which can be intrinsically defined in the following way.
Observe that on the r.h.s. of \eqref{eq:nabla^k-f-norm-coord-v}, $D^k \tilde{f}(0)$ is  a $k$-way tensor mapping from $(\R^d)^{\otimes k}$ to $\R$.
Because $f$ is $C^k$, 
$D^k \tilde{f}(0)$ is a real symmetric tensor, 
and this allows us to use the spectral norm of the symmetric tensor 
$D^k \tilde{f}(0)$
to define the operator norm of $\nabla^k f|_x$. 
Specifically, the spectral norm of a $k$-way real symmetric tensor $T$ is defined as $\| T\| := \sup_{ v_1, \cdots, v_k \in  \R^d, \|v_i\| \le 1} |T( v_1,\cdots, v_k)|$,
and $ \| T\| = \sup_{ v \in \R^d, \|v\| \le 1} |T( v,\cdots, v)|$ as a result of Banach's Theorem \cite{banach1938homogene}.
Inserting into \eqref{eq:nabla^k-f-norm-coord-v}, we have
$
\| D^k \tilde{f}(0) \| 
=  \sup_{v  \in T_x \sM, \|v\| \leq 1} \left|  \nabla^k f\big|_x  (v,\dots,v) \right|$,
which is intrinsic since both $\exp_x$ and $\nabla^k f$ are.

\begin{definition}\label{def:Dkf-norm}
For $f \in C^k(\calM)$, define   
$ \|\nabla^k f|_x\| :=  \sup_{v  \in T_x \sM, \|v\| \leq 1} \left|  \nabla^k f\big|_x  (v,\dots,v) \right|$
for any $x\in \calM$,
and 
$   \| \nabla^k f \|_\infty 
    \coloneqq  \sup_{x \in \sM}  \| \nabla^k f |_x \|$.
\end{definition}

In the definition of $\| \nabla^k f \|_\infty$, the supremum $ \sup_{x \in \sM} $ is finite and attained due to the continuity of $\nabla^k f$ and the compactness of $\calM$.
The definitions of $\nabla^k f$ and the norm $\|\nabla^k f|_x\|$ are all intrinsic.

\paragraph{Graph affinity and graph Laplacians.}

We first explain how our general affinity $W$ in \eqref{eq:def-W} contains the classical $k$NN affinity \eqref{eq:def-kNN-graph} and the self-tuned affinity \eqref{eq:def-self-tune-graph}:
We will introduce $r_k$ in \eqref{eq:def-rk} which leads to \eqref{eq:rhohat-Rhat}.
By \eqref{eq:rhohat-Rhat} and choosing $\epsilon = r_k$, \eqref{eq:def-W} becomes (we omit the constant factor $\epsilon^{-d/2}$)
\begin{equation*} 
	W_{ij} =  k_0 \left(  \frac{\| x_i - x_j \|^2}{ r_k \phi(\hat{R}(x_i) / \sqrt{r_k}, \hat{R}(x_j) / \sqrt{r_k})^2 } \right).
\end{equation*}

Recovery of \eqref{eq:def-kNN-graph}:
let $\phi(u, v) = \max\{u,v\}$. Since $\phi$ is homogeneous, the factors $r_k$ in the denominator cancel out, and we have that $W_{ij} = k_0 ( \frac{\| x_i - x_j \|^2}{ \max\{\hat{R}(x_i), \hat{R}(x_j) \}^2 } ) $.
By setting $k_0(\eta) = \mathbf{1}_{[0,1]}(\eta)$, this recovers the affinity with ``max'' in \eqref{eq:def-kNN-graph}.
Similarly, by letting $\phi(u, v) = \min\{u,v\}$, this recovers the affinity with ``min'' in \eqref{eq:def-kNN-graph}.

Recovery of \eqref{eq:def-self-tune-graph}:
let $\phi(u,v) = \sqrt{uv}$. Since this $\phi$ is also homogeneous, 
the denominator becomes $\hat{R}(x_i)\hat{R}(x_j)$, and 
we have that $W_{ij} = k_0 ( \frac{\| x_i - x_j \|^2}{ \hat{R}(x_i)\hat{R}(x_j) } ) $ which is \eqref{eq:def-self-tune-graph}.

Given an affinity matrix $W$, the associated un-normalized graph Laplacian matrix is defined as $ D - W$, where $D$ denotes the degree matrix, which is a diagonal matrix with $D_{ii} = \sum_{j=1}^N W_{ij}$. 
The random-walk graph Laplacian is defined as $ D^{-1}(D-W) = I - D^{-1}W$.
In addition to these two basic constructions, various normalization techniques have been introduced in the literature, see e.g. \cite{hoffmann2022spectral, berry2016variable}. 
In this work, our analysis will focus on un-normalized and random-walk graph Laplacians with possibly minor modifications,
and details will be specified later.

\section{Refined $k$NN estimation}\label{sec:knn-estimation}

In this section, we bound the (relative) error between 
the $k$NN bandwidth function $\hrho$ and a population bandwidth function $\brho_{r_k}$ that depends on $k$. We will formalize the definitions of $\hrho$ and $\brho_{r_k}$. 
The relative error is bounded uniformly in Theorem \ref{thm:consist-hrho}.
All the proofs in this section are in Appendix \ref{sec:proof-sec-rho}.

\subsection{$k$NN distance and the empirical bandwidth $\hat \rho$}  \label{sec:def-hrho}

Given a dataset $X$ in $\R^m$ consisting of $N$ samples $\{ x_j \}_{j=1}^N$ and $x \in \R^m$, we define
\begin{equation}\label{eq:def-hat-R}
\hat{R}(x)\coloneqq \inf_r \{r>0, \text{s.t. } \sum_{j=1}^N \mathbf{1}_{\{ \|x_j-x\|\leq r \}} \geq k\},
\end{equation}
which is equivalently the distance between $x$ and its $k$NN in $X$. 
In this work, we will focus on the manifold data setting where samples $x_i$ lie on a low-dimensional manifold embedded in $\R^m$:

\begin{assumption}[Data manifold]\label{assump:M}
Let ${\sM}$ be a $d$-dimensional connected compact smooth manifold (without boundary) isometrically embedded in $\mathbb{R}^{m}$.
The data samples $x_i$, $i=1,\dots, N$ are drawn i.i.d. from a probability distribution with density $p$ supported on $\calM $.
\end{assumption}

Note that the $k$NN function $\hat R$ is defined everywhere on $\R^m$ (not necessarily on the support of $X$), and $\hat R$ can be constructed once an arbitrary dataset $X$ is given. 
The Lipschitz continuity of $\hat{R}$ in $\R^m$ (Lemma \ref{lemma:Lip-cont-hat-R}) also holds for general $X$ and does not use the manifold data assumption.

When the number of samples $N$ increases, we will theoretically set the $k$NN parameter $k$ to scale with $N$ in a way such that $k \to \infty$ and  $k/N \to 0$. 
As a result, the magnitude of $\hat R(x)$ will also decrease to zero. 
For the $k$NN estimator to have an $O(1)$ limit, one will need to rescale $\hat R(x)$ with a factor.
Following \cite{cheng2022convergence}, we introduce the (rescaled) $k$NN estimator $\hrho(x)$ as
\begin{equation}\label{eq:def-hat-rho}
\hrho(x) \coloneqq \hat{R}(x)\left(\frac{k}{\alpha_dN}\right)^{-1/d},
\quad 
\text{$\alpha_d := \Vol\,(\,\text{unit ball in $\R^d$} \,)$},
\end{equation}
which is called the {\it empirical bandwidth function}, and $\alpha_d$ is a constant depending on $d$. 
We emphasize that  $\hat{\rho}$ is introduced as a theoretical object, 
and particularly we do not need to know the normalizing constant $\alpha_d$ in practice.

It is known that, with a proper choice of $k$,  as $N$ increases, $\hat \rho $ converges to an $O(1)$ limit 
\begin{equation*}
\bar \rho: = p^{-1/d}
\end{equation*}
under the manifold data setting \cite{cheng2022convergence, calder2022improved}.
We will refine the estimation analysis by introducing a correction term to $\bar \rho$  that depends on $k$ and $N$ in Section \ref{sec:def-brho}. 
We define the constant $r_k$ as 
\begin{equation} \label{eq:def-rk}
 r_k \coloneqq \left(\frac{k}{\alpha_d N} \right)^{2/d}, 
\end{equation}
and then by \eqref{eq:def-hat-rho}, we have 
\begin{equation}  \label{eq:rhohat-Rhat}
\sqrt{r_k} \hrho(x) = \hat{R}(x).    
\end{equation}
The constant $ r_k \sim (k/N)^{2/d}$ is asymptotically small as $k/N \to 0$.
We will denote the corrected population bandwidth function as $\bar \rho_r$ and $r$ may have a subscript $_k$ to indicate the dependence on $k$.

\subsection{The population bandwidth function $\brho_{r_k}$}  \label{sec:def-brho}

The goal of our analysis in this section is to establish the approximation 
\[
\hat{\rho}(x) \approx \bar{\rho}_{r_k} (x),
\quad \text{uniformly for $x \in \sM$,}
\]
where $\bar{\rho}_{r_k}$ is called the {\it population bandwidth function}.
The definition of $\bar{\rho}_{r_k}$ involves a correction term of order $O( (k/N)^{2/d} )$ added to $\bar\rho $, the latter being the $O(1)$ limit of $\hat{\rho}$.
We will show in Section \ref{sec:concen-hrho} that 
this correction term refines the bound of bias error and improves over previous estimates in the literature.
The refined estimation is key to our later analysis of $k$NN graph Laplacians.

We first introduce the definition of $\brho_r$ for any $r \geq 0$ sufficiently small, and the threshold for $r$ depends on $\sM$ and $p$. 
In later analysis, we will set $r =r_k $ in $\bar \rho_r$,
where $r_k$ is defined in \eqref{eq:def-rk}.
At every point $x\in \sM$, the value of $\brho_r (x)$ 
is implicitly defined via the solution to the following equation in variable $t$, 
\begin{equation}\label{eq:def-bar-rho-epsilon}
t^d\left(1 + r Q(x)
t^2\right) = \frac{1}{p(x)}, 
\quad 
Q(x)\coloneqq \frac{1}{2(d+2)}
    \left( \frac{\Delta p(x)}{p(x)} + \omega(x) \right),
\end{equation}
where $\omega(x)$ is a function on $\calM$ depending on the manifold extrinsic coordinates at $x$. $\omega(x)$ is defined in \eqref{eq:def-omega} and introduced in Lemma \ref{lemma:G-expansion-h-indicator}. 
\eqref{eq:def-bar-rho-epsilon} is a univariate polynomial equation of degree up to $d+2$.
We will show in Lemma \ref{lemma:bar-rho-epsilon} that 
the construction is well-defined when $r $ is less than an $O^{[p]}(1)$ threshold.
In addition, we will construct the function $\bar \rho_r$ to be at least $C^1$ on $\sM$.
This calls for the $C^3$ regularity (and boundedness conditions) of $p$, which we introduce in the following assumption.

\begin{assumption}[Data density $p$]\label{assump:p}
The data distribution on $\calM$ has a density $p\in C^3(\sM)$ that is
uniformly bounded from below and above; that is, 
$\exists \, p_{\min}, \, p_{\max}>0$ s.t.

\begin{equation*} 
0< p_{\min} 
\le p(x) \le
p_{\max} < \infty,
\quad\forall x\in\sM.
\end{equation*}
\end{assumption}

We further define the constants $\rho_{\min}, \rho_{\max}$ as 
\begin{equation} \label{eq:def-rhomin-rhomax}
	\rho_{\min} \coloneqq \left(\frac{2}{3p_{\max}} \right)^{1/d},  \quad \rho_{\max} \coloneqq \left(\frac{2}{p_{\min}} \right)^{1/d},
\end{equation}
which provide the lower and upper bounds for $\brho_{r}$, to be proved in Lemma \ref{lemma:bar-rho-epsilon}.
We also introduce the positive constants $r_0$ and $\tilde r_0$,  $\tilde r_0 \le r_0$, defined as 
 \begin{equation}\label{eq:def-r0-tilder0}
r_0 \coloneqq 
\frac{d}{2 (d+2) (\|Q\|_\infty+1) \rho_{\max}^2 },
\quad \tilde{r}_0 \coloneqq \frac{d \rho_{\min}^{d-1}}{2 (d+2) (\|Q\|_\infty + 1) \rho_{\max}^{d+1} },
\end{equation}
 both depending on $p$ and $\calM$. 
$\brho_{r}$ is guaranteed to be  well-defined when $ r \leq r_0$,
and $r \le \tilde r_0$ is needed when we handle the first derivative of $\brho_{r}$.
Note that since $\sM$ is compact and smooth and $p$ is $C^3$, 
by the definition of $Q$ in \eqref{eq:def-bar-rho-epsilon},
we have $\|Q\|_\infty \leq \frac{1}{2(d+2)} ( \| \Delta p / p \|_\infty  + \| \omega \|_\infty ) < \infty$. Thus, both $r_0 $ and $\tilde r_0$ are well-defined, positive and finite.

\begin{definition}[Population bandwidth function $\brho_r$]  \label{def:bar-rho-epsilon}
If $0 \leq r \leq r_0$, 
for $x \in \sM$, we define $\brho_r(x)$ as the unique solution to the polynomial equation \eqref{eq:def-bar-rho-epsilon} with respect to the variable $t$ on $(0, \rho_{\max})$.
\end{definition}

When $r = 0$, the correction term vanishes and $\brho_r$ reduces to $\brho = p^{-1/d}$.
The following lemma establishes the construction and properties of $\brho_r$ and is proved by the Implicit Function Theorem.

\begin{lemma}[Construction of $\brho_r$]
\label{lemma:bar-rho-epsilon}
Under Assumptions \ref{assump:M} and \ref{assump:p}, 
suppose $0\leq r \leq r_0$.

(i)
$\brho_r(x)$ is well-defined,
i.e., 
the equation \eqref{eq:def-bar-rho-epsilon} has a unique solution in $t \in (0, \rho_{\max})$.
If $p \in C^l(\sM)$ for some integer $l \geq 3$, then $\brho_r \in C^{l-2}(\sM)$.

(ii) $\brho_r(x)\in [(\frac{2}{3p(x)})^{1/d}, (\frac{2}{p(x)})^{1/d}]$  for any $x\in\sM$ and for all $r \in [0, r_0]$.
As a result, 
$ \brho_r(x)\in  [\rho_{\min}, \rho_{\max}]$ for any $x\in\sM$ and any $r$.

Recall that $\bar\rho = p^{-1/d}$. 

(iii)
$\|\brho_r - \brho\|_\infty \leq C_0(p)r$, where $C_0(p) > 0$ is a constant that depends on $\sM$ and the derivatives of $p$ up to the second order. 

(iv)
For each $l=1,2,3$, if we further have that $p\in C^{l+2}(\sM)$ and $0\leq r \leq \tilde{r}_0$, then
$\|\nabla^l \brho_r - \nabla^l \brho \|_\infty \leq C_l(p)r$, where $C_l(p) > 0$ is a constant that depends on $\sM$ and the derivatives of $p$ up to the $(l+2)$-th order.
\end{lemma}

The upper bounds for $\brho_r -\brho$ and its derivatives in (iii) and (iv) are not to control the deviation via small $r$. Instead, we will set $r=r_k$ 
(when $k= o(N)$, $r_k$ defined in \eqref{eq:def-rk} will be $o(1)$ and satisfies the required smallness condition with large $N$), and these estimates will be used in the analysis of the $k$NN graph Laplacian.
Specifically, the constants $C_l(p)$ in (iii)(iv), which only depend on $p$ (and $\calM$), will enter the constants in big-$O$ in the (bias) error terms of graph Laplacian convergence bounds. 
The estimate in part (iv) can be extended to higher derivatives 
under higher regularity assumptions on $p$, 
while in this work we only use derivatives up to order $ l =3 $.
Note that $\brho_{r_k}$ is introduced for theoretical purposes and will not be assumed to be known in practice.

\subsection{Concentration of $\hrho$ at $\brho_{r_k}$}  \label{sec:concen-hrho}

We are ready to prove the approximation $\hat{\rho} \approx \bar{\rho}_{r_k}$.
The following proposition proves this at a single point $x \in \calM$.

\begin{proposition}\label{prop:consist-hrho-x0}
Under Assumptions \ref{assump:M} and \ref{assump:p}, if as $N\to\infty$, $k=o(N)$ and $k=\Omega(\log N)$, then for any $s>0$, when $N$ is sufficiently large, 
for any $x\in\sM$, w.p. $\ge 1-2N^{-s}$, 
\begin{equation*} 
\frac{|\hrho(x)-\brho_{r_k}(x)|}{\brho_{r_k}(x)} = O^{[p]}\left(\left(\frac{k}{N}\right)^{3/d}\right)+O\left(\sqrt{\frac{s\log N}{k}}\right).    
\end{equation*}
The constant in $O^{[p]}(\cdot)$ depends on $\sM$ and $p$, 
the constant in $O(\cdot)$ only depends on $d$ (and is independent of $p$), 
and both constants are uniform for all $x\in\sM$.
The threshold for large $N$ depends on $(\sM, p, s)$ and is uniform for $x$.
\end{proposition}

The relative error control in Proposition \ref{prop:consist-hrho-x0} is at a single point $x$. We can bound the relative error uniformly over $\calM$ using a union bound, following a similar approach as in \cite{cheng2022convergence} and making use of the Lipschitz continuity of $\hrho$ and $\brho_{r_k}$ and the compactness of the manifold (the covering lemma in Lemma \ref{lemma:covering}).
Specifically, 
we define
\begin{equation}  \label{eq:def-vareps-rho}
    \varepsilon_{\rho, k} 
    \coloneqq \sup_{x \in {\sM}} \frac{|\hrho(x)-\brho_{r_k}(x)|}{\brho_{r_k}(x)},
\end{equation}
and the uniform bound is proved in the following theorem.

\begin{theorem}\label{thm:consist-hrho}
Under Assumptions \ref{assump:M} and \ref{assump:p}, let $\hrho$ and $\brho_{r_k}$ be defined as in \eqref{eq:def-hat-rho} and \eqref{eq:def-bar-rho-epsilon}.
If as $N \to \infty$, $k = o(N)$ and
$k = \Omega( \log N)$,
then when $N$ is sufficiently large,
w.p. higher than  $1- N^{-10}$,
\begin{equation}  \label{eq:vareps-rho-bound}
\varepsilon_{\rho, k}
=
O^{[p]} \left(\left(\frac{k }{N }\right)^{3/d} \right) 
+   O \left(\sqrt{\frac{\log N}{k}}\right).
\end{equation}
The constant in $O^{[p]}(\cdot)$ depends on $\sM$ and $p$, and the constant in $O(\cdot)$ only depends on $d$ (and is independent of $p$). 
The large-$N$ threshold depends on $\sM$ and $p$.
\end{theorem}
The bound in  \eqref{eq:vareps-rho-bound} is $o^{[p]}(1)$ under the asymptotic regime of the theorem. 

\begin{remark}[The overall error at optimal $k$] \label{rmk:hrho-sim-rate}
The error bound in \eqref{eq:vareps-rho-bound} consists of two terms: the $O((k / N)^{3/d})$ term is the bias error, and the $O(\sqrt{\log N / k}) $ term is the variance error. 
The choice of $k$ that balances the two terms is $k \sim N^{6/(d+6)}$ (omitting the $\log N$ factor). Under this scaling, the overall error is $O^{[p]}( N^{-3/(d+6)} )$, up to a $\log N$ factor.
\end{remark}

\begin{remark}[The $(k/N)^{3/d}$ bias error]
\label{rmk:comp-hrho-est}
In previous NNDE literature, it was proved that $\hrho^{-d}$ converges to $p$ at the rate of $O((k/N)^{2/d} + \sqrt{\log N/ k})$, see, e.g., \cite{mack1979multivariate}.
The same rate was established and used in the study of $k$NN graph Laplacians
for the convergence of $\hrho^d$ to $p^{-1}$ in  \cite{calder2022improved}
and the uniform convergence of $\hrho$ to $p^{-1/d}$ in \cite{cheng2022convergence}.
Our result here does not conflict with these previous findings, since the improved $O((k/N)^{3/d})$ bias error is for the approximation of $\hrho$ to $\brho_{r_k}$.
By construction, $\brho_{r_k}$ has an $O(r_k)$ correction term added to $\brho = p^{-1/d}$, and $r_k \sim (k/N)^{2/d}$  by definition \eqref{eq:def-rk}. 
As shown in the proof, this correction term improves the $O((k/N)^{2/d})$ bias error to be $O((k/N)^{3/d})$. 
\end{remark}

\section{Graph Laplacian pointwise convergence}
\label{sec:un-GL-fast}

The matrix $L_\un$ is defined as (all graph Laplacians involve a negation in the definition)
\begin{equation*} 
 L_\un \coloneqq  - \frac{1}{\frac{m_2}{2}N} \diag\left(  \left\{ \frac{1}{ \epsilon \hrho(x_i)^2 }  \right\}_i \right) (D - W) \in \R^{N\times N}. 
\end{equation*}
We will also use the notation $L_{\rm un} f$ to stand for
\begin{equation}\label{eq:def-L-un}
L_\un f(x) \coloneqq \frac{1}{\frac{m_2}{2}N\epsilon \hrho(x)^2 }\sum_{j=1}^N \epsilon^{-d/2}
k_0 \left(  \frac{\| x - x_j\|^2}{ \epsilon \phi(\hrho(x) ,  \hrho(x_j))^2 } \right) (f(x_j)-f(x)),   
\end{equation}
which is defined at any $x \in \calM$. 
The operator in \eqref{eq:def-L-un} can be understood as extending the matrix-vector multiplication $L_\un  (\rho_X f )$, viewed as a function evaluated at $x_i$'s, to everywhere on $\calM$. 
Here, $\rho_X$ is the  function evaluation operator defined as 
\begin{equation} \label{eq:def-rhoX}
    \rho_X f \coloneqq  (f(x_1), \dots, f(x_N)) \in \R^N.
\end{equation}
For other graph Laplacian matrices $L$ below, we similarly use $L$ to stand for both the matrix and the operator in $L f(x)$ applied to a test function $f$ on $\sM$.

For regular $\phi$ and $k_0$ (Assumptions \ref{assump:phi-diff}-\ref{assump:k0-smooth}),
we will prove that $L_\un$ converges pointwise at a rate of $O (N^{-2/(d+6)} \sqrt{\log N})$.
The regularity conditions on the data density $p$ and the test function $f$ will be specified later. 
All proofs are provided in Appendix \ref{sec:proof-W-un-fast}.

\subsection{Limiting manifold operators}
\label{sec:limiting-operator}

We will show that the limiting operator of $L_{\un}$ is $\mathcal L_p$ defined by
\begin{equation}\label{eq:def-sL}
\mathcal L_p f
:=
\Delta f+
\left(1-\frac{2}{d}\right)
\frac{\nabla p}{p}\cdot\nabla f.
\end{equation}
Since $\nabla p$ and $\nabla f$ are covariant derivatives, the dot denotes the inner product induced by the Riemannian metric. Equivalently,
$
\mathcal L_p f
=
p^{-(1-2/d)}
\operatorname{div}\!\left(
p^{1-2/d}\operatorname{grad}f
\right).
$

For the self-tuned choice $\phi(u,v)=\sqrt{uv}$,
\cite{cheng2022convergence} considered the family of affinities
\[
W_{ij}^{(\alpha)}
=
\epsilon^{-d/2}
k_0\!\left(
\frac{\|x_i-x_j\|^2}
{\epsilon\widehat\rho(x_i)\widehat\rho(x_j)}
\right)
\frac{1}{
\widehat\rho(x_i)^\alpha
\widehat\rho(x_j)^\alpha
},
\qquad \alpha\in\mathbb R.
\]
With the corresponding normalization of the graph Laplacian, the
limiting differential operators are
\[
\mathcal L_p^{(\alpha)} f
:=
\Delta f+
\left(
1+\frac{2(\alpha-1)}{d}
\right)
\frac{\nabla p}{p}\cdot\nabla f.
\]
With general $\phi$, we analogously write
$W_{ij}^{(\alpha)} = \frac{ W_{ij} }{\phi(\hrho(x_i), \hrho(x_j))^{2\alpha} }$.
Our affinity $W$ in \eqref{eq:def-W} coincides with $W^{(0)}$,
and its limiting operator $\mathcal L_p$ in \eqref{eq:def-sL} is $\mathcal L_p^{(0)}$.

The choice $\alpha=1$ in $\mathcal L_p^{(\alpha)}$  gives the weighted Laplacian
\begin{equation}\label{eq:def-Delta-p}
\Delta_p f
:=
\Delta f+
\frac{\nabla p}{p}\cdot\nabla f
=
p^{-1}\operatorname{div}\!\left(
p\operatorname{grad}f
\right),
\end{equation}
which is the backward Kolmogorov operator of a diffusion process on $\mathcal M$ with invariant density $p$.
In Section~\ref{sec:main-results-norm}, 
we consider the normalized affinity
$\widetilde W = W^{(1)}$ for general $\phi$,
and  show that the corresponding graph Laplacians converge to $\Delta_p$.

In this work, we focus our analysis on the two cases involving $W$ and $\widetilde W$ for general $\phi$, 
corresponding to $\alpha=0$ and $\alpha=1$, respectively.
The same analysis can also be adapted to other values of $\alpha$.
In addition, a further normalization can
remove the density-dependent drift and recover the Laplace--Beltrami
operator $\Delta$. Within the family above, this corresponds to
$\alpha=1-d/2$. As proposed in
\cite{cheng2022convergence}, prior knowledge of $d$ can instead be
avoided by introducing an additional density estimator $\hat p$.
For the general class of affinities considered here, one possible construction is 
$
W_{ij}^{\mathrm{LB}}
:=
{\widetilde W_{ij}}/
(\hat p(x_i)^{1/2}\hat p(x_j)^{1/2}).
$
The appropriately normalized graph Laplacian associated with $W^{\mathrm{LB}}$ recovers $\Delta$.
The present pointwise convergence framework can be extended to this
construction by combining the density-estimation analysis for
$\hat p$ in \cite{cheng2022convergence} with an additional
replacement argument.

\subsection{Assumptions on graph affinity}\label{sec:assump-kernel}

Our graph affinity $W$ utilizes two functions $\phi$ and $k_0$,
and we introduce the needed assumptions on them respectively.
As shown in Section \ref{sec:prelim}, 
the classical $k$NN graph affinity uses $\phi$ as $\min$ or $\max$,
and the self-tuned graph affinity uses  $\phi(u, v) = \sqrt{uv}$.
We consider a generalized notion of differentiable $\phi$ satisfying additional technical requirements:
	
\begin{assumption}[Differentiable $\phi$]\label{assump:phi-diff}
	The bivariate function $\phi: \R_+ \times \R_+ \to \R_+$ satisfies:

	(i) For any $u, v\in \R_+$, $\phi(u, v) = \phi(v, u)$,  $\phi(u,u) = u$.

	(ii) $\phi \in C^3(\R_+\times \R_+)$, and for any $u\in \R_+$, $\partial_1 \phi(u,u) = {1}/{2}$,
 	where $\partial_1$ denotes the partial derivative with respect to the first argument.

	(iii) There exist constants $L_\phi, \delta_\phi > 0$, such that for any $u_1, v_1, u_2, v_2 \in \R_+$ with $\frac{\left| u_1 - u_2 \right|}{u_1}, \frac{\left| v_1 - v_2 \right|}{v_1} \leq \delta_\phi$, 
	\begin{equation*}  
		\frac{\left|  \phi(u_1, v_1) - \phi(u_2, v_2) \right| }{\phi(u_1, v_1)} \leq L_\phi\left( \frac{\left| u_1 - u_2 \right|}{u_1}  + \frac{\left| v_1 - v _2 \right|}{v_1}\right).    
	\end{equation*}

	(iv) There exist constants $0< c_{\min} \leq c_{\max} < \infty$, such that for any $u,v\in \R_+$,
	\begin{equation*} 
		c_{\min} \min\{ u,v \} \leq   \phi(u,v)  \leq c_{\max} \max\{ u,v \}.   
	\end{equation*}
\end{assumption}

We give a few remarks about the technical conditions in Assumption \ref{assump:phi-diff}. 
First, while $\phi$ is defined on $\R_+ \times \R_+$, the later analysis only uses $\phi$ on a bounded domain due to the boundedness of the $k$NN bandwidth. 
Second, the $O(1)$ constants in conditions (i)(ii) are generic choices without loss of generality.
Specifically, the choice is up to a constant multiplied to $\phi$ as long as $\partial_1 \phi(u,u)$ is assumed to be a constant (Lemma \ref{lemma:c1=1/2}). If we instead have $\phi(u,u) = cu$ and $\partial_1 \phi(u,u) = c/2$, then it at most incurs a constant multiplied to the limiting operators.
Third, condition (iii) is naturally satisfied by commonly used $\phi$, see the Examples \ref{exmp:phi-diff}.
Finally, the choice of $\phi$ may affect constants in the error bound (constants in big-$O$ notations) but not the convergence rate,
see the comments beneath Theorem \ref{thm:conv-un-Laplacian-case1} below.

\begin{exmp}\label{exmp:phi-diff}
	The functions $\phi(u, v) = \sqrt{uv},  \phi(u, v) = (u+v)/2, \phi(u, v) = \sqrt{(u^2+v^2)/2}$ 
   are smooth on $\R_+ \times \R_+$ and satisfy Assumption \ref{assump:phi-diff}.
The verification of	conditions (i)(ii) is straightforward. 
For all the three choices of $\phi$, condition (iii) holds with $L_\phi = 1.2$ and $\delta_\phi = 0.1$. 
(The selection of the values of $L_\phi$ and $\delta_\phi$ is not unique, and any valid combination can be used in later analysis.)
Finally, condition (iv) is satisfied with $c_{\min} = c_{\max} = 1$.
\end{exmp}

We consider $k_0$ as a differentiable and decaying function, and a representative example is the exponential function.

\begin{assumption}[Differentiable $k_0$]\label{assump:k0-smooth}
	$k_0: [0, \infty) \to [0, \infty)$ satisfies
	
	(i) Regularity. $k_0$ is continuous on $[0,\infty)$ and  $C^3$ on $\R_+$. 
	
	(ii) Decay condition.  $\exists a, a_l >0$, s.t., $ |k_0^{(l)}(\eta)| \leq a_l e^{-a \eta}$ for all $\eta \ge 0$, $l= 0, 1,2,3 $.
	
	We exclude the case that $k_0\equiv 0$. 
\end{assumption}

Throughout the paper, we denote by $a[k_0]$, $a_l[k_0]$ the constants in Assumption \ref{assump:k0-smooth}(ii), which are absolute constants determined by the function $k_0$.
We define
\begin{equation*}
	m_0[k_0] \coloneqq\int_{\R^d} k_0( \|u\|^2) du, \quad m_2 [k_0] \coloneqq \frac{1}{d}\int_{\R^d} \|u\|^2 k_0(\|u\|^2) du.
\end{equation*}
We may drop the dependence on $k_0$ in the notations of $m_0$ and $m_2$ when there is no confusion.

\begin{exmp}[Exponential $k_0$] \label{exmp:k0-exponential}
We consider 
	\begin{equation}\label{eq:def-k0-exp}
		k_0(\eta) = (4\pi)^{-d/2} \exp(-\eta/4), \quad \eta\geq 0. 
	\end{equation}
	The normalizing constant $(4\pi)^{-d/2}$ is a theoretical one and typically not needed in practice.
	For this $k_0$, $m_0=1$, $m_2 = 2$. 
This $k_0$ results in a Gaussian kernel in the ambient space $\R^m$ (with fixed bandwidth),
and with $k$NN bandwidth it results in the self-tuned kernel \cite{zelnik2004self, cheng2022convergence}. 
\end{exmp}

\subsection{Fast rate of $L_\un$}

\begin{theorem}[Fast rate  of $L_\un$] \label{thm:conv-un-Laplacian-case1}
Under Assumptions \ref{assump:M}-\ref{assump:p}, 
$\phi$ under Assumption \ref{assump:phi-diff} and $k_0$ under Assumption \ref{assump:k0-smooth},
suppose $p \in C^5(\sM)$,
and as $N\to\infty$, $\epsilon=o(1), 
\epsilon^{d/2}= \Omega(\log N/N), k = o(N), k = \Omega(\log N)$.
Then, for any $f\in C^4(\sM)$, when $N$ is sufficiently large, for any $x\in\sM$, w.p. higher than $1-5N^{-10}$, 
\begin{equation} \label{eq:fast-rate}
 \left| L_\un f(x)  -  \sL_p f(x)  \right| =    O^{[f,p]} \left(\epsilon  + r_k \right) 
 + O \left(\|\nabla f \|_\infty p(x)^{1/d} 
 \left(  \frac{\varepsilon_{\rho, k}}{\sqrt{\epsilon}}    + \sqrt{\frac{\log N}{N\epsilon^{d/2+1}}} \right)  \right) ,    
\end{equation}
and meanwhile $\varepsilon_{\rho,k}$ is bounded as in \eqref{eq:vareps-rho-bound}.
In particular, when $\epsilon \sim N^{-2/(d+6)}$ and $k \sim N^{6/(d+6)}$,
\begin{equation}  \label{eq:fast-rate-N}
	\left| L_\un f(x)  -  \sL_p f(x)  \right| =  O(N^{-2/(d+6)} \sqrt{\log N}).
\end{equation}  
\end{theorem}

In the second term of r.h.s. of \eqref{eq:fast-rate},
the constant in the big-$O$ notation 
depends on the density $p$ 
via a factor of $p(x)^{1/d}$.
This suggests that at a place where $p(x)$ is small and close to zero, 
the variance error (contained in the 2nd term) will not diverge.

\vspace{2pt}
\begin{remark}[Error rates and balancing choices of $\epsilon, k$]
The error bound \eqref{eq:fast-rate} involves $r_k$ and $\varepsilon_{\rho,k}$. 
Substituting \eqref{eq:def-rk} and
\eqref{eq:vareps-rho-bound} yields the bound
\eqref{eq:fast-rate1} in the proof, expressed in terms of $\epsilon$, $k$, and $N$.
By balancing the terms in this bound, we obtain the choices of
$\epsilon$ and $k$ that are optimal for this upper bound; see
Lemma \ref{lemma:fast-rate}. These choices yield the convergence rate
in \eqref{eq:fast-rate-N} as a function of $N$.
\end{remark}

\vspace{2pt}
\begin{remark}[Reduction to known point-wise rate  and natural scaling of $r_k$]\label{rk:natural-rk}
For graph Laplacians constructed using differentiable $k_0$ and fixed bandwidth $\epsilon$, the point-wise convergence rate was shown to be 
$	O^{[f]} \left(\epsilon  \right) + \, O\left( \|\nabla f\|_\infty   \sqrt{\frac{\log N}{N\epsilon^{d/2+1}}} \right)$,
when $p$ is uniform \cite{singer2006graph}. 
One can verify that our bound \eqref{eq:fast-rate} will reduce to the same two-term bound (trading-off $\epsilon$ and $N$) 
when setting $k \sim \epsilon^{d/2}N$, which is equivalent to that $r_k \sim \epsilon$. 
This is a natural scaling of $r_k$ in terms of $\epsilon$ because, when $\phi$ has homogeneity (e.g., min/max function or self-tuned),
by \eqref{eq:rhohat-Rhat},
the kernel affinity \eqref{eq:def-W} will then involve
$
k_0 \left(  \frac{\| x_i - x_j \|^2}{ c  \phi(\hat{R}(x_i), \hat{R}(x_j))^2 } \right)$
for some $O(1)$ constant $c$. 
This means that the kernel bandwidth is directly tuned by symmetrized $k$NN distance up to a constant factor. 
\end{remark}

We make a few more comments concerning the influence of the choice of $\phi$ on the theoretical error bound.
The constants in the big-$O$ notations in \eqref{eq:fast-rate} depend on several constants in Assumption \ref{assump:phi-diff}:
		$L_\phi$ in condition (iii), 
		$c_{\min}, c_{\max}$ in condition (iv), 
		and $C_{\phi, l}$ defined in \eqref{eq:def-C-phi-l}.
Specifically, the constant in the term $O^{[f,p]}(\epsilon)$ is dependent on $c_{\max}$ and $C_{\phi, l}$ for $l \leq 3$; 
The constant in $O(\|\nabla f\|_\infty p(x)^{1/d} \varepsilon_{\rho,k} / \sqrt{\epsilon})$ 
includes a factor of  $c_{\max}^{d+1} (L_\phi+1)$; 
The constant in the last term $O(\|\nabla f\|_\infty p(x)^{1/d} \sqrt{\log N / (N\epsilon^{d/2+1})})$ 
contains a factor of $c_{\max}^{d/2+1}$.
Details can be found in the proof.    
The requirement for $\phi$ to be $C^3$ on $\R_+ \times \R_+$ is technical. 
If $\phi$ only has $C^2$ regularity instead of $C^3$, then our technique will give a  slower convergence rate.
The influence is similar for other types of graph Laplacians later (Theorems \ref{thm:conv-rw-Laplacian} and \ref{thm:conv-un-Laplacian-hrho-norm}).

\subsection{Proof overview}\label{subsec:proof-overview-Lun}

To analyze the convergence of $L_\un$, our strategy is to introduce an operator $\bar{L}_\un$ by replacing $\hrho$ with $\brho_{r_k}$ in $L_\un$.
Specifically, $\bar{L}_\un$ is defined as
\begin{equation}\label{eq:def-Wbar-GL}
\bar{L}_\un f(x) \coloneqq \frac{1}{\frac{m_2}{2}N\epsilon \brho_{r_k}(x)^2 }\sum_{j=1}^N \epsilon^{-d/2}   k_0 \left(  \frac{\| x - x_j\|^2}{ \epsilon \phi(\brho_{r_k}(x) ,  \brho_{r_k}(x_j))^2 } \right)    (f(x_j)-f(x)). 
\end{equation}
The analysis  consists of three steps:

\paragraph{$\bullet$ Step 1. To bound $|L_{\rm un} f(x) - \bar{L}_{\rm un} f(x)|$}
    This utilizes the $k$NN estimation analysis in Section \ref{sec:concen-hrho} to control the error brought by replacing $\hrho$  with $\brho_{r_k}$. 
    The summation over $j$ in \eqref{eq:def-L-un} is not an independent sum because $\hat \rho(x_j)$ depends on all the data samples. 
    By replacing $\hrho$  with $\brho_{r_k}$, the summation in \eqref{eq:def-Wbar-GL} becomes an independent sum and is ready to be analyzed by a concentration argument. 
Specifically, 
recall $\varepsilon_{\rho,k}$ defined in \eqref{eq:def-vareps-rho},
and when Theorem \ref{thm:consist-hrho} holds, we have that, with high probability, \eqref{eq:vareps-rho-bound} holds and then $\varepsilon_{\rho,k}$ is $o(1)$.

\begin{proposition} [Replacement for differentiable $k_0$]  \label{prop:step1-diff}

Under Assumptions \ref{assump:M}-\ref{assump:p}, 
suppose 
$\phi$ satisfies Assumption \ref{assump:phi-diff}(i)(iii)(iv)
and 
$k_0$ satisfies Assumption \ref{assump:k0-smooth},
$p \in C^4(\sM)$,
and
as $N\to\infty$, $\epsilon=o(1)$, $\epsilon^{d/2}= \Omega(\log N/N)$,
and also $k = o(N), k = \Omega(\log N)$.
Then, for any $f\in C^1(\sM)$, when $N$ is sufficiently large, 
for any $x\in\sM$, 
w.p. higher than $1-3N^{-10}$, 
\begin{align}
    \left|L_\un f(x)-\bar{L}_\un f(x)\right| &=  O\left(\|\nabla f \|_\infty p(x)^{1/d} \frac{\varepsilon_{\rho,k}}{\sqrt{\epsilon}}\right)
    + O( \|f\|_\infty p(x)^{2/d}  \varepsilon_{\rho,k} \epsilon^{9}).
\label{eq:error-L-barL-un}
\end{align}
\end{proposition}
In Proposition \ref{prop:step1-diff} and all the propositions and theorems below, the threshold for large $N$ will at most depend on $(\sM, p, k_0, \phi)$;
 The constant in big-$O$, if not declared to depend on $p$ and $f$, will at most depend on $(\calM, k_0, \phi)$.
The statement will concern a fixed $x$, while both the large $N$ threshold and the constants in big-$O$ are uniform for all $x \in \calM$.

\paragraph{$\bullet$ Step 2. To bound  $|\bar{L}_{\rm un} f(x) - \sL_p f(x)| $}
    This follows similar arguments of graph Laplacian convergence analysis in previous literature, 
     but calls for a new analysis during the asymptotically vanishing $r=r_k$ and $\epsilon$.
In prior works, this type of convergence analysis breaks the error into two parts, namely the bias and variance errors respectively. The {\it bias error} assumes population kernel (infinite many samples) and handles the error due to small but non-zero $\epsilon$;
the {\it variance error} handles the finite-sample effect. 
This approach for analyzing graph Laplacian convergence has been employed for fixed bandwidth kernel \cite{singer2006graph} and variable bandwidth kernel \cite{berry2016variable}.
However, different from  \cite{berry2016variable} where the bandwidth  function is fixed, our $\brho_{r_k}$ involves the parameter $r_k$ defined in \eqref{eq:def-rk} and the $k$ (in $k$NN) can vary jointly with $N$ and $\epsilon$. 
Thus, we need a new analysis to handle the joint limit of $\epsilon, N, k$.
In addition, we consider a more general class of $\phi$ while \cite{berry2016variable} only covers the ``self-tuned'' case. 

The following lemma gives the key estimate in the bias error analysis. 
\begin{lemma} [Bias error for fast rate] \label{lemma:step2-diff}
Under Assumptions \ref{assump:M}-\ref{assump:p}, 
$\phi$ under Assumption \ref{assump:phi-diff} and $k_0$ under Assumption \ref{assump:k0-smooth},
let $\tilde{r}_0$ be defined in \eqref{eq:def-r0-tilder0}.
Suppose $p \in C^5(\sM)$, $f \in C^4(\sM)$, 
then, when $\epsilon <  \epsilon_D(k_0, c_{\max})$,
for any $r \in [0, \tilde{r}_0]$ and any $x\in \calM$, 
\begin{align*}
\brho_r(x)^{-2}\int_\sM \epsilon^{-d/2}k_0\left(  \frac{\| x - y\|^2}{ \epsilon \phi( \brho_r(x), \brho_r(y))^2 } \right)(f(y)-f(x))p(y)dV(y) 
= \epsilon \frac{m_2}{2}  
\sL_p  f(x) 
+ O^{[f,p]}\left(\epsilon^2 + \epsilon r \right).
\end{align*}
\end{lemma}
The small-$\epsilon$ threshold $\epsilon_D(k_0, c_{\max})$ is defined in \eqref{eq:def-epsD}, and 
it depends on $\calM$, $p$, the function $k_0$, and the constant $c_{\rm max}$ in Assumption \ref{assump:phi-diff}(iv) about $\phi$.

The variance error can be bounded using concentration arguments, 
and for this part of analysis we also use upper bounds for kernel integrals in Lemma \ref{lemma:right-operator-degree}.  
By combining the bias error and variance error, 
the following proposition proves the point-wise convergence rate of $\bar{L}_\un $ to $\sL_p $.

\begin{proposition}\label{prop:step2-diff}
Under Assumptions \ref{assump:M}-\ref{assump:p}, 
$\phi$ under Assumption \ref{assump:phi-diff} and $k_0$ under Assumption \ref{assump:k0-smooth},
suppose $p \in C^5(\sM)$,
and as $N\to\infty$, $\epsilon=o(1), 
\epsilon^{d/2}= \Omega(\log N/N), k = o(N)$.
Recall the definition of $r_k$ in \eqref{eq:def-rk}.
Then, for any $f\in C^4(\sM)$, when $N$ is sufficiently large, 
for any $x\in\sM$, w.p. higher than $1-2N^{-10}$, 
\begin{align}
    \left|  \bar{L}_\un f(x)  -   \sL_p  f(x)  \right| &=   O^{[f,p]}\left(\epsilon  + r_k\right)  +  O\left(  \|\nabla f\|_\infty p(x)^{1/d}  \sqrt{\frac{\log N}{N\epsilon^{d/2+1}}}\right).
\label{eq:error-barL-LM-un}
\end{align}
\end{proposition}

\paragraph{$\bullet$ Step 3. To bound $|L_\un f(x) - \sL_p f(x)| $}
The final bound in Theorem \ref{thm:conv-un-Laplacian-case1} follows from the previous two steps,
i.e. Propositions \ref{prop:step1-diff} and \ref{prop:step2-diff},
and the triangle inequality.

\section{Theoretical extensions}\label{sec:theory-extend}

Going beyond the unnormalized graph Laplacian $L_\un$, we extend the result to normalized (random-walk) graph Laplacian,
and also with density correction that leads to the weighted Laplacian limiting operator.
All the proved rates are the same as in Theorem \ref{thm:conv-un-Laplacian-case1}, i.e., the fast rate.
We also discuss the extension to eigen-convergence results.
Proofs are left to Appendix \ref{app:proof-theory-extend}.

\subsection{Random-walk graph Laplacian}\label{subsec:Lrw}

With $W$ as in \eqref{eq:def-W}, we consider the random-walk graph Laplacian defined as 
\begin{equation}\label{eq:def-L-rw-matrix}
	L_\rw \coloneqq - \frac{1}{\frac{m_2}{2m_0}}
	\diag\left( \left\{ \frac{1}{\epsilon \hrho(x_i)^2} \right\}_i \right) (I - D^{-1}W) \in \R^{N \times N}.
\end{equation}
$L_\rw f$, the operator applied to $f: \sM \to \R$, can be written as 
\begin{equation}\label{eq:def-L-rw}
	L_\rw f(x) \coloneqq \frac{1}{\frac{m_2}{2m_0}\epsilon\hrho(x)^2}\left(\frac{\sum_{j=1}^N k_0 \left(  \frac{\| x - x_j\|^2}{ \epsilon \phi(\hrho(x), \hrho(x_j))^2 } \right)f(x_j)}{\sum_{j=1}^N k_0 \left(  \frac{\| x - x_j\|^2}{ \epsilon   \phi(\hrho(x), \hrho(x_j))^2   } \right)}-f(x)\right), \quad x\in\sM. 
\end{equation}

By definition, we have 
$ L_\rw f(x) = \frac{L_\un f(x)}{   D(x)  },  $
where
\begin{align} \label{eq:def-D}
	D(x) \coloneqq \frac{1}{m_0 N} \sum_{j=1}^N \epsilon^{-d/2} k_0 (  \frac{\| x - x_j\|^2}{ \epsilon   \phi(\hrho(x), \hrho(x_j))^2   } ), \quad x \in \sM.
\end{align}
Based on the proof in Section \ref{sec:un-GL-fast},
the main idea here is to handle the denominator $D(x)$, which turns out to approximate $1$.
As a result, we can prove the convergence of $L_\rw$ to the same limiting operator as $L_\un$ and at the same rate.

\begin{theorem}[Fast rate of $L_\rw$]\label{thm:conv-rw-Laplacian}

Under the same condition as in Theorem \ref{thm:conv-un-Laplacian-case1}, 
 for any $f\in C^4(\sM)$,
 when $N$ is sufficiently large, for any $x\in\sM$,
	w.p. higher than $ 1- 11N^{-10} $, the same bound as in \eqref{eq:fast-rate} holds for $ | L_\rw f(x) - \sL_pf(x)|$.
When $\epsilon \sim N^{-2/(d+6)}$ and $k \sim N^{6/(d+6)}$, the same bound as in \eqref{eq:fast-rate-N} holds for $|  L_\rw f(x) -  \sL_p f(x) |$.
\end{theorem}

\subsection{Density-corrected graph Laplacians}\label{sec:main-results-norm}

We introduce a normalized graph affinity as 
\begin{equation} \label{eq:def-Wtilde}
	\widetilde{W}_{ij} :=  \frac{ W_{ij} }{\phi(\hrho(x_i), \hrho(x_j))^2 },
\end{equation}
and the corresponding (un-normalized) graph Laplacian matrix is 
\begin{equation*}  
 \widetilde{L}_\un \coloneqq  - \frac{1}{\frac{m_2}{2}N \epsilon }  (\widetilde{D} - \widetilde{W}) \in \R^{N\times N}, 
\end{equation*}
where $\widetilde{D}_{ii} = \sum_j \widetilde{W}_{ij}$.
The operator of applying $\widetilde{L}_\un$ to $f$ is given by 
\begin{equation} \label{eq:def-L-un-tilde}
\widetilde{L}_\un f(x) \coloneqq \frac{1}{\frac{m_2}{2}N\epsilon}\sum_{j=1}^N \epsilon^{-d/2}
k_0 \left(  \frac{\| x - x_j\|^2}{ \epsilon \phi(\hrho(x), \hrho(x_j))^2 } \right) \frac{f(x_j)-f(x)}{ \phi(\hrho(x), \hrho(x_j))^2},  \quad x\in\sM. 
\end{equation}
The next theorem establishes the fast rate of $\widetilde{L}_\un$ converging to $\Delta_p$ defined in \eqref{eq:def-Delta-p}.
The proof follows the same three-step strategy as in Section \ref{subsec:proof-overview-Lun}.

\begin{theorem}[Fast rate of $\widetilde{L}_\un$]
	\label{thm:conv-un-Laplacian-hrho-norm}
Under the same condition as in Theorem \ref{thm:conv-un-Laplacian-case1}, 
 for any $f\in C^4(\sM)$,
when $N$ is sufficiently large, for any $x \in \sM$,  
w.p. higher than $ 1- 5N^{-10} $,  the same bound as in \eqref{eq:fast-rate} holds for $|   \widetilde{L}_\un f(x) -  \Delta_p f(x)  |$.
When $\epsilon \sim N^{-2/(d+6)}$ and $k \sim N^{6/(d+6)}$, the same bound as in \eqref{eq:fast-rate-N}  holds for $|   \widetilde{L}_\un f(x) -  \Delta_p f(x)  | $.
\end{theorem}

We compare our result with \cite{cheng2022convergence} which studied the $k$NN (self-tuned) graph Laplacian corresponding to when $\phi(u,v) = \sqrt{uv}$.
The setup therein utilized a separate dataset $Y$ to compute $k$NN bandwidth $\hrho_Y(x_i)$ for $x_i \in X$, while in practice it is more common to compute $k$NN on the dataset $X$ itself as we have analyzed in this paper. 
In addition, our result improves the point-wise convergence rates in \cite{cheng2022convergence}.
Specifically, the replacement error is improved thanks to our refined $k$NN estimation which compares $\hat \rho$ to $\bar \rho_{r_k}$ in Section \ref{sec:knn-estimation}.  
This allows us to show the overall graph Laplacian convergence rate (the fast-rate case) to be $O(N^{-2/(d+6)})$ up to a log factor, which matches the point-wise convergence rate for fixed bandwidth in the literature \cite{singer2006graph, cheng2022eigen}.
Our result also covers a broader class of kernel affinities that includes the self-tuned kernel as a special case.

Finally, similarly as in Section \ref{subsec:Lrw}, one can consider the random-walk graph Laplacian $\widetilde{L}_\rw$ constructed using $\widetilde D^{-1} \widetilde W$;  its explicit definition is given in Appendix \ref{app-subsec:tildeLrw}.
Combining the convergence of $\widetilde{L}_\un$ established in Theorem \ref{thm:conv-un-Laplacian-hrho-norm} with the argument used in  Theorem \ref{thm:conv-rw-Laplacian},
one obtains convergence of $\widetilde{L}_\rw$ to $ \Delta_p$ at the same rate. The proof is omitted.

\subsection{The eigen-convergence problem}\label{subsec:eigen-rate}

We have been focused on the operator pointwise convergence in our theoretical results.
Another related and important question is the eigen-convergence,
namely how the eigenvalues and associated eigenvectors of the graph Laplacian matrix $L_N$ converge to
the eigenpairs of the limiting manifold operator $\calL$.
To see the relationship between the two, 
first, operator point-wise convergence rates were utilized as a middle step toward proving the eigen-convergence rate \cite{calder2022improved, cheng2022eigen}.
Specifically, pointwise rate bounds $ L_N f  - \calL f  $ in $\infty$-norm,
which, applied to when $f$ equals an eigenfunction of $\calL$,
can be used to prove  eigenvector consistency (in $L^2$-norm) up to the same pointwise rate, see \cite{calder2022improved}.
This means that, with an $N^{-2/(d+6)}$ pointwise convergence rate,
one expects to obtain the same $N^{-2/(d+6)}$ rate for eigenvector $L^2$ convergence (all rates here are up to log factor). 
Combined with an analysis of the graph Dirichlet form, \cite{cheng2022eigen} further showed  an eigenvalue convergence rate of $N^{-2/(d+4)}$.

Back to this work, with the $N^{-2/(d+6)}$  pointwise convergence rate proved for $k$NN graph Laplacians here, 
we expect the same eigen-convergence rates as in \cite{cheng2022eigen} following these prior arguments.
Technically, the proof of Theorem \ref{thm:conv-un-Laplacian-case1} can be extended to bound $\left| L_\un f(x_i)  -  \sL_p f(x_i)  \right|$ for an $x_i$ in $X$ (instead of a fixed $x \in \calM$), by some additional analysis.
(The standard trick of conditioning on $x_i$ will not work directly here due to the joint dependence of the $k$NN bandwidth at other $x_j$'s, but it does apply after the deterministic replacement in Lemma \ref{lemma:step1-diff},
so one instead uses the fact that the replacement argument in Proposition \ref{prop:step1-diff} still holds when $x = x_i$, detailed in Remark \ref{rk:replace-lemma1-xi}.)
This will enable us to bound the vector $\{ L_N f(x_i) - \calL f(x_i) \}_{i=1}^N$ in vector $\infty$-norm,
which readily implies the eigenvector 2-norm convergence bound.

Meanwhile, we note that eigen-convergence can be obtained via other techniques than pointwise bounds and potentially achieve a faster rate, such as in \cite{wormell2021spectral} and recently in \cite{trillos2025minimax}. 
These eigen-convergence results, however, do not imply pointwise convergence of the graph Laplacian operator nor the same rate in the $\infty$-norm.
In particular, convergence rates for eigenvalues and eigenvectors measured in integral norms such as $L^2$ do not  generally control $L_Nf-\mathcal Lf$ in the $\infty$-norm. In applications, eigen-convergence can provide the consistency of spectral embedding for unsupervised learning,
whereas pointwise convergence is useful when one needs local estimates
or error guarantees at individual sample locations, as often arise in
manifold denoising and regression.
We thus view eigen-convergence as a different type of result that is complementary to the pointwise convergence studied here,
and we postpone its investigation to future work.

\section{Numerical experiments}\label{sec:nume-exp}

We numerically investigate the  $k$NN bandwidth function
and
the operator point-wise convergence of graph Laplacian
on simulated manifold data. 
Code is available at \url{https://github.com/xycheng/knn_GL}.

\subsection{$k$NN bandwidth function} \label{sec:nume-exp-knn}

\paragraph{Dataset}
We simulate $N$ i.i.d. data samples from a one-dimensional closed curve $\calM$ embedded in $\R^4$, thus $d=1$.
The density $p$ of data samples is shown in Figure \ref{fig:data-1d}(b), which is non-uniform along the curve. 
The data manifold curve is as in Figure \ref{fig:data-1d}(a). 
Details of the example can be found in Appendix \ref{app:exp-data}.

\paragraph{Method}
We simulate $N = 2000$ samples $\{ x_i \}_{i=1}^N$ and 
let $k = 32$ and $64$. 
At each $x_i$, $\hrho(x_i)$ is calculated from the $k$NN distance $\hat R(x_i)$ by the definition \eqref{eq:def-hat-rho}. 
In addition, we compute $\brho(x_i) = p(x_i)^{-1/d}$ using the analytical expression of $p$. 
We can also obtain the expression of $\brho_{r_k}$ 
defined in Definition \ref{def:bar-rho-epsilon}
making use of the expression of the manifold embedding mapping. 
The details are described in Appendix \ref{app:comp-brho-r}.
We then compute the values of $\brho_{r_k}(x_i)$.

\begin{figure}[t]
\centering
\includegraphics[width=0.85\linewidth]{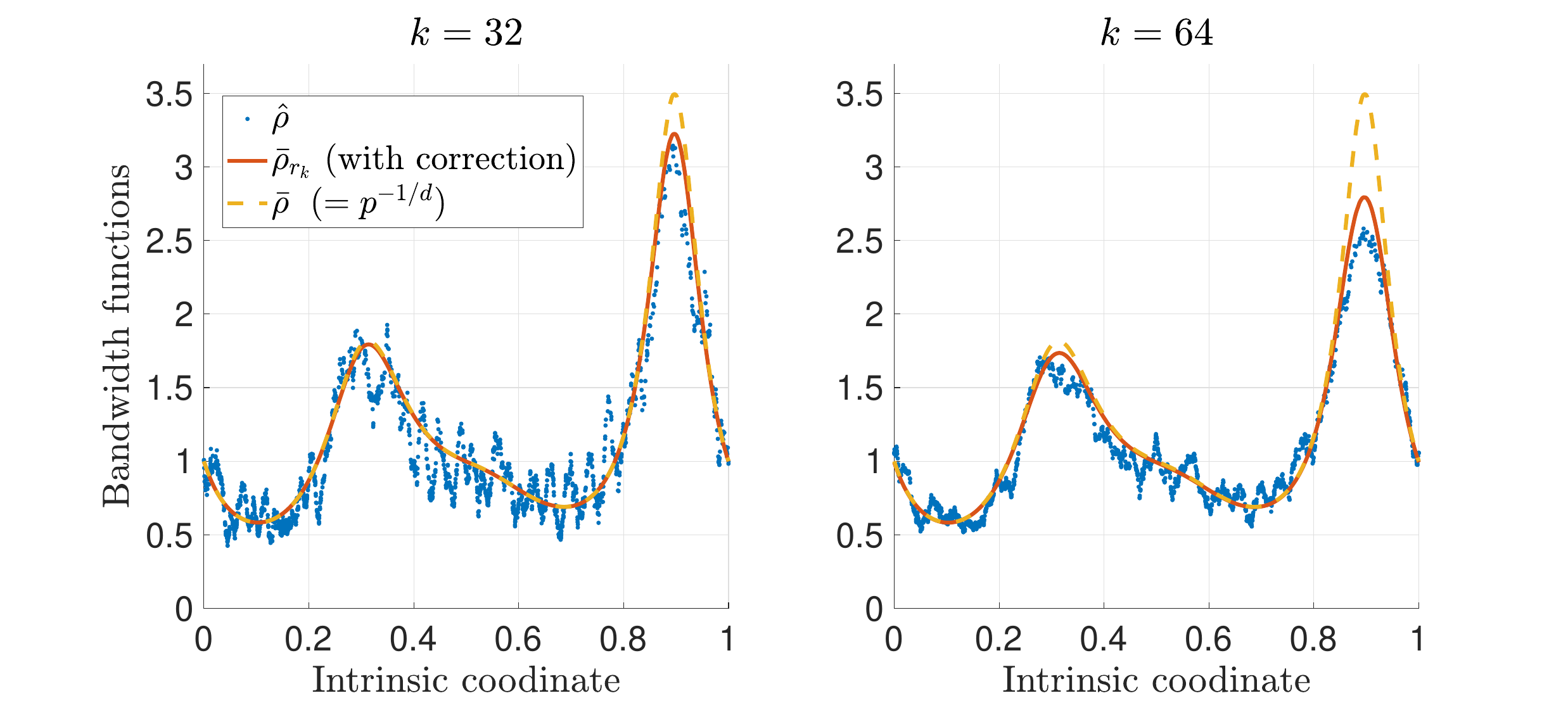}
\vspace{-5pt}
\caption{
\small
The empirical $k$NN bandwidth function $\hrho$  defined in \eqref{eq:def-hat-rho} (marked with blue dots)
computed from  $N = 2000$ samples on a one-dimensional curve 
where $k= 32$ (Left) and 64 (Right). 
Compared with 
the population bandwidth function $\brho_{r_k}$ defined as in Definition \ref{def:bar-rho-epsilon} (marked in red solid line) 
and $\brho = p^{-1/d}$ (marked in orange dashed line). 
}
\label{fig:hrho-est}
\end{figure}

\paragraph{Result}

Figure \ref{fig:hrho-est} shows a typical realization of $\hrho$ and compares it to $\brho_{r_k}$ and  $\brho $, with two different choices of $k$. It can be seen that the correction we introduced makes $\brho_{r_k}$  closer to  $\hat \rho$ than the original $\brho$.
Moreover, when $k$ is larger, the oscillation of $\hrho$ is smaller (due to reduced variance error), and the effect of improvement is more significant. This is consistent with our theory.

\subsection{Point-wise convergence of graph Laplacians}  \label{sec:nume-exp-GL}

\paragraph{Dataset}
We use the same data example as the one in Section \ref{sec:nume-exp-knn}.
To study the operator point-wise convergence,
we specify a test function $f$.
The analytical expression of $f$ and $\Delta_p f$ can be found in Appendix \ref{app:exp-data},
and the plots of the two functions are shown in Figure \ref{fig:data-1d}.

\paragraph{Method}

We calculate $\widetilde{L}_\rw f(x_0)$ at a fixed point $x_0 \in \sM$,
and use approximately 4800 random points within a local neighborhood around $x_0$. 
Details of the choice of the local neighborhood are in Appendix \ref{app:choice-neighhood}. 
$\sigma_0^2$ takes values on a grid that ranges from $0.06$ to $1.54$ (evenly spaced on the $\log_{10}$ scale).
We compute for $k \in \{64, 128, 256, 512\}$,
and only the results for $k=512$, which gives the best overall performance, are presented.
The results for other choices of $k$ are similar. 

We investigate five types of kernels: 
(i) $k_0(\eta) = \exp( - \eta / 4 )$ and self-tuned $\phi$;
(ii) $k_0(\eta) = \exp( - \eta / 4 )$ and squared-mean $\phi$;
(iii) $k_0(\eta) = \exp( - \eta / 4 )$ and min $\phi$;
(iv) $k_0(\eta) = \mathbf{1}_{[0,3]}(\eta)$ and self-tuned $\phi$;
(v) $k_0(\eta) = \mathbf{1}_{[0,3]}(\eta)$ and min $\phi$.
Among these five kernels, (i) and (ii) satisfy 
our fast-rate assumption, namely Assumptions \ref{assump:phi-diff} and \ref{assump:k0-smooth}.
The graph affinity values on the neighborhood around $x_0$ computed by the five kernels are shown in Figure \ref{fig:kernel}.

The value of  $\Delta_p f(x)$ at any $x$ is analytically available.
The  approximation error at $x_0$ is
\begin{equation}\label{eq:def-L1-err}
  \text{Err} := |\widetilde{L}_\rw f(x_0) - \Delta_p f(x_0)|. 
\end{equation}
We report the average error over 2000 replicas of the experiment.
In addition, we also generate evenly-spaced data points along the curve and compute 
$\bar{L}_{\rw}^{\rm (even)} f(x_0)$
by the same expression as \eqref{eq:def-L-rw-tilde} except that the summations over $j$ are weighted by $p(x_j)$.
This can be viewed as computing the kernel integrals by Riemann sums on an even grid. 
We then compute the error
\begin{equation}\label{eq:def-L1bar-err}
	\overline{\rm Err}: =|\bar{L}_{\rw}^{\rm (even)} f(x_0) - \Delta_p f(x_0)|.
\end{equation}
As we use a refined even grid, $\overline{\rm Err}$ represents ``bias error'' due to using a finite kernel bandwidth $\epsilon$.
We thus use the value of $\overline{\rm Err}$ to measure the bias error incurred by the graph Laplacian method.

\begin{figure}[ptb]
\centering{
\includegraphics[width=0.9\linewidth]{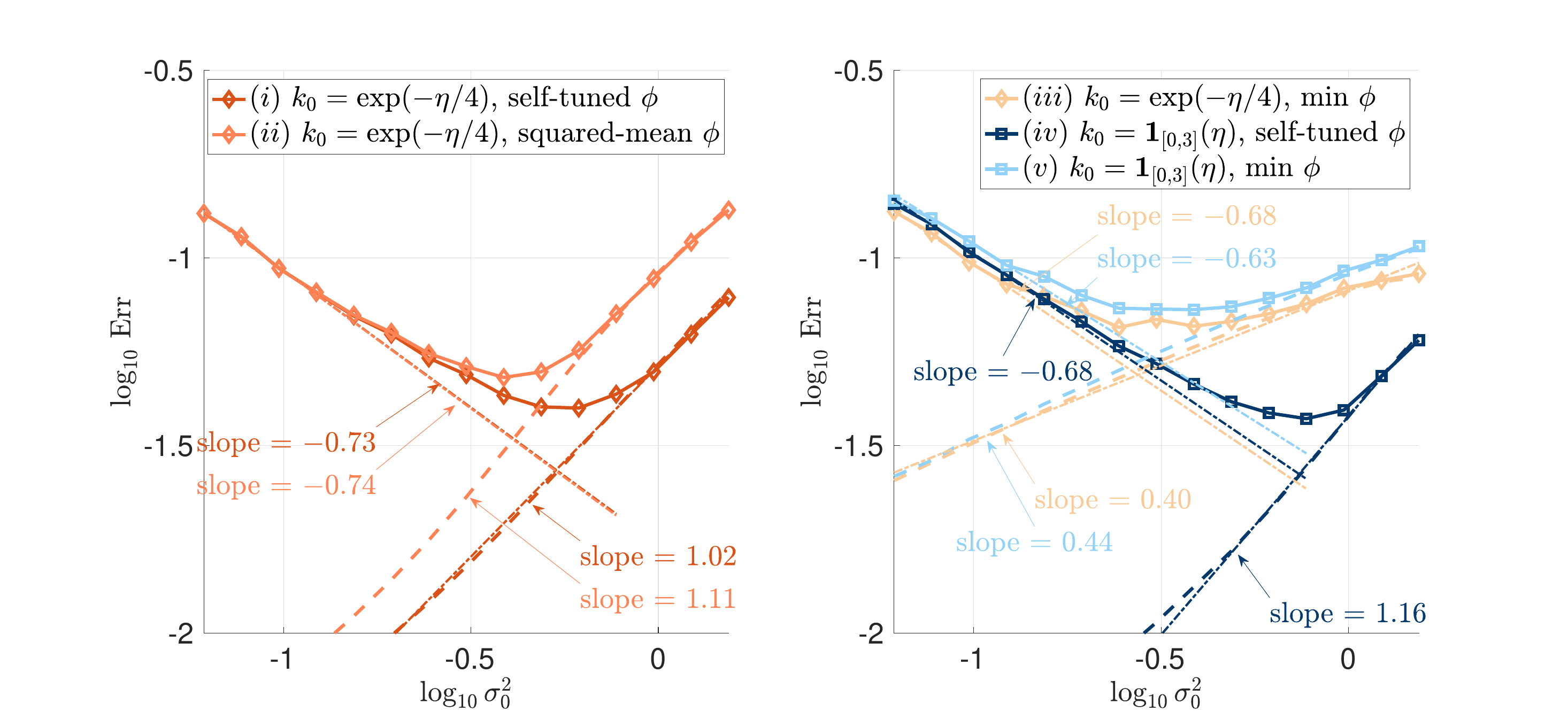}
}
\vspace{-5pt}
\caption{
\small
The errors of different kernels plotted against values of $\sigma_0^2$,
where $\text{Err}$ defined in \eqref{eq:def-L1-err}
(averaged over $2000$ runs) is shown in solid curves,
and $\overline{\rm Err}$ defined in \eqref{eq:def-L1bar-err} is shown in dashed lines.
The fitted slopes on the log-log plots are also shown. 
(Left) 
Theoretical fast-rate cases. 
(Right) 
Other cases. 
}
\label{fig:converg-rates}
\end{figure}

\paragraph{Result}

The averaged values of $\mathrm{Err}$ are plotted in Figure \ref{fig:converg-rates}
over a range of $\sigma_0$ for different choices of kernels. 
We plot against $\sigma_0^2$ because, with fixed $k$ and $N$, it is proportional to $\epsilon$, 
and then we can compare with our theoretical results. 
The dashed lines show the values of  $\overline{\rm Err}$.
In both (left and right) plots, at smaller values of $\sigma_0$ where the variance error dominates, all types of kernels exhibit errors scaling approximately as $\epsilon^{-d/4-1/2} = \epsilon^{-0.75}$ (since $d=1$). 
As $\sigma_0$ increases, the bias error starts to dominate. 
Kernels utilizing smooth $\phi$ exhibit errors that approximately scale as $\epsilon$, while those using min $\phi$ show errors that approximately scale as $\epsilon^{1/2}$. 

In summary, kernels (i)(ii)(iv) show errors approximately $O(\epsilon + \epsilon^{-d/4-1/2})$, 
and  kernels (iii)(v) approximately as $O(\sqrt{\epsilon} + \epsilon^{-d/4-1/2})$.
Recall that  (i)(ii) are theoretically fast-rate kernels.
This is consistent with our theory that regularity of $\phi$ and $k_0$ can improve the bias error in graph Laplacian approximation.
Note that (iv) presents a fast-rate behavior, which suggests that the analysis may be further improved.

\subsection{Spectral embedding}

\begin{figure}[t]
\centering
\includegraphics[width=1\linewidth]{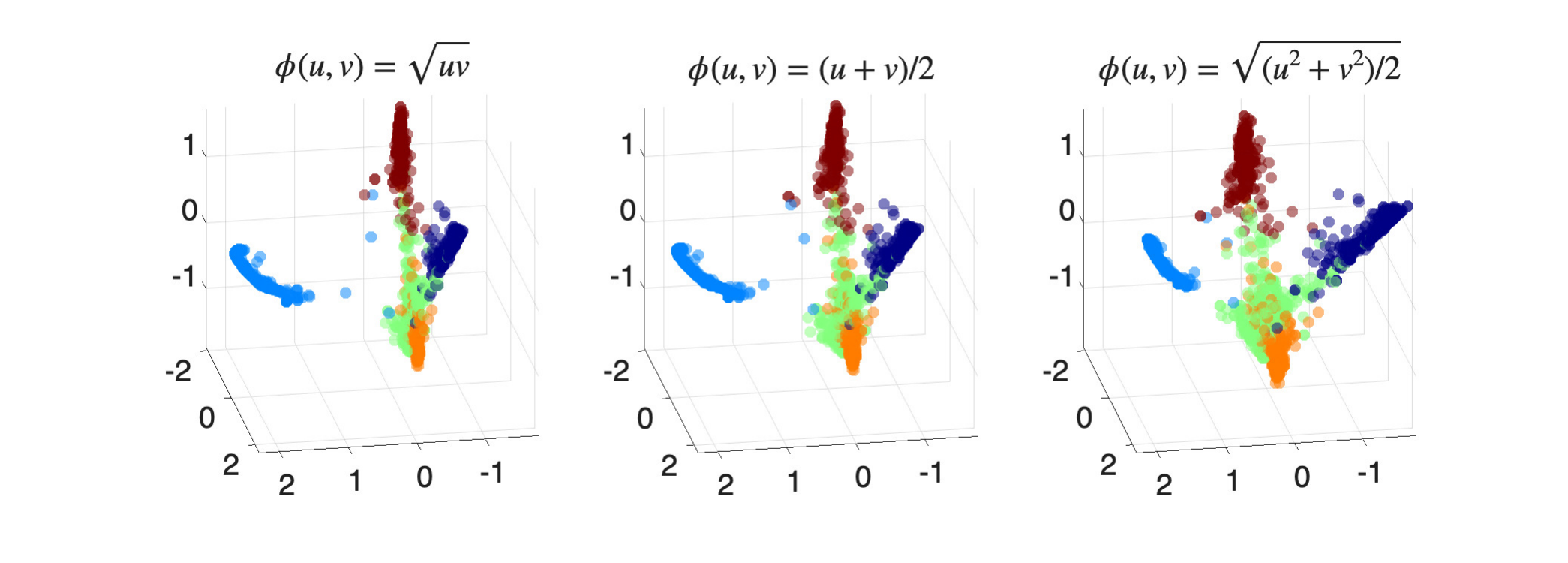}
\vspace{-20pt}
\caption{
\small
Spectral embeddings of the MNIST dataset obtained using three symmetric bandwidth functions $\phi$,
colored by class labels.
}
\label{fig:mnist-embed}
\end{figure}

\paragraph{Data and method.}
We further investigate the effect of the symmetric bandwidth function $\phi$ on spectral embedding using the MNIST handwritten digit dataset. We randomly sample $500$ images from each of the digits `0', ... ,`4', resulting in a dataset of $2500$ images. 
Each image is  $28\times28$ gray-scale
 and represented as a vector in $\mathbb{R}^{784}$. 
 We use $k=7$, and 
 (i) $\phi(u,v)=\sqrt{uv}$,
(ii) $\phi(u,v)=(u+v)/{2}$,
(iii) $\phi(u,v)=\sqrt{(u^2+v^2)/{2}}$.
 The graph affinity is defined by $\widetilde W$ as in \eqref{eq:def-Wtilde} with bandwidth parameter $\epsilon=0.25^2$.
We construct the corresponding random-walk graph Laplacian 
and compute the first three nontrivial graph eigenvectors.
These eigenvectors are then used as coordinates for the spectral embedding. 

\paragraph{Result}
Figure~\ref{fig:mnist-embed} 
shows that all three bandwidth functions produce similar spectral embeddings with clear class separation. Although the cluster shapes vary slightly, each choice of $\phi$ yields a meaningful low-dimensional representation of the handwritten digits. This demonstrates that meaningful spectral embeddings are obtained not only with the self-tuned bandwidth (i), but also with other symmetric bandwidth functions.

\section{Discussion}

\paragraph{Slow rate} 
If one considers $k_0$ or $\phi$ with less regularity,
then our technique allows us to prove a slower rate for the $k$NN graph Laplacians.
Specifically, 
if $\phi$ remains under  Assumption \ref{assump:phi-diff} and $k_0 = \mathbf{1}_{ [0,1] }$,
or $k_0$ remains under Assumption \ref{assump:k0-smooth} and $\phi(u, v) = \max\{u, v\}$ or $\phi(u, v) = \min\{u, v\}$,
then one can prove  a bound by replacing  $O(\epsilon)$ with $O(\sqrt{\epsilon})$ in the 
first term in the bias error in \eqref{eq:fast-rate},
and this leads to a final rate of $O (N^{-1/(d+4)} \sqrt{\log N})$. 
The argument extends to when $k_0$ is compactly supported and  Lipschitz  on $[0,1]$.
This recovers the same rate as previously shown in \cite{calder2022improved}, 
while our detailed bound separates the contribution from $\epsilon$ and $k$
and covers a broader range of graph affinities.

\paragraph{Extension to connection Laplacians}
The pointwise analysis developed here can also be extended to graph
connection Laplacians and vector diffusion maps, following the local
parallel-transport formulation in \cite{singer2017spectral} built upon \cite{singer2006graph}.
 For a fixed evaluation point, a smooth section can be represented in a local frame obtained by
parallel transport, reducing the analysis locally to the vector-valued counterpart of the scalar problem considered here. 
The bandwidth replacement, bias, and concentration arguments then proceed in the same
manner, leading to analogous pointwise convergence results for $k$NN graph connection Laplacians.

\paragraph{Other extensions}

One can consider a symmetric construction of $k$NN graph Laplacian, while $k$NN bandwidth can also give rise to an asymmetric affinity, e.g., by $W_{ij} = k_0( \|x_i- x_j\|^2/\hat R(x_i)^2)$. This will render an $N$-by-$N$ affinity matrix asymmetric, but symmetric affinity can be restored
 if one constructs the asymmetric affinity $A_R$ between $X$ and a ``reference set'' $R$ (which can be subsampled from $X$) and then use $A_R A_R^T$ as the graph affinity \cite{bermanis2016measure, cheng2020two, shen2022scalability}. 
Meanwhile, we consider clean manifold data where data samples lie exactly on the manifold. The extension to include high-dimensional noise and analyze noise robustness will be of practical importance. 
Another interesting direction is to extend to manifolds with boundary, possibly by utilizing techniques in \cite{kuo2024boundary}.
One may also combine the techniques developed here with recent eigen-convergence techniques, such as those in \cite{trillos2025minimax},
to study eigen-convergence for graph Laplacians with empirically estimated $k$NN bandwidths.
Finally, it would be interesting to explore further applications of the method.

\section*{Acknowledgments}
We thank the anonymous reviewers for their valuable feedback, which has helped improve this work.
The work was supported partially by NSF DMS-2237842, 
NSF DMS-2007040, 
and the Simons Foundation MPS-MODL-00814643.

\bibliographystyle{plain}
\bibliography{ref}


\setcounter{figure}{0}
\renewcommand{\thefigure}{A.\arabic{figure}}

\setcounter{table}{0}
\renewcommand{\thetable}{S\arabic{table}}

\setcounter{equation}{0} \renewcommand{\theequation}{A.\arabic{equation}}

\setcounter{assumption}{0}
\renewcommand{\theassumption}{A.\arabic{assumption}}

\appendixtitleon
\appendixtitletocon
\begin{appendices}

\section{Proofs in Section \ref{sec:knn-estimation}}  
\label{sec:proof-sec-rho}

We use $B_r(x)$ to denote the (open) Euclidean ball, and $\bar{B}_r(x)$ is the closure,
for some radius $r > 0$.
To emphasize the dimension, we use superscript $B_r^{m}(x)$ to stand for the Euclidean ball in $\R^m$,
and similarly $B_r^{d}(x)$ means the Euclidean ball in $\R^d$.

\begin{proof}[Proof of Lemma \ref{lemma:bar-rho-epsilon}]
We prove the statements (i)-(iv) respectively. 

\vspace{0.5em}
\noindent
$\bullet$ Proof of (i)-(ii): 
We first show that,  for a fixed $r \in [0, r_0]$, $\brho_r(x)$ is well-defined at each $x\in \calM$.
For a fixed $r \ge 0$, we define
\begin{align} \label{eq:def-F(t,x)}
	F(t, x) \coloneqq t^d+r Q(x)t^{d+2} - \frac{1}{p(x)},
\end{align}
which is $C^1$ on $\R_+ \times \sM$, and for any fixed $x \in \sM$ it is a polynomial of $t$. 
Recall the definitions of $\rho_{\min}$, $\rho_{\max}$ in \eqref{eq:def-rhomin-rhomax}. 
We will consider $F(t,x)$ restricted to $[0, \rho_{\max}] \times \sM$.
We claim that if $r \leq r_0$, then 
\begin{align} \label{eq:property-F}
	\forall x \in \sM, \, \partial_t F(t, x) > 0, \forall t \in (0, \rho_{\max}], \text{ and }
	F(t_{\min}(x), x) < 0, F(t_{\max}(x), x) > 0, 
\end{align}
where $t_{\min}(x) \coloneqq (2 / (3p(x)))^{1/d}$ and $t_{\max}(x) \coloneqq (2/p(x))^{1/d} $. Note that by the definitions of $\rho_{\min}$ and $\rho_{\max}$, we have $0< \rho_{\min} \le t_{\min}(x) < t_{\max}(x) \le \rho_{\max}$ for any $x \in \sM$,
then, assuming claim \eqref{eq:property-F} holds, we have $\partial_t F(t, x) > 0$ for $t \in (t_{\min}(x), t_{\max}(x))$. Combined with the fact that $F(t_{\min}(x), x) < 0, F(t_{\max}(x), x) > 0$, we have that for any $x \in \sM$, there is a unique $\tilde{t}  \in (t_{\min}(x), t_{\max}(x))$ such that $F( \tilde{t} , x) = 0$. This $\tilde{t}$ is thus a function of $x$, which we denote as $f$,
and 
\begin{align} \label{eq:bound-barrho}
	f(x) \in (t_{\min}(x), t_{\max}(x)).
\end{align}
Thus, $f(x) \in (\rho_{\min}, \rho_{\max}) \subset (0, \rho_{\max})$.
While $f$ also depends on $r$, we omit this dependence here since we fixed an $r$ to begin with.
This function $f$ on $\sM$ is then defined as $\brho_r(x)$, which is equivalent to Definition \ref{def:bar-rho-epsilon}.

If $p \in C^l(\sM)$ for some $l \geq 3$, then by the definition of $Q$ in \eqref{eq:def-bar-rho-epsilon},
we have that $Q \in C^{l-2}(\sM)$: this is because ${\Delta p}/{p} \in C^{l-2}(\sM)$ and $\omega$ as defined in \eqref{eq:def-omega} in Lemma \ref{lemma:G-expansion-h-indicator} is $C^\infty$ on $\sM$. 
Therefore, $F(t, x)$ defined in \eqref{eq:def-F(t,x)} is $C^{l-2}$ on $\R_+ \times \sM$.
Then, the fact that $\brho_r \in C^{l-2}(\sM)$ follows from the Implicit Function Theorem. To see this, one can use the definition of differentiable functions on $\sM$ to verify the differentiability of $\brho_r$ with respect to $x \in \sM$. This proves (i).
To verify (ii), one can use the fact that \eqref{eq:bound-barrho} implies that $\brho_r(x) \in (t_{\min}(x), t_{\max}(x))$, where the endpoints $t_{\min}(x)$ and $t_{\max}(x)$ are defined to be $(2 / (3p(x)))^{1/d}$ and $(2 / p(x))^{1/d}$.

Therefore, if the claim in \eqref{eq:property-F} holds, then (i)-(ii) are proved.
Below, we prove  \eqref{eq:property-F}.
First, if $0\leq r \leq r_0 $, then for $0 < t \leq \rho_{\max}$, 
\begin{align*}
\partial_t F(t, x)  &= dt^{d-1}\left(1+r Q(x)\frac{d+2}{d}t^2\right)\geq dt^{d-1}\left(1-r\|Q\|_\infty\frac{d+2}{d}t^2\right) \\
&\geq dt^{d-1}\left(1-r_0\|Q\|_\infty\frac{d+2}{d}\rho_{\max}^2\right) > \frac{1}{2} dt^{d-1} > 0.  
\end{align*}
The second last inequality follows from the definition of $r_0$ as in \eqref{eq:def-r0-tilder0}.
Similarly, for $t_{\max}(x) = (2/p(x))^{1/d} > 0$, we have
\begin{align*}
F(t_{\max}(x), x) 
 &\geq t_{\max}(x) ^d\left(1-r\|Q\|_\infty t_{\max}(x) ^2\right)-\frac{1}{p(x)}  
 \geq  t_{\max}(x) ^d\left(1-r_0\|Q\|_\infty \rho_{\max} ^2\right)-\frac{1}{p(x)} \\
 &> \frac{1}{2}t_{\max}(x)^d-\frac{1}{p(x)} = \frac{1}{p(x)}-\frac{1}{p(x)} = 0,
\end{align*}
and for $t_{\min}(x) = (2/(3p(x)))^{1/d} > 0$, we have
\begin{align*}
	F(t_{\min}(x), x) 
	&\leq t_{\min}(x) ^d\left(1 + r\|Q\|_\infty t_{\min}(x) ^2\right)-\frac{1}{p(x)}  
	\leq  t_{\min}(x) ^d\left(1 + r_0\|Q\|_\infty \rho_{\max} ^2\right)-\frac{1}{p(x)} \\
	&< \frac{3}{2}t_{\min}(x)^d-\frac{1}{p(x)} = \frac{1}{p(x)}-\frac{1}{p(x)} = 0.
\end{align*}
This proves the claim in  \eqref{eq:property-F}, and in turn proves (i)-(ii).

\vspace{0.5em}
\noindent
$\bullet$ Proof of (iii):

Recall that $\brho = p^{-1/d}$,
together with the definition of $\brho_r$ in \eqref{eq:def-bar-rho-epsilon},
we have
$ \brho_r^d ( 1 + r Q \brho_r^2)  =  \brho^d$.
Then, for any $x\in \sM$, we have 
$|\brho_r(x)^d -  \brho(x)^d | = |Q(x)\brho_r(x)^{d+2}| r $.
Using the upper bound $\brho_r(x)  \leq  \rho_{\max}$ from (ii) which was established,
we then have 
$|\brho_r(x)^d -  \brho(x)^d | \leq C_p r $,  where $C_p \coloneqq 
(\|Q\|_\infty +1) \rho_{\max}^{d+2} > 0$. 
Meanwhile, since
\begin{align} \label{eq:bound-rhobar-r^d-rhobar^d}
	 |\brho_r(x)^d -  \brho(x)^d | =  | \brho_r(x) - \brho(x) | \sum_{j=0}^{d-1} \brho_r(x)^j \brho(x)^{d-1-j},
\end{align}
we will derive a lower bound of the term $\sum_{j=0}^{d-1} \brho_r(x)^j \brho(x)^{d-1-j}$ so as to upper bound $| \brho_r(x) - \brho(x) |$.
To do this, for $\rho_{\min}$ defined in \eqref{eq:def-rhomin-rhomax}, we have that $\brho(x) \ge (1/p_{\max})^{1/d} > \rho_{\min}$ from the definition of $\brho$
and $\brho_r(x)  \geq  \rho_{\min}$ from (ii).
Thus, for each $j = 0, \dots, d-1$, we have $ \brho_r(x)^j \brho(x)^{d-1-j} \ge \rho_{\min}^{d-1}$.
This implies that 
 the term $\sum_{j=0}^{d-1} \brho_r(x)^j \brho(x)^{d-1-j}$ is lower bounded by $d\rho_{\min}^{d-1}$.
Therefore, using the upper bound $|\brho_r(x)^d -  \brho(x)^d | \leq C_p r $ and the lower bound $\sum_{j=0}^{d-1} \brho_r(x)^j \brho(x)^{d-1-j} \ge d\rho_{\min}^{d-1}$, 
by \eqref{eq:bound-rhobar-r^d-rhobar^d},
we have
$ | \brho_r(x) - \brho(x) | \leq C_0(p) r$, where $C_0(p) \coloneqq C_p / (d\rho_{\min}^{d-1}) > 0$, that is, $ \| \brho_r - \brho \|_\infty \leq C_0(p) r $.
Here, $C_0(p)$ depends on $\sM$ and the derivatives of $p$ up to the second order.

\vspace{0.5em}
\noindent
$\bullet$ Proof of (iv):

For $f \in C^k(\sM)$, recall the definitions of $\| \nabla^k f|_x\|$ and $\| \nabla^k f\|_\infty$ in Definition \ref{def:Dkf-norm}, and the other notations for the manifold derivative in Section \ref{sec:prelim}. 

For each $l=1,2,3$, we assume that $p\in C^{l+2}(\sM)$. From the established (i), $\brho_r\in C^l(\sM)$, so the derivatives needed in the corresponding case below are well-defined and bounded.
It suffices to prove the bound for $\| \nabla^l (\brho_r - \brho) |_x \|$ for $x \in \sM$. Then, by Definition \ref{def:Dkf-norm}, the bound for  $\| \nabla^l (\brho_r - \brho) \|_\infty $ follows. 

We define $\tilde{Q} \coloneqq Q \circ \exp_x$, $\trho \coloneqq \brho \circ \exp_x$, and $\trho_r \coloneqq \brho_r \circ \exp_x$.
In this proof, for any vector $v \in  T_x \sM $, we also use $v$ to denote the vector in $\R^d$.

Let $l = 1$. 
For any $v  \in T_x\sM$ such that $\|v\| \leq 1$, by \eqref{eq:nabla^k-f-norm-coord-v},
\begin{align*}
	&\nabla \left(\brho_r^d - \brho^d \right) \big|_x \left(v\right)  = D( \trho_r^d - \trho^d )(0)(v) = d \trho_r(0)^{d-1} D \trho_r (0)(v) - d \trho(0)^{d-1} D \trho(0)(v) \\
	&\nabla \left( 	- r Q \brho_r^{d+2} \right) \big|_x \left( v \right) = D(- r \tilde{Q} \trho_r^{d+2})(0)(v) = - r \left(  (d+2) \tilde{Q}(0) \trho_r(0)^{d+1} D \trho_r(0)(v)  + \trho_r(0)^{d+2} D \tilde{Q}(0)(v) \right).
\end{align*}
By the bound for $\| \brho_r  - \brho \|_\infty$ in (iii), $\|\brho_r^{d-1}  - \brho^{d-1}\|_\infty \leq (d-1) \rho_{\max}^{d-2} C_0(p) r$.
Then, combined with the triangle inequality and the lower bound of $\brho_r$ in (ii), we have
\begin{align*} 
	&\left| \nabla \left(\brho_r^d - \brho^d \right) \big|_x \left(v\right)   \right| = \left| d \trho_r(0)^{d-1} D \trho_r (0)(v) - d \trho(0)^{d-1} D \trho(0)(v) \right| \\
	&\geq d \trho_r(0)^{d-1}  \left| D \trho_r (0)(v) - D \trho (0)(v)\right| - d\left| D \trho (0)(v) \right|  \left|\trho_r(0)^{d-1} - \trho(0)^{d-1} \right|\\
	&\geq d \rho_{\min}^{d-1} \left| \nabla (\brho_r - \brho )|_x(v)   \right| -  d(d-1) \rho_{\max}^{d-2} \left| \nabla \brho |_x(v) \right| C_0(p) r  \notag \\
	&\geq d \rho_{\min}^{d-1} \left| \nabla (\brho_r - \brho )|_x(v)   \right| -  d(d-1) \rho_{\max}^{d-2} \| \nabla \brho \|_\infty C_0(p) r,
\end{align*}
where in the second inequality, we used \eqref{eq:nabla^k-f-norm-coord-v};
and in the last inequality, we used the definition of $\| \nabla \brho \|_\infty$ by Definition \ref{def:Dkf-norm}.
We denote by $C_{1,L}(p) \coloneqq d(d-1) \rho_{\max}^{d-2}  \| \nabla \brho \|_\infty C_0(p)$, then
\begin{align} \label{eq:proof-brhor-diff-1}
	\left| \nabla \left(\brho_r^d - \brho^d \right) \big|_x \left(v\right)   \right| \geq d \rho_{\min}^{d-1} \left| \nabla (\brho_r - \brho )|_x(v)   \right| -  C_{1,L}(p) r.
\end{align}
Similarly, by triangle inequality,
\begin{align*} 
	\left| \nabla \left( 	- r Q \brho_r^{d+2} \right) \big|_x \left( v \right)  \right|	
	&\leq  r (d+2) \| Q\|_\infty \rho_{\max}^{d+1} \left|  \nabla (\brho_{r} - \brho)|_x(v) \right| + r (d+2) \| Q\|_\infty \rho_{\max}^{d+1} \left| \nabla \brho|_x(v) \right| \\
	&\quad + r \rho_{\max}^{d+2} \left| \nabla Q|_x(v) \right|  \\
	&\leq r (d+2) \| Q\|_\infty \rho_{\max}^{d+1} \left|  \nabla (\brho_{r} - \brho)|_x(v) \right| + \big(  (d+2) \| Q\|_\infty \rho_{\max}^{d+1} \| \nabla \brho \|_\infty \\
	&\quad + \rho_{\max}^{d+2} \| \nabla Q \|_\infty \big) r,
\end{align*}
where we used the definitions of $\|\nabla \brho\|_\infty$ and $\|\nabla Q\|_\infty$.
We denote by $C_{1,R}(p) \coloneqq  (d+2) \| Q\|_\infty \rho_{\max}^{d+1} \| \nabla \brho \|_\infty + \rho_{\max}^{d+2} \| \nabla Q \|_\infty$, then
\begin{align} \label{eq:proof-brhor-diff-2}
	\left| \nabla \left( 	- r Q \brho_r^{d+2} \right) \big|_x \left( v \right)  \right| \leq r (d+2) \| Q\|_\infty \rho_{\max}^{d+1} \left|  \nabla (\brho_{r} - \brho)|_x(v) \right| + C_{1,R}(p) r.
\end{align}

We are ready to bound $| \nabla (\brho_r - \brho )|_x(v)|$.
We first show an equality 
\begin{equation}\label{eq:proof-nabala-l=1-bound-fact1}
\nabla \left(\brho_r^d - \brho^d \right) = \nabla \left( 	- r Q \brho_r^{d+2} \right).    
\end{equation}
By the definition of $\brho_r$ in Definition \ref{def:bar-rho-epsilon}, we have $\brho_r^d ( 1 + r Q \brho_r^2)  =  \brho^d$, which can be written as
$\brho_r^d - \brho^d   =  - r Q \brho_r^{d+2}$. 
By taking the manifold derivative on both sides, we have \eqref{eq:proof-nabala-l=1-bound-fact1}.
Combining \eqref{eq:proof-nabala-l=1-bound-fact1} with \eqref{eq:proof-brhor-diff-1} and \eqref{eq:proof-brhor-diff-2}, we have
\[ d \rho_{\min}^{d-1} \left| \nabla (\brho_r - \brho )|_x(v)   \right| -  C_{1,L}(p) r
\leq  r (d+2) \| Q\|_\infty \rho_{\max}^{d+1} \left|  \nabla (\brho_{r} - \brho)|_x(v) \right| + C_{1,R}(p) r.
\]
Recall the definition of $\tilde{r}_0$ in \eqref{eq:def-r0-tilder0}.
When $r \leq \tilde{r}_0$, $r(d+2) \| Q\|_\infty \rho_{\max}^{d+1}  \leq \frac{1}{2}d \rho_{\min}^{d-1}$, and thus
\begin{align} \label{eq:proof-brhor-diff-3}
	\left| \nabla (\brho_r - \brho) |_x (v) \right|   \le C_1(p)r,	
\end{align}
where $C_1(p) \coloneqq ( C_{1,L}(p) + C_{1,R}(p) + 1) / (\frac{1}{2}d \rho_{\min}^{d-1}  ) > 0$. Note that $C_1(p)$ depends on $\sM$ and the derivatives of $p$ up to the third order.
Since \eqref{eq:proof-brhor-diff-3} holds for any $v\in T_x \sM$ with $\|v\| \leq 1$, 
by Definition \ref{def:Dkf-norm}, 
we have
\[ \| \nabla (\brho_r - \brho) |_x \|  
\leq  C_1(p)r.\]
Since this holds for arbitrary $x$, we have $\| \nabla \brho_r -   \nabla\brho  \|_\infty \leq  C_1(p)r$. 
Consequently, for $r \leq \tilde{r}_0$,
\begin{align} \label{eq:def-rhomax-1}
	\| \nabla \brho_r \|_\infty \leq   \| \nabla \brho \|_\infty + C_1(p)  \tilde{r}_0
 =: \rho_{\max}^{(1)}.
\end{align}

For $l = 2$,
for any $v \in T_x\sM$ such that $\|v\| \leq 1$, a direct calculation gives 
\begin{align*}
	& \nabla^2 \left(\brho_r^d - \brho^d \right) \big|_x \left(v, v\right)  = D^2 ( \trho_r^d - \trho^d )(0)(v,v) \\
	&= d\trho_r(0)^{d-1} D^2\trho_r(0)(v,v) -  d\trho(0)^{d-1}D^2\trho(0)(v,v)  \\
	&~~~~ +  d(d-1) \trho_r(0)^{d-2} (D\trho_r(0)(v))^2  - d(d-1) \trho(0)^{d-2} (D\trho(0)(v))^2,  \\
    &\nabla^2 \left( 	- r Q \brho_r^{d+2} \right) \big|_x \left( v, v \right) =  -r D^2 (\tilde{Q} \trho_r^{d+2} )(0)(v,v) \\
    &
    = - r \Big(  (d+2) \tilde{Q}(0) \trho_r(0)^{d+1} D^2\trho_r(0)(v,v) + (d+2)(d+1) \tilde{Q}(0) \trho_r(0)^d  (D\trho_r(0)(v))^2 
	\\
	&~~~~ + 2 (d+2) \trho_r(0)^{d+1} D\trho_r(0)(v)  D\tilde{Q}(0)(v) + \trho_r(0)^{d+2} D^2\tilde{Q}(0)(v,v) \Big) . 
 \end{align*}
By a similar argument used to prove the bound for $| \nabla (\brho_r^d - \brho^d ) |_x (v)|$ and $	| \nabla ( 	- r Q \brho_r^{d+2} ) |_x ( v ) | $, we have that
\begin{align} \label{eq:bound-2nd-deriv-lb}
	\left| \nabla^2 \left(\brho_r^d - \brho^d \right) \big|_x \left(v, v\right) \right|   \geq d \rho_{\min}^{d-1} 	\left| \nabla^2 \left(\brho_r - \brho \right) \big|_x \left(v, v\right) \right|   - C_{2,L}(p)r,
\end{align}
where $C_{2,L}(p) := d(d-1) \rho_{\max}^{d-2} \|\nabla^2 \brho\|_\infty C_0(p) + d(d-1)(d-2) \rho_{\max}^{d-3}( \rho_{\max}^{(1)} )^2 C_0(p) + 2d(d-1)  \rho_{\max}^{d-2} \rho_{\max}^{(1)} C_1(p)$, which depends on $\sM$ and the derivatives of $p$ up to the third order.
Meanwhile,
\begin{align} \label{eq:bound-2nd-deriv-ub}
	\left|  \nabla^2 \left( 	- r Q \brho_r^{d+2} \right) \big|_x \left( v, v \right)  \right|  \leq r (d+2) \| Q\|_\infty \rho_{\max}^{d+1} 	\left| \nabla^2 \left(\brho_r - \brho \right) \big|_x \left(v, v\right) \right|   +   C_{2,R}(p)r,
\end{align}
where $C_{2,R}(p) \coloneqq (d+2)  \| Q\|_\infty \rho_{\max}^{d+1} \| \nabla^2 \brho \|_\infty + 2(d+2) \rho_{\max}^{d+1} \rho_{\max}^{(1)} \| \nabla Q \|_\infty + \rho_{\max}^{d+2} \| \nabla^2 Q \|_\infty + (d+2)(d+1) \|Q \|_\infty \rho_{\max}^d (\rho_{\max}^{(1)})^2$, which depends on $\sM$ and the derivatives of $p$ up to the fourth order.
By taking first-order derivatives on both sides of \eqref{eq:proof-nabala-l=1-bound-fact1}, we obtain $\nabla^2 \left(\brho_r^d - \brho^d \right)|_x \left( v, v \right) = \nabla^2 \left( 	- r Q \brho_r^{d+2} \right)|_x \left( v, v \right)$.
Then, 
by the same argument to derive \eqref{eq:proof-brhor-diff-3} from \eqref{eq:proof-brhor-diff-1}\eqref{eq:proof-brhor-diff-2}\eqref{eq:proof-nabala-l=1-bound-fact1},
here, we combine with  
the two inequalities \eqref{eq:bound-2nd-deriv-lb}\eqref{eq:bound-2nd-deriv-ub}
and obtain that 
\[ 
\left| \nabla^2 \left(\brho_r - \brho \right) \big|_x \left(v, v\right) \right|  \leq C_2(p)r, 
\]
where $C_2(p) \coloneqq ( C_{2,L}(p) + C_{2,R}(p) + 1) / (\frac{1}{2}d \rho_{\min}^{d-1}  ) > 0$. 
Note that $C_2(p)$ depends on $\sM$ and the derivatives of $p$ up to the fourth order.
Since this holds for any $v \in T_x \sM$ with $\|v\| \leq 1$, by Definition \ref{def:Dkf-norm}, 
\[   \| \nabla^2 (\brho_r - \brho) |_x \|  \leq  C_2(p)r.   \]
Since this is true for arbitrary $x$, we have $\| \nabla^2 \brho_r -   \nabla^2 \brho  \|_\infty \leq  C_2(p)r$ when $r \le \tilde{r}_0$.

Finally, we consider $l=3$. 
For any $v  \in T_x \sM$ such that $\|v\| \leq 1$, by similar arguments to those for $l=1, 2$, we have
\begin{align*}
	& \left| \nabla^3 \left(\brho_r^d - \brho^d \right) \big|_x \left(v, v, v\right) \right|   \geq d \rho_{\min}^{d-1}  \left| \nabla^3 \left(\brho_r - \brho \right) \big|_x \left(v, v, v\right) \right|  - C_{3,L}(p)r, \\
	& \left|  \nabla^3 \left( 	- r Q \brho_r^{d+2} \right) \big|_x \left( v, v, v \right)  \right|  \leq r (d+2) \| Q\|_\infty \rho_{\max}^{d+1} \left| \nabla^3 \left(\brho_r - \brho \right) \big|_x \left(v, v, v\right) \right|  +   C_{3,R}(p)r,	
\end{align*}
where $C_{3,L}(p)$ is a constant depending on $\sM$ and the derivatives of $p$ up to the fourth order, and $C_{3,R}(p)$ is a constant depending on $\sM$ and the derivatives of $p$ up to the fifth order, similar to the definitions of $C_{2,R}(p)$ and $C_{2,L}(p)$.
By taking the second-order derivative on both sides of \eqref{eq:proof-nabala-l=1-bound-fact1}, we have  $\nabla^3 \left(\brho_r^d - \brho^d \right)|_x \left( v, v, v \right) = \nabla^3 \left( 	- r Q \brho_r^{d+2} \right)|_x \left( v, v, v \right)$. 
Then, similarly as in the proof for $l=2$, combined with the two inequalities above, we have that when $ r \leq \tilde{r}_0$, 
\[ 
\left| \nabla^3 \left(\brho_r - \brho \right) \big|_x \left(v, v, v\right) \right|     \leq C_3(p)r, 
\]
where $C_3(p) \coloneqq ( C_{3,L}(p) + C_{3,R}(p) + 1) / (\frac{1}{2}d \rho_{\min}^{d-1}  ) > 0$. 
Note that $C_3(p)$ depends on $\sM$ and the derivatives of $p$ up to the fifth order.
Consequently, 
by Definition \ref{def:Dkf-norm}, 
$ \| \nabla^3 (\brho_r - \brho) |_x \|   \leq  C_3(p)r$, $\forall x\in \calM$,
and then $\|\nabla^3  \brho_r -   \nabla^3 \brho  \|_\infty \le C_3(p)r$ whenever $r \leq \tilde{r}_0$.
\end{proof}

\begin{proof}[Proof of Proposition \ref{prop:consist-hrho-x0}]
The proof follows the strategy of Proposition 2.2 in \cite{cheng2022convergence}, and we include the proof here for completeness.

Recall the definition of $r_k$ in \eqref{eq:def-rk}.
Since $k = o(N)$, when $N$ is sufficiently large, $r_k \leq r_0$, where $r_0$ is as defined in \eqref{eq:def-r0-tilder0}.
Then, $\brho_{r_k}(x)$ is well-defined, namely, 
$\brho_{r_k}(x)$ is the unique solution of \eqref{eq:def-bar-rho-epsilon} with $r = r_k$. Define
\begin{equation}\label{eq:def-barR}
\bar{R}(x) := \brho_{r_k}(x)\sqrt{r_k} = \brho_{r_k}(x) \left( \frac{k}{\alpha_dN} \right)^{1/d}\,.
\end{equation}
Since we also have $\hat{R}(x) = \hrho(x) \sqrt{r_k}$ by \eqref{eq:rhohat-Rhat},
the proposition can be equivalently proved 	by controlling 
${ |\hat{R}(x) - \bar{R}(x) |}/{\bar{R}(x)}$.

We derive some useful facts about $\bar{R}(x)$. 
By  the definition of $\brho_{r_k}$ in Definition \ref{def:bar-rho-epsilon}, equivalently by \eqref{eq:def-bar-rho-epsilon},
$\bar{R}(x)$ satisfies that
\begin{align} \label{eq:def-bar-R}
	\alpha_d p(x) \bar{R}(x)^d \left(1 + Q(x) \bar{R}(x)^2 \right) = \frac{k}{N},
	\quad \forall x \in \sM.
\end{align}
Furthermore, by Lemma \ref{lemma:bar-rho-epsilon}(ii), we have that for any $ x \in \sM$,
\begin{align} \label{eq:bound-barR}
	\bar{R}(x)  \in  \left[  \left( \frac{2k}{3 \alpha_dN p(x)} \right)^{1/d}, \; \left( \frac{2k}{ \alpha_dN p(x)} \right)^{1/d} \right] \subset \left[ \rho_{\min} \left( \frac{k}{\alpha_dN} \right)^{1/d}, \;  \rho_{\max} \left( \frac{k}{ \alpha_dN} \right)^{1/d} \right],
\end{align} 
where $\rho_{\min}$ and $\rho_{\max}$ are defined in \eqref{eq:def-rhomin-rhomax}.

For the given $s>0$,  
define
\begin{equation*} 
\delta_r := t_1 \left(\frac{k}{N}\right)^{3/d} + \frac{t_2}{d} \sqrt{\frac{s \log N}{k}}\,,
\end{equation*}
where 
$t_1 =\Theta^{[p]} (1)$, $t_2 =\Theta(1)$, 
and both will be determined later.
We will show that,
when $N$ exceeds a threshold depending on $(\sM, p,s)$,
for any fixed $x \in \sM$, 
w.p. $\ge 1-2 N^{-s}$, 
\begin{equation}\label{Proposition2.2 Proof first bound}
\bar{R}(x)(1 - \delta_r) \leq  \hat{R}(x)  \leq \bar{R}(x)(1 + \delta_r).
\end{equation}
To prove \eqref{Proposition2.2 Proof first bound}, 
we introduce some notations. Denote
\[
R_{-}(x) := \bar{R}(x)(1 - \delta_r),
\quad\mbox{and}\quad 
R_ +(x) := \bar{R}  {(x)}(1+ \delta_r)\,.
\]
Let $h(\eta) = {\bf 1}_{[0,1]}(\eta)$,
and define, for any $x\in\sM$ and $r >0$,
\begin{equation} \label{eq:def-Hj}
  \hat{\mu}(x,r) := \frac{1}{N} \sum_{j=1}^N h \left(  \frac{ \| x- x_j\|^2 }{r^2}\right) 
= \frac{1}{N} \sum_{j=1}^N H_j(x,r),
\quad H_j(x,r) :=  h \left(  \frac{ \| x- x_j\|^2 }{r^2}\right),   
\end{equation}
then, by \eqref{eq:def-hat-rho},
$\hat{R} (x) = \inf\left\{ r > 0,\, \text{ s.t. }  \hat{\mu}(x,r) \ge \frac{k}{N} \right\}$.
For fixed $x$ and $r$, $H_j$ are i.i.d. random variables, and
\begin{equation*}
\E H_j(x,r) = \int_{\sM} h \left(  \frac{ \| x- y\|^2 }{r^2}\right)  p(y) dV(y) =: \mu(x,r).
\end{equation*}
Below, to simplify notation, we omit the dependence on $x$ in  $\bar{R}$ and $R_\pm$ when there is no confusion.
The argument is for a fixed $x$, and we make sure that the constants $t_1$ and $t_2$ in $\delta_r$
as well as the large-$N$ threshold
 are uniform for all $x$.

\vspace{5pt}
$\bullet$ Lower bound in \eqref{Proposition2.2 Proof first bound}: 
By definition, for any fixed $x\in\sM$, $\hat{\mu}(x,r)$ is monotonically increasing on $\R_+$ with respect to the variable $r$.
 We claim that 
\begin{equation}\label{eq:lower-bound-relation-1}
\Pr [ \hat{R}(x) < R_- ]
\leq 
\Pr \left[ \hat{\mu}( x, R_- ) \ge \frac{k}{N}  \right].
\end{equation}
Because 
$\hat{R}(x) = \inf \{ r > 0,  \hat{\mu}(x,r) \ge \frac{k}{N} \}$, if $\hat{R}(x) < R_- $, 
there is some $r'$, $ \hat{R}(x) < r'  < R_-$ such that $\hat{\mu}(x, r') \ge \frac{k}{N}$,
and by monotonicity $\hat{\mu}(x, R_-) \ge \hat{\mu}(x, r') \ge \frac{k}{N}$.

To bound the probability $\Pr \left[ \hat{\mu}(x,  R_- ) \ge \frac{k}{N}  \right]$,
we will prove and use that the expectation $\mu(x,R_-)$ would be smaller than $\frac{k}{N}$ for properly chosen $t_1$ and sufficiently large $N$.

Note that by the lower and upper bounds of $\bar{R}$ in \eqref{eq:bound-barR},
$\bar{R} = \Theta^{[p]} ( (\frac{k}{N})^{1/d} ) =o^{[p]}(1)$,
and the constant in the big-$\Theta$ notation is uniform for all $x$.
Also, we have that
$\delta_r = o^{[p]}(1)$ 
under the asymptotic condition on $k$ in the proposition.
As a result, we have that
\begin{equation}\label{eq:Rminus(x)-asymp-uniform-x}
R_-  = \Theta^{[p]} (\bar{R} ) = o^{[p]}(1).
\end{equation}
We apply Lemma \ref{lemma:G-expansion-h-indicator} to expand $\mu( x,R_- ) $ with $\epsilon = R_-^2$. 
By \eqref{eq:Rminus(x)-asymp-uniform-x}, with large $N$ and then $R_-$ is small enough such that $\epsilon$ satisfies the condition of 
Lemma \ref{lemma:G-expansion-h-indicator}, 
\eqref{eq:G-expansion-h-indicator} gives that
\[\mu( x,R_- ) 
=  \alpha_d p(x) R_- ^d + \frac{\alpha_d}{2(d+2)}R_- ^{d+2}(\Delta p(x) + \omega(x)p(x))+ O^{[p]} ( R_-^{d+3}).
\]
Recall that $R_{-} = \bar{R}(1 - \delta_r)$, 
and $Q(x) = \frac{1}{2(d+2)}(\frac{\Delta p(x)}{p(x)} + \omega(x))$ by definition \eqref{eq:def-bar-rho-epsilon} (where $\omega(x)$ is as in \eqref{eq:def-omega} in Lemma \ref{lemma:G-expansion-h-indicator}), we further have
\begin{align}
\mu( x,R_- ) 
& = \alpha_d p(x) \bar{R}^d  (1- \delta_r)^d (1 + Q(x) \bar{R}^2 (1- \delta_r)^2) + O^{[p]}( \bar{R}^{d+3})   \nonumber \\
& = \alpha_d p(x) \bar{R}^d  \left( 1 + Q(x)\bar{R}^2- d \delta_r + O(\delta_r^2) + O^{[p]}(\bar{R}^{2}\delta_r)  \right)  + O^{[p]}( \bar{R}^{d+3})   \nonumber  \\
& \leq 
\frac{k}{N} (1 - 0.6 d  \delta_r )
+ O^{[p]}(\bar{R}^{d+3})
=: \frac{k}{N} - \delta_{\mu_-},
\label{eq:def-delta-mu-minus}
\end{align}
where the last inequality is by \eqref{eq:def-bar-R},
together with that the $ O(\delta_r^2)+ O^{[p]}(\bar{R}^{2}\delta_r)$ term is less than $0.1d\delta_r$ with large $N$ (recall that both $\delta_r$ and $\bar R$ are $o^{[p]}(1)$) and 
that $\alpha_d p(x)\bar{R}^d \geq 2k/ (3N)$ from the lower bound of $\bar{R}$ in \eqref{eq:bound-barR}. 
Note that the constant in the term $O^{[p]}(\bar{R}^{d+3})$ in \eqref{eq:def-delta-mu-minus}, 
denoted as  $c_{p,L}$, is uniform for all $x$.
Meanwhile, since $r_k \leq r_0$, by the upper bound of $\bar{R}$ in \eqref{eq:bound-barR}, we have
\begin{equation} \label{eq:def-cB}
    \bar{R}(x) \leq c_B  \left(\frac{k}{N}\right)^{1/d}, 
\quad \forall x \in \sM,
\quad c_B := \alpha_d^{-1/d}\rho_{\max}.
\end{equation}
We choose
\begin{equation}\label{eq:pick-t1}
t_{1,L} := \frac{ c_{p,L} c_B^{d+3} }{0.5 d}
= \Theta^{[p]}(1),
\end{equation}
which is uniform for all $x$. 
We will choose $t_1$ such that $t_1 \geq t_{1,L}$.
Then, for this choice of $t_{1,L}$, we have
$ 0.6 d  t_{1,L}  \left(\frac{k}{N}\right)^{1+3/d} > c_{p,L}   c_B^{d+3} \left(\frac{k}{N}\right)^{1+3/d}  \ge c_{p,L} \bar{R}(x)^{d+3}, \forall x \in \sM$.
Thus,  when $N$ is sufficiently large and the threshold depends on $(\calM,p)$, we have
\begin{align}
\delta_{\mu_-} &\geq 0.6 d \frac{k}{N}
\left(  t_{1,L} \left(\frac{k}{N}\right)^{3/d} + \frac{t_2}{d} \sqrt{\frac{s \log N}{k}} \right)
 + O^{[p]}(\bar{R}^{d+3}) \nonumber \\
& >  0.6  t_2  \frac{k}{N}\sqrt{\frac{s \log N}{k}} 
= 0.6  t_2  \left( \frac{k}{N} \right)^{1/2} \sqrt{\frac{s \log N}{N}} 
=: \tilde{s}.
\label{eq:def-tilde-s}
\end{align}
Combining \eqref{eq:def-tilde-s} with \eqref{eq:def-delta-mu-minus}, 
we have
\begin{equation} \label{eq:mu-R--bound}
    \mu(x, R_-) 
\leq \frac{k}{N} - {\delta_{\mu_-}} 
< \frac{k}{N} - \tilde{s} 
\implies 
\mu(x, R_-) + \tilde{s} < \frac{k}{N} \,.
\end{equation}

To show the concentration of $\hat{\mu}( x, R_- )$ at $\mu( x, R_- )$,
we use \eqref{eq:def-Hj}
and the boundedness and variance of the r.v. $H_j(x, R_-)$:
Because $0 \leq h \leq 1$, so is $H_j( x, R_- )$, and then $|H_j( x, R_- )| \leq L _H :=1$. The variance satisfies
\[
\var(H_j( x, R_- ))
\leq \E H_j( x, R_- )^2 
= \int_{\sM} h^2 \left(  \frac{ \| x- y\|^2 }{ R_-^2}\right)  p(y) dV(y)
= \mu(x, R_-),
\]
because the indicator function $h$ satisfies $h^2 = h$. 
By \eqref{eq:def-tilde-s}, $\delta_{\mu_-} > \tilde{s} > 0$.
Thus, combined with \eqref{eq:def-delta-mu-minus}, we have
\begin{equation} \label{eq:bound-var-rho}
	\var(H_j( x, R_- )) \leq  \frac{k}{N} - \delta_{\mu_-} < 3 \frac{k}{N} =: \bar{\nu}_H.
\end{equation}
By the Bernstein inequality, 
as long as $\tilde{s} L_H \leq 3 \bar{\nu}_H$, 
then 
\[
\Pr [  \hat{\mu}( x, R_-)-  \mu( x, R_-) > \tilde{s}] \leq e^{- \frac{1}{4} \tilde{s}^2 \frac{N} {\bar{\nu}_H}}.
\]
To verify that $\tilde{s} L_H \leq 3 \bar{\nu}_H$:
note that it is equivalent to that $ 0.6  t_2   \leq 9( \frac{k}{s \log N} )^{1/2}$,
and since  we have assumed $k = \Omega(\log N)$, 
if we have $t_2 = \Theta( 1)$,
then it holds when $N$ is sufficiently large where the threshold depends on $s$.
This is fulfilled by
setting $t_2$ being an absolute constant such that
\begin{equation}\label{eq:pick-t2}
{ ( 0.6 t_2 )^2  }/{3} = 4.
\end{equation}
With this choice of $t_2$, by the definition of $\tilde{s}$ in \eqref{eq:def-tilde-s} and of $\bar{\nu}_H$ in \eqref{eq:bound-var-rho}, $\frac{1}{4} \tilde{s}^2 \frac{N} {\bar{\nu}_H} = s \log N$ holds.
As a result, \eqref{eq:lower-bound-relation-1} continues as
\begin{equation*} 
\Pr [ \hat{R}(x) < R_- ]
\le
\Pr \left[ \hat{\mu}( x, R_-) \ge \frac{k}{N} \right]
\leq 
\Pr \left[ \hat{\mu}( x, R_-) > \mu(x, R_- ) + \tilde{s}  \right]
\leq e^{- \frac{1}{4} \tilde{s}^2 \frac{N} {\bar{\nu}_H}} = N^{-s},
\end{equation*}
which proves that  w.p. $\ge  1 - N^{-s}$,
the lower bound $ \hat{R}(x) \ge R_-  $ holds. 
We call the event $ [ \hat{R}(x) \ge R_- ]$ the good event $E_1$.

\vspace{5pt}
$\bullet$ Upper bound in \eqref{Proposition2.2 Proof first bound}: 
The upper bound is proved in a similar way.
Specifically, we apply Lemma \ref{lemma:G-expansion-h-indicator} to $\mu(x, R_+)$ with $\epsilon = R_+^2$, where,
similar to \eqref{eq:Rminus(x)-asymp-uniform-x}, we have $R_+ = o^{[p]}(1)$,
and then with large $N$, we have 
\begin{align} \label{eq:bound-var}
\mu(x, R_+)	
&= \alpha_d p(x) \bar{R}^d  (1 + \delta_r)^d (1 + Q(x) \bar{R}^2 (1 + \delta_r)^2) + O^{[p]}( \bar{R}^{d+3}) \notag  \\
&= \alpha_d p(x) \bar{R}^d  \left( 1 + Q(x)\bar{R}^2 +  d \delta_r + O(\delta_r^2) + O^{[p]}(\bar{R}^{2}\delta_r)  \right) + O^{[p]}( \bar{R}^{d+3}). 
\end{align}
We will choose a  constant $t_{1,U} \le t_1$ below, and then, following  the same argument used in proving the lower bound, we have that
\begin{align*} 
 \mu(x, R_+)
 &
 \ge \frac{k}{N} (1 + 0.6 d \delta_r ) + O^{[p]}( \bar{R}^{d+3}) \notag \\
 &\geq \frac{k}{N}  + 0.6 d \frac{k}{N} 
 \left( t_{1,U} \left(\frac{k}{N}\right)^{3/d} + \frac{t_2}{d} \sqrt{\frac{s \log N}{k}} \right) + O^{[p]}( \bar{R}^{d+3}),
\end{align*}
where the constant in $O^{[p]}(  \bar{R}^{d+3})$ is denoted by $c_{p,U}$.

Next, following the same argument used to derive \eqref{eq:mu-R--bound} from \eqref{eq:def-delta-mu-minus}, 
we again use \eqref{eq:def-cB} and now introduce 
\begin{equation*} 
t_{1,U} := \frac{ c_{p,U} c_B^{d+3} }{0.5 d}
= \Theta^{[p]}(1),
\end{equation*}
and then we have
\[
\mu(x, R_+) > \frac{k}{N} + 0.6  t_2  \left(\frac{k}{N}\right)^{1/2} \sqrt{ \frac{s \log N}{N}} = \frac{k}{N} + \tilde{s}.
\]

We now bound the magnitude and variance of $H_j(x, R_+)$. Same as before, $H_j(x, R_+)$ is bounded by 1.
We upper bound $\text{Var}(H_j(x, R_+))$ using that $\text{Var}(H_j(x, R_+)) \leq \E H_j^2= \mu( x, R_+) $ and the expansion of $\mu( x, R_+) $ in \eqref{eq:bound-var}.
By the upper bound of $\bar{R}$ in \eqref{eq:bound-barR}, we have $\alpha_d p(x) \bar{R}^d \leq 2k/N$.
Since $\bar{R} = o^{[p]}(1),  \delta_r = o^{[p]}(1)$,
for sufficiently large $N$, $\mu( x, R_+) \leq 2k/N \times 1.5 = 3k/N$. 
Therefore,
\[
\text{Var}(H_j(x, R_+)) \leq \E H_j^2
= \mu( x, R_+) 
\leq 3\frac{k}{N} = \bar{\nu}_H\,.
\]
By letting $t_2$ as in \eqref{eq:pick-t2}, we have
\begin{equation*} 
\Pr[ \hat{R}(x) > R_+ ]
\leq 
\Pr \left[ \hat{\mu}(x, R_+) < \frac{k}{N} \right]
\leq 
\Pr \left[ \hat{\mu}(x, R_+) <  \mu(x, R_+) - \tilde{s} \right]
\leq e^{-\frac{1}{4} \tilde{s}^2 \frac{N}{ \bar{\nu}_H}} = N^{-s}.
\end{equation*}
This proves that w.p. higher than $1- N^{-s}$, the upper bound $\hat{R}(x) \leq R_+$ holds.
We call the event  $[ \hat{R}(x) \leq R_+ ]$ the good event $E_2$.

All the large-$N$ thresholds involved in proving the lower and upper bounds in  \eqref{Proposition2.2 Proof first bound} depend on $\sM$, $p$, and $s$, and are uniform for all $x$.

\vspace{5pt}
$\bullet$ Combining lower and upper bounds in \eqref{Proposition2.2 Proof first bound}:
We choose the constant
	\[ t_1 \coloneqq\max\{t_{1,L}, t_{1,U}\} = \max\left\{ \frac{ c_{p,L} c_B^{d+3} }{0.5 d}, \frac{ c_{p,U} c_B^{d+3} }{0.5 d}\right\} = \Theta^{[p]}(1),
 \]
which ensures $t_1 \ge t_{1,L}$ and $t_1 \ge t_{1,U}$, and then the lower bound holds under $E_1$ and the upper bound holds under $E_2$.
Putting together, 
when $N$ is sufficiently large (whose threshold depends on ($\sM$, $p$, $s$), and is uniform for all $x \in \sM$), under  $E_1\cap E_2$  which happens w.p. $\ge 1- 2N^{-s}$, we have
\begin{align} \label{eq:bound-Rhat-Rbar-const}
\frac{|\hat{R}(x) - \bar{R}(x)|}{ \bar{R}(x) } \leq {\delta_r} =
 \frac{ \max\{c_{p,L}, c_{p,U}\} c_B^{d+3} }{0.5 d} \left(\frac{k}{N}\right)^{3/d} 
 + \frac{2\sqrt{3}}{0.6 d} \sqrt{\frac{s \log N}{k}}.    
\end{align}
This proves the proposition where the constants in big-$O$ are as declared therein. 
\end{proof}

\begin{proof}[Proof of Theorem \ref{thm:consist-hrho}]

Since $k = o(N)$ and $\tilde{r}_0$ is an $O(1)$ constant defined in \eqref{eq:def-r0-tilder0}, with large $N$, 
we have $r_k \leq \tilde{r}_0$.
As a result, $\brho_{r_k}(x)$ is well-defined by Lemma \ref{lemma:bar-rho-epsilon}(i).

We cover $\sM$ using Euclidean balls in $\R^m$ with radius $r$, 
where $r>0$ takes the form as
\begin{equation}\label{eq:def-r-proof-uniform-bound}
r = t_3  (k/N)^{4/d},  
\end{equation}
and $t_3 $ is an $O(1)$ positive constant to be determined. 
Thus, $r = \Theta( (k/N)^{4/d})  = o(1)$.
Then, with large $N$, 
$r < \delta_1$ required in Lemma \ref{lemma:covering}.
Then, by Lemma \ref{lemma:covering}, we can find an $r$-net of $\calM$,
$F :=\{ z_1, \cdots, z_n \}$,
i.e., $\calM \subset \cup_{i=1}^n \bar{B}_r(z_i)$ (where $\bar{B}_r$ is $\bar{B}_r^m$),
and the cardinal number of $F$ is $n$, $n \leq V({\sM}) r^{-d}$.

We now apply Proposition \ref{prop:consist-hrho-x0} at $x=z_i$ for each $i$,
where we choose $s=14$, and the reason will be explained by the argument below. 
Specifically, we will use the bound \eqref{eq:bound-Rhat-Rbar-const} in the proof of Proposition \ref{prop:consist-hrho-x0}, where, due to the normalizing relationship between $\hat R$ ($\bar R$) and $\hat \rho$ ($\bar \rho_{r_k}$) in \eqref{eq:rhohat-Rhat} (in \eqref{eq:def-barR}), \eqref{eq:bound-Rhat-Rbar-const} gives 
\begin{equation}\label{eq:def-epsilon-covering-bound}
\left| \frac{\hrho(z_i) }{\brho_{r_k}(z_i) } - 1  \right| 
\leq c_{1,p} \left(\frac{k}{N}\right)^{3/d} +  c_2 \sqrt{ \frac{s \log N }{k}} 
=: \varepsilon, 
\quad \text{for all $i=1,\cdots, n$,}
\end{equation}
where positive constants $c_{1,p}, c_2$ are as specified in \eqref{eq:bound-Rhat-Rbar-const}, 
$c_{1,p}$ depends on $(\calM, p)$ and $c_2$ only depends on $d$;
 the $n$ inequalities in \eqref{eq:def-epsilon-covering-bound} hold w.p. $\ge 1- 2n N^{-s}$ ($s=14$) by a union bound,
 and we call the intersection of the $n$ good events a good event $E_\rho$.
For \eqref{eq:def-epsilon-covering-bound} to hold, we also require $N$ to exceed a threshold required by \eqref{eq:bound-Rhat-Rbar-const}, and this threshold only depends on $(\sM, p)$ because the large-$N$ threshold required by \eqref{eq:bound-Rhat-Rbar-const} is uniform over all $z_i$.

We now bound the change of $\hrho/\brho_{r_k}$ within each $\bar{B}_{r}(z_i) \cap {\sM}$ from its value at $z_i$. 
Since we already have $r_k \leq \tilde{r}_0$, by Lemma \ref{lemma:bar-rho-epsilon}(i), $\brho_{r_k}$ is $C^1$ on ${\sM}$. 
From the proof of Lemma \ref{lemma:bar-rho-epsilon}(iv) $\| \nabla \brho_{r_k} \|_\infty \leq  L_p  \coloneqq \rho_{\max}^{(1)}$, where $ \rho_{\max}^{(1)}$ is defined in \eqref{eq:def-rhomax-1}.
Then, by applying Lemma \ref{lemma:f-diff-bound} to $\brho_{r_k}$, 
and combined with Lemma \ref{lemma:M-metric}(iii) (which applies since we have assumed $r < \delta_1(\sM)$), for each $i$, we have that 
\begin{align} \label{eq:brho-rk-Lip}
    | \brho_{r_k}(x) - \brho_{r_k}(z_i) | \leq L_p  d_{\sM}(x,z_i) \le
1.1 L_p \| x - z_i \| \leq 1.1 L_p r,
\quad 
\forall x \in \bar{B}_{r}(z_i) \cap {\sM}.
\end{align}
Meanwhile, we restrict to when $X$ has $N$ distinct points,
which, under Assumptions \ref{assump:M}-\ref{assump:p}, holds w.p. one.
Then Lemma \ref{lemma:Lip-cont-hat-R} applies to give that 
\[
\text{Lip}_{\R^m}(\hrho) =  \left( \frac{k}{\alpha_d N} \right)^{-1/d}
\text{Lip}_{\R^m}( \hat{R})
\leq  \left(\frac{k}{\alpha_d N} \right)^{-1/d},
\]
so we have
\begin{align} \label{eq:hrho-Lip}
    | \hrho(x) - \hrho(z_i) | \leq  \left( \frac{k}{\alpha_d N} \right)^{-1/d} r,
\quad
\text{$\forall x \in \bar{B}_{r}(z_i) \cap {\sM}$.}
\end{align}
Besides, $\brho_{r_k}(x) \geq  \rho_{\min}, \forall x \in \sM$ by Lemma \ref{lemma:bar-rho-epsilon}(ii).
Together, we have that, for each $i$  and $\forall x \in \bar{B}_{r}(z_i) \cap {\sM}$,
\begin{align}
\left| \frac{\hrho(x)}{\brho_{r_k}(x)} - \frac{\hrho(z_i)}{\brho_{r_k}(z_i)} \right|
&\leq 
\frac{1}{\brho_{r_k}(x)} \left|  (\hrho(x) - \hrho(z_i)) 
- \frac{\hrho(z_i)}{\brho_{r_k}(z_i)} ( \brho_{r_k}(x)-\brho_{r_k}(z_i) ) \right| \nonumber \\
& \leq 
\rho_{\min}^{-1}
\left( 
\left( \frac{k}{\alpha_d N} \right)^{-1/d} 
+ (1+\varepsilon) 1.1 L_p
\right) r  \nonumber \\
& ~~~~~~~~~~~~
    \text{(the 1st term is by \eqref{eq:hrho-Lip}, 
        the 2nd term is by \eqref{eq:def-epsilon-covering-bound}\eqref{eq:brho-rk-Lip})}  \nonumber \\
& \le c_p'  \left({k}/{N}\right)^{-1/d}  r
\quad \text{with large $N$ and under $E_\rho$,}
\label{eq:bound-diff-rhoratio-from-xi}
\end{align}
where $c_p'$ is a positive constant depending on $p$.
The last inequality in \eqref{eq:bound-diff-rhoratio-from-xi} holds  because
$(k/N)^{-1/d} = \Omega(1)$ (because $k = o(N)$)
and $1+\varepsilon = O(1)$
(by the definition of $\varepsilon$ in \eqref{eq:def-epsilon-covering-bound},
$\varepsilon = o^{[p]}(1)$ as $N \to \infty$).

Now we can choose the constant in the definition of $r$ in \eqref{eq:def-r-proof-uniform-bound} to be 
\begin{align} \label{eq:choice-t3}
    t_3 := {c_{1,p}}/{c_p'},
\end{align}
where $c_{1,p}$ is as in \eqref{eq:def-epsilon-covering-bound}, and $c_p'$ is as in \eqref{eq:bound-diff-rhoratio-from-xi}.
With this choice of $r$, by \eqref{eq:bound-diff-rhoratio-from-xi}, 
we have that when $N$ exceeds all the needed large-$N$ thresholds and under $E_\rho$, 
\[
\left| \frac{\hrho(x)}{\brho_{r_k}(x)} - \frac{\hrho(z_i)}{\brho_{r_k}(z_i)} \right| \leq  c_{1,p}  \left(\frac{k}{N}\right)^{3/d}, \quad \forall i = 1,\dots, n, \quad \forall x \in \bar{B}_{r}(z_i) \cap {\sM}. 
\]
Putting together with \eqref{eq:def-epsilon-covering-bound}, 
for each $i$ and $\forall x \in \bar{B}_{r}(z_i) \cap {\sM}$,
\[
\left| \frac{\hrho(x)}{\brho_{r_k}(x)} -1 \right|
\leq 
\left| \frac{\hrho(x)}{\brho_{r_k}(x)} - \frac{\hrho(z_i)}{\brho_{r_k}(z_i)} \right|
+ \left| \frac{\hrho(z_i)}{\brho_{r_k}(z_i)} -1 \right|
\leq c_{1,p} \left(\frac{k}{N}\right)^{3/d} 
+ \varepsilon.
\]
Since $\sM\subset \cup_{i=1}^n  \bar B_r(z_i)$, the above bound holds for all $x \in \sM$.
Combined with the definition of $\varepsilon$ in \eqref{eq:def-epsilon-covering-bound},
we have that, with large $N$ and under $E_\rho$, 
\begin{align*} 
	\sup_{x \in {\sM}}
	\left| \frac{\hrho(x)}{\brho_{r_k}(x)} -1 \right|
	\leq 2 c_{1,p} \left(\frac{k}{N}\right)^{3/d} +  c_2 \sqrt{ \frac{s \log N }{k}},
\end{align*}
which implies  \eqref{eq:vareps-rho-bound} in the theorem with the declared constant dependence in the big-$O$ notation.

Finally, we verify the high probability of the good event $E_\rho$.
Since $n \leq V({\sM}) r^{-d} $,
and the definition of $r$ in \eqref{eq:def-r-proof-uniform-bound} with the choice of $O(1)$ constant $t_3$ in \eqref{eq:choice-t3}, 
we have 
\begin{align*}
 2n N^{-s}
&\leq 2V({\sM}) r^{-d} N^{-s}
= \frac{ 2V({\sM}) } {({c_{1,p}}/{c_p'})^{d}} k^{-4} N^{-(s-4)} \leq N^{-(s-4)}. 
\end{align*}
The last inequality holds when $k$ is large enough, and since $k = \Omega(1)$, it holds when $N$ is large enough.
By choosing $s = 14$,  we have that the good event $E_\rho$  happens w.p. $\ge 1- N^{-(s-4)} = 1-N^{-10}$.
We note that all the large-$N$ thresholds involved in the proof are uniform for all $x$ and $i$, and only depend on $(\calM, p)$.
\end{proof}

\section{Proofs in Section \ref{sec:un-GL-fast}}
\label{sec:proof-W-un-fast}

For the convergence of the graph Laplacians,
the results are established for large $N$ and under good events which happen with high probability. 
We collect the thresholds when necessary, denoted as $N_i, i=1,2,\dots$, and the good events, denoted as $E_i, i=1,2,\dots$, in the proofs below.

The first threshold  $N_{\rho,1}$ is introduced to ensure that  $r_k \leq \tilde{r}_0$: 
Recall that $\tilde{r}_0$ is an $O(1)$ constant as defined in \eqref{eq:def-r0-tilder0}, 
and $r_k$ is defined as in \eqref{eq:def-rk}, 
thus when $k = o(N)$, $k/N = o(1)$.
Then, there is $N_{\rho,1}(\sM, p)$ s.t. 
when $N \ge N_{\rho,1}$, we have $r_k \leq \tilde{r}_0$.

The second threshold $N_{\rho,2}$ and the good event $E_\rho$ are to ensure that $\varepsilon_{\rho,k} \leq \min\{\delta_\phi, 0.05/L_\phi\}$,
where $\varepsilon_{\rho,k}$ is defined in \eqref{eq:def-vareps-rho},
and $\delta_\phi$ and $L_\phi$ are positive constants introduced in Assumption \ref{assump:phi-diff}(iii). 
The good event $E_\rho$ has already been introduced in the proof of Theorem \ref{thm:consist-hrho},
which happens w.p. $\geq 1 - N^{-10}$.
For large $N$ and under $E_\rho$,  \eqref{eq:vareps-rho-bound} in Theorem \ref{thm:consist-hrho} holds. 
Assuming $\log N \ll k \ll N$, for large $N$, we have the r.h.s. in \eqref{eq:vareps-rho-bound} less than $\min\{\delta_\phi, 0.05/L_\phi\}$.
We denote the maximum of this needed large $N$ and that needed by Theorem \ref{thm:consist-hrho} as  $N_{\rho,2}(\sM, p, \phi)$, and then 
when $N \ge N_{\rho,2}$ and under $E_\rho$, we have $\varepsilon_{\rho,k} \leq \min\{\delta_\phi, 0.05/L_\phi\}$. 

Putting together, we define
$N_\rho(\sM, p, \phi) \coloneqq \max\{ N_{\rho,1}, N_{\rho,2} \}$,
 then 
 \begin{equation} \label{eq:property-Nrho}
     r_k \leq \tilde{r}_0 \text{ when } N \ge N_\rho; 
     \quad 
     \varepsilon_{\rho,k} \leq \min\{\delta_\phi, 0.05/L_\phi\}
     \text{ when $N \ge N_\rho$ and under $E_\rho$.} 
 \end{equation}

We also introduce a technical quantity $E_{\phi,k}$ that bounds the relative error of the empirical bandwidth $\phi(\hrho(x), \hrho(y))$ around the population one $\phi(\brho_{r_k}(x), \brho_{r_k}(y))$:

\begin{equation}  \label{eq:def-vareps-rho-phi}
E_{\phi,k}    \coloneqq    \sup_{x,y \in \calM}   \frac{ \left| \phi( \hrho(x), \hrho(y) )  -  \phi( \brho_{r_k}(x), \brho_{r_k}(y) )   \right|}{ \phi( \brho_{r_k}(x), \brho_{r_k}(y) )  }  .
\end{equation}
The following lemma shows that 
 $E_{\phi,k}$ can be both lower and upper bounded by multiples of $\varepsilon_{\rho,k}$.
\begin{lemma} \label{lemma:lb-ub-E_phi}
	For $\phi$ satisfying Assumption \ref{assump:phi-diff}(i), we have
	\begin{equation} \label{eq:bound-phi-rela-left}
		\varepsilon_{\rho,k}  \leq E_{\phi,k}. 
	\end{equation}
	 If $\phi$ satisfies Assumption \ref{assump:phi-diff}(iii) and $\varepsilon_{\rho,k}$ satisfies $\varepsilon_{\rho,k} \leq \min\{ \delta_\phi, 0.05 / L_\phi\}$, we have
	\begin{equation}  \label{eq:bound-phi-rela-0.1}
		E_{\phi,k} \leq 2L_\phi \varepsilon_{\rho,k} \leq 0.1.    
	\end{equation}
\end{lemma}

\begin{proof}[Proof of Lemma \ref{lemma:lb-ub-E_phi}]

	\eqref{eq:bound-phi-rela-left} follows from the definitions of $\varepsilon_{\rho,k}$ and $E_{\phi,k}$. Specifically, since $\phi(u,u) = u$ for all $u \in \mathbb{R}_+$, we have 
	\[
	\varepsilon_{\rho,k} = \sup_{x \in \sM} \frac{| \hrho(x) - \brho_{r_k}(x)|}{ \brho_{r_k}(x)} = \sup_{x \in \sM} \frac{| \phi(\hrho(x), \hrho(x)) - \phi(\brho_{r_k}(x), \brho_{r_k}(x)) |}{\phi(\brho_{r_k}(x), \brho_{r_k}(x))} \leq E_{\phi,k}.
	\]
	\eqref{eq:bound-phi-rela-0.1} is a direct consequence of Assumption \ref{assump:phi-diff}(iii). Since $\varepsilon_{\rho,k} \le \delta_\phi$, by Assumption \ref{assump:phi-diff}(iii), for any $x,y \in \sM$,
	\[
	\frac{ \left| \phi( \hrho(x), \hrho(y) )  -  \phi( \brho_{r_k}(x), \brho_{r_k}(y) )   \right|}{ \phi( \brho_{r_k}(x), \brho_{r_k}(y) )  }  \leq L_\phi \left( \frac{| \hrho(x) - \brho_{r_k}(x)|}{ \brho_{r_k}(x)}  + \frac{| \hrho(y) - \brho_{r_k}(y)|}{ \brho_{r_k}(y)} \right) \leq 2L_\phi \varepsilon_{\rho,k} \leq 0.1.
	\]
	The last inequality is because of the condition $\varepsilon_{\rho,k} \leq 0.05 / L_\phi$. 
	Since this holds for any $x,y \in \sM$, by the definition of $E_{\phi,k}$ in \eqref{eq:def-vareps-rho-phi},
 we have proved \eqref{eq:bound-phi-rela-0.1}.
\end{proof}

\subsection{Step 1: replacement of $\hrho$ with $\brho_{r_k}$}

\begin{proof}[Proof of Proposition \ref{prop:step1-diff}]
Under the 
condition of the proposition,  \eqref{eq:property-Nrho} holds. Then, when $N > N_\rho$, we have $r_k \leq \tilde{r}_0$, and this implies that $\brho_{r_k}$ is well-defined by Lemma \ref{lemma:bar-rho-epsilon}(i). As a result, $\bar{L}_\un$ (defined in \eqref{eq:def-Wbar-GL}) is well-defined.

In the following lemma, we upper bound $|L_\un f(x)-\bar{L}_\un f(x) |$ by $\varepsilon_{\rho,k}$ multiplied by an independent sum involving population bandwidth $\brho_{r_k}$ instead of $\hrho$.
This lemma will be a key intermediate result towards the proof of the replacement error bound in Proposition \ref{prop:step1-diff}.
The proof of the lemma is given after that of Proposition \ref{prop:step1-diff}.

\begin{lemma}\label{lemma:step1-diff}
	Under Assumptions \ref{assump:M} and \ref{assump:p}, suppose $k_0$ satisfies Assumption \ref{assump:k0-smooth} and 
	$\phi$ satisfies Assumption \ref{assump:phi-diff}(i)(iii), 
	if $r_k \leq r_0$ and $\varepsilon_{\rho,k} \leq \min\{ \delta_\phi, 0.05 / L_\phi\}$, then for any $\epsilon > 0$, any $f: \sM \to \R$, and any $x \in \sM$, 
	\begin{align}   \label{eq:lemma-step1-diff}
		\left|L_\un f(x)-\bar{L}_\un f(x)\right|& \leq 
  \frac{C \varepsilon_{\rho,k}}{\epsilon \brho_{r_k}(x)^2} 
    \frac{1}{N}\sum_{j=1}^N \epsilon^{-d/2}k_1\left(  \frac{\| x - x_j\|^2}{ \epsilon \phi( \brho_{r_k}(x), \brho_{r_k}(x_j))^2 } \right) \left|f(x_j)-f(x)\right|,  
	\end{align}
	where $C = 6 (1 + 3L_\phi) / m_2[k_0]$ is a constant depending on $k_0$ and $\phi$, and
	\begin{equation} \label{eq:def-k1}
		k_1(\eta) : = ( a_1[k_0] \eta + a_0[k_0]) \exp( - \frac{a[k_0]}{2}\eta ), \quad \eta\geq 0.   
	\end{equation}
\end{lemma}
	Here,  $a[k_0], a_0[k_0], a_1[k_0]$ are fixed positive constants determined by $k_0$, see the comments beneath  Assumption \ref{assump:k0-smooth}.

\begin{remark}[Extension to $x=x_i$]\label{rk:replace-lemma1-xi}
The bound \eqref{eq:lemma-step1-diff} holds at $x = x_i$ for any $x_i$ in $X$, because Lemma \ref{lemma:step1-diff} is a deterministic statement that holds under any realization of the data randomness.  
Based on Lemma \ref{lemma:step1-diff}, Proposition \ref{prop:step1-diff} is proved by bounding the r.h.s of \eqref{eq:lemma-step1-diff} using a concentration argument because it is an independent summation (it only involves $\bar{\rho}_{r_k}$ inside $\phi$).
In the case of $x = x_i$, the replacement error bound in Proposition \ref{prop:step1-diff} also holds because one can condition on $x_i$ and apply concentration argument to the summation in \eqref{eq:lemma-step1-diff}, which is a conditional independent sum over $j$ (the term $j=i$ does not contribute to the summation).
\end{remark}

We claim that when $N$ is sufficiently large, whose threshold depends on $(\sM, p, k_0, \phi)$,
for any $r \in [0, \tilde{r}_0]$ and any $x \in \sM$, w.p. $\ge 1-2N^{-10}$, 
 \begin{align}\label{eq:bound-k1-sum}
 	\frac{1}{N}\sum_{j=1}^N \epsilon^{-d/2}k_1\left(  \frac{\| x - x_j\|^2}{ \epsilon \phi( \brho_r(x), \brho_r(x_j))^2 } \right) \left|f(x_j)-f(x)\right| =  O\left(\|\nabla f\|_\infty p(x)^{-1/d}\sqrt{\epsilon}\right) + O(\| f \|_\infty \epsilon^{10}).  
 \end{align}
 Here, the threshold for $N$ and the constant in the big-$O$ notation are uniform for both $r$ and $x$.

Suppose the claim holds,  we apply it at $r=r_k$ and the fixed point $x\in \sM$ as in the proposition, and then \eqref{eq:bound-k1-sum} leads to an upper bound of the r.h.s. of \eqref{eq:lemma-step1-diff} combined with an upper bound of ${(C \varepsilon_{\rho,k})}/{(\epsilon \brho_{r_k}(x)^2)}$. 
To upper bound the latter, 
we use that $\brho_{r_k}(x)^{-2} \leq (2/3)^{-2/d} p(x)^{2/d}$ by Lemma \ref{lemma:bar-rho-epsilon}(ii).
Here, we can use Lemma \ref{lemma:step1-diff} due to \eqref{eq:property-Nrho}.
Putting together, we have 
 \begin{align} \label{eq:bound-k1-sum-2}
 		\left|L_\un f(x)-\bar{L}_\un f(x)\right|&\leq C(2/3)^{-2/d}  p(x)^{2/d} \frac{\varepsilon_{\rho,k}}{\epsilon}  \left( O\left(\|\nabla f\|_\infty p(x)^{-1/d}\sqrt{\epsilon}\right) + O (\|f\|_\infty \epsilon^{10})  \right), 
 \end{align}
 which is $O\left(\|\nabla f\|_\infty p(x)^{1/d} \frac{\varepsilon_{\rho,k}}{\sqrt{\epsilon}} \right) 
+ O( \|f\|_\infty p(x)^{2/d}  \varepsilon_{\rho,k} \epsilon^{9}) $,
 and this holds when $N$ is large enough (exceeding the large-$N$ threshold of the claim, and $N_\rho$ in \eqref{eq:property-Nrho}) and w.p.  $\ge 1-3N^{-10}$
(under the good event of $E_\rho$ in \eqref{eq:property-Nrho} and the claim at $r=r_k$ and $x$).
 This shows \eqref{eq:error-L-barL-un}  and will finish the proof of the proposition.

Below, we focus on the proof of the claim \eqref{eq:bound-k1-sum}.
To bound the l.h.s. of \eqref{eq:bound-k1-sum}, we first truncate the kernel $k_1$. 
We show that $k_1$ satisfies Assumption \ref{assump:k0-smooth} and use the exponential decay of $k_1$ to truncate the kernel.
Specifically,  by the definition of $k_1$ in \eqref{eq:def-k1}, $|k_1^{(l)}(\eta)| \leq a_l[k_1] \exp(-a[k_1] \eta)$ 
 holds with the decay constant
 \begin{align} \label{eq:def-a-k1}
 	a[k_1] \coloneqq \frac{1}{4} a[k_0], 
 \end{align}
 and for $l=0$, the decay condition is satisfied with
 \begin{align} \label{eq:def-a0-k1}
 	a_0[k_1] \coloneqq \frac{4}{e}\frac{ a_1[k_0]}{ a[k_0]  } + a_0[k_0]. 
 \end{align}
 The constants $a[k_1]$, $a_0[k_1]$ are determined by $a[k_0]$ and $a_l[k_0]$.

Meanwhile, by Lemma \ref{lemma:bar-rho-epsilon}(ii) and Assumption \ref{assump:phi-diff}(iv), 
 \begin{align} \label{eq:bound-phi-denominator}
	\phi(\brho_{r}(x), \brho_{r}(y))  \leq c_{\max}  \max\{ \brho_r(x), \brho_r(y)\} \leq c_{\max} \rho_{\max}, \quad \forall x, y \in \sM ,\quad \forall r \in [0, \tilde{r}_0],
 \end{align}
which implies that  ${\|x-y\|^2}/{(\epsilon \phi(\brho_{r}(x), \brho_{r}(y))^2)} \geq {\|x-y\|^2}/{(\epsilon c_{\max}^2 \rho_{\max}^2)}$. 
Then, 
\begin{align}\label{eq:envelope-k1}
	k_1\left(\frac{\|x-y\|^2}{\epsilon \phi(\brho_{r}(x), \brho_{r}(y))^2}\right) 
	\leq a_0[k_1] \exp\left(- \frac{a[k_1] \, \|x-y\|^2}{\epsilon \phi(\brho_{r}(x), \brho_{r}(y))^2}\right)  
	\leq a_0[k_1] \exp\left(- \frac{a[k_1] \,\|x-y\|^2}{\epsilon c_{\max}^2 \rho_{\max}^2} \right),
\end{align}
where in the second inequality we used the monotonicity of $\exp(-a[k_1]\eta)$.
Recall the definition of $\delta_\epsilon[k_1]$ in \eqref{eq:def-delta-eps}, 
let $B_{\delta_\epsilon[k_1]}$ denote $B^m_ {\delta_\epsilon[k_1]}$,
the exponential decay upper bound in \eqref{eq:envelope-k1} implies that
 for any $x, y \in \sM$ and any $r \in [0, \tilde{r}_0]$, 
  \begin{align} \label{eq:k1-trunc}
 	\epsilon^{-d/2}k_1\left(  \frac{\| x - y\|^2}{ \epsilon \phi( \brho_r(x), \brho_r(y))^2 } \right) = \epsilon^{-d/2}k_1\left(  \frac{\| x - y\|^2}{ \epsilon \phi( \brho_r(x), \brho_r(y))^2 } \right) \mathbf{1}_{ \{y \in B_ {\delta_\epsilon[k_1]}(x) \} }    + O( \epsilon^{10} ),
 \end{align}
 where the constant in $O( \epsilon^{10} )$ depends on $\sM$ and is uniform for $x$, $y$, and $r$.

 To proceed, we need to consider a small bandwidth $\epsilon$.
 Let $\epsilon_1 \coloneqq \epsilon_D(k_1, c_{\max}) > 0$, where $\epsilon_D(\cdot,\cdot)$ is defined in \eqref{eq:def-epsD}.
By definition of $\epsilon_D$,  for $\epsilon < \epsilon_1$, we have $\delta_\epsilon[k_1] < \delta_1(\sM)$, where $\delta_1(\sM)$ is introduced in Lemma \ref{lemma:M-metric}(iii).
Since $\delta_\epsilon[k_1] < \delta_1(\sM)$, 
we can apply the metric comparison in \eqref{eq:metric-comp2} in Lemma \ref{lemma:M-metric}(iii).
Together with the truncation in \eqref{eq:k1-trunc} and Lemma \ref{lemma:f-diff-bound} (applied to $f$), we have
 \begin{align} \label{eq:k1-trunc-independ-sum}
 	& \frac{1}{N}\sum_{j=1}^N \epsilon^{-d/2}k_1\left(  \frac{\| x - x_j\|^2}{ \epsilon \phi( \brho_r(x), \brho_r(x_j))^2 } \right) \left|f(x_j)-f(x)\right|  \notag \\
 	&=  \frac{1}{N}\sum_{j=1}^N \epsilon^{-d/2}k_1\left(  \frac{\| x - x_j\|^2}{ \epsilon \phi( \brho_r(x), \brho_r(x_j))^2 } \right) \mathbf{1}_{ \{x_j \in B_ {\delta_\epsilon[k_1]}(x) \} }  \left|f(x_j)-f(x)\right| + R_{\epsilon,T}(x,r) \notag \\ 
 	&\leq \|\nabla f\|_\infty    \frac{1}{N}\sum_{j=1}^N \epsilon^{-d/2}k_1\left(  \frac{\| x - x_j\|^2}{ \epsilon \phi( \brho_r(x), \brho_r(x_j))^2 } \right)  \mathbf{1}_{ \{x_j \in B_ {\delta_\epsilon[k_1]}(x) \} }   d_\sM(x,x_j)  + R_{\epsilon,T}(x,r)  \notag \\
 	&\leq \|\nabla f\|_\infty    1.1  \frac{1}{N}\sum_{j=1}^N \epsilon^{-d/2}k_1\left(  \frac{\| x - x_j\|^2}{ \epsilon \phi( \brho_r(x), \brho_r(x_j))^2 } \right)  \mathbf{1}_{ \{x_j \in B_ {\delta_\epsilon[k_1]}(x) \}  }  \|x-x_j\|  + R_{\epsilon,T}(x,r), 
 	\notag\\
 	&\qquad\qquad\qquad\qquad\qquad\qquad\qquad\qquad  \text{where } \sup_{x\in \sM,r \in [0,\tilde{r}_0]}|R_{\epsilon,T}(x,r)| = O(  \| f\|_\infty  \epsilon^{10} ),
 \end{align}
where the constant in $O( \| f\|_\infty \epsilon^{10})$ depends on $\sM$ and $k_0$ and is uniform for $x$ and $r$.
We used $_T$ in the subscript of $R_{\epsilon, T}$ to stand for ``truncation''.

Below, we bound the independent sum in \eqref{eq:k1-trunc-independ-sum}. We define $Y_j$ as
\begin{align*}
	Y_j \coloneqq \epsilon^{-d/2}k_1\left(  \frac{\| x - x_j\|^2}{ \epsilon \phi( \brho_r(x), \brho_r(x_j))^2 } \right)  
    \mathbf{1}_{ \{  x_j \in B_ {\delta_\epsilon[k_1]}(x) \}}  \|x-x_j\|.
\end{align*}
$\{Y_j\}_{j=1}^N$ are i.i.d. random variables and $Y_j \ge 0$.
We bound $\frac{1}{N}\sum_{j = 1}^N Y_j$ using the concentration argument and the upper bounds of $\E Y_j$ and $\var(Y_j)$.

For $\E Y_j$, we have
\begin{align} \label{eq:bound-bias-15}
	\E Y_j &\leq \int_\sM \epsilon^{-d/2}k_1\left(  \frac{\| x - y\|^2}{ \epsilon \phi( \brho_r(x), \brho_r(y))^2 } \right)  \|x-y\| p(y) dV(y) \notag \\
	&= \sqrt{\epsilon}  \int_\sM \epsilon^{-d/2}k_1\left(  \frac{\| x - y\|^2}{ \epsilon \phi( \brho_r(x), \brho_r(y))^2 } \right)  \frac{\|x-y\|}{ \sqrt{\epsilon}  \phi( \brho_r(x), \brho_r(y)) } \phi( \brho_r(x), \brho_r(y)) p(y) dV(y).
\end{align}
We define 
\begin{align} \label{eq:def-k2}
	k_2(\eta) \coloneqq \frac{a_0[k_1]}{ \sqrt{ea[k_1]} }\exp(-\frac{a[k_1]}{2}\eta), \quad \eta\geq 0,
\end{align}
where $a[k_1]$ and $a_0[k_1]$ are defined in \eqref{eq:def-a-k1} and \eqref{eq:def-a0-k1}.
Then, $k_2$ satisfies $k_1(\eta) \sqrt{\eta} \leq k_2(\eta)$, which can be verified using the exponential decay of $k_1$:
 by the fact $\exp(-\frac{a[k_1]}{2}\eta) \sqrt{\eta} \leq \frac{1}{\sqrt{ea[k_1]}}$, we have 
$k_1(\eta) \sqrt{\eta} \leq a_0[k_1] \exp(-{a[k_1]}\eta) \sqrt{\eta} \leq 
 \frac{a_0[k_1]}{\sqrt{ea[k_1]}} \exp(-\frac{a[k_1]}{2}\eta) = k_2(\eta)$.
 
Therefore, the inequality \eqref{eq:bound-bias-15} continues as
\begin{align*}
	\E Y_j &\leq \sqrt{\epsilon}   \int_\sM \epsilon^{-d/2}k_2\left(  \frac{\| x - y\|^2}{ \epsilon \phi( \brho_r(x), \brho_r(y))^2 } \right)  \phi( \brho_r(x), \brho_r(y)) p(y) dV(y).
\end{align*}
By the upper bound of $\phi( \brho_r(x), \brho_r(y))$ in \eqref{eq:bound-phi-denominator} and the fact that $k_2$ is monotonic,
\[ 
k_2\left(  \frac{\| x - y\|^2}{ \epsilon \phi( \brho_r(x), \brho_r(y))^2 } \right)  
\le k_2\left( \frac{\| x - y\|^2}{ \epsilon  c_{\max}^2\max\{ \brho_r(x), \brho_r(y) \}^2 } \right). 
\]
Combined with \eqref{eq:bound-phi-denominator} and the fact that $\max\{ \brho_r(x), \brho_r(y) \} \le \brho_r(x) + \brho_r(y)$, $\E Y_j$ can further be bounded by
\begin{align} \label{eq:bound-bias-13}
	\E Y_j 
	&\leq \sqrt{\epsilon}  c_{\max}  \int_\sM \epsilon^{-d/2}k_2\left( \frac{\| x - y\|^2}{ \epsilon  c_{\max}^2\max\{ \brho_r(x), \brho_r(y) \}^2 } \right) (\brho_r(x) + \brho_r(y) ) p(y) dV(y).
\end{align}
We will apply Lemma \ref{lemma:right-operator-degree}(ii) to bound \eqref{eq:bound-bias-13}: $k_2$ satisfies Assumption \ref{assump:k0-smooth} with the decay constant $a[k_2] \coloneqq a[k_1] / 2 = a[k_0] / 8$, and $m_0[k_2] > 0$ is a constant depending on $\sM$ and $k_0$. 
Let $ \epsilon_2 \coloneqq \epsilon_D(k_2, 1) / c_{\max}^2 > 0$, where $\epsilon_D$ is defined in \eqref{eq:def-epsD}, then
for $\epsilon < \epsilon_2$, we have 
$c_{\max}^2 \epsilon < \epsilon_D(k_2, 1)$, which is the small-$\epsilon$ threshold required in Lemma \ref{lemma:right-operator-degree}(ii) (with $\epsilon$ in the lemma being $\epsilon c_{\max}^2$ and $k_0$ in the lemma being $k_2$).
We apply \eqref{eq:ker-expan-degree-slowI} in Lemma \ref{lemma:right-operator-degree}(ii) (specifically, we use $i=0$ to analyze the integral involving $\brho_r(x)$ and $i=1$ to analyze the integral involving $\brho_r(y)$ in \eqref{eq:bound-bias-13}), and have
\begin{align*} 
	\E Y_j &\leq 2 \sqrt{\epsilon}  c_{\max}^{d+1} \left( m_0[k_2] p \brho_r^{d+1}(x) + R_{\epsilon,1}(x,r) \right), \quad \sup_{x\in \sM,r \in [0,\tilde{r}_0]}|R_{\epsilon,1}(x,r)| 
	= 
	{O^{[p]}(\sqrt{\epsilon})}, 
\end{align*}
where the constant dependence in 
${O^{[p]}(\sqrt{\epsilon})}$ is as in Lemma \ref{lemma:right-operator-degree}(ii) and this 
constant depends on $\phi$ only through $c_{\max}$.
There is $\epsilon_3 > 0$ such that for $\epsilon < \epsilon_3$ (this threshold is uniform for $x$ and $r$),
\[ R_{\epsilon,1}(x,r) \leq (2/3)^{(d+1)/d} m_0[k_2] p_{\max}^{-1/d} \leq m_0[k_2] p \brho_r^{d+1}(x), \quad \forall x \in \sM, \forall r \in [0, \tilde{r}_0]. \]
The last inequality used the lower bound of $\brho_r$, i.e., $\brho_r(x) \geq (2/(3p(x)))^{1/d}$ from Lemma \ref{lemma:bar-rho-epsilon}(ii). By the upper bound of $R_{\epsilon,1}(x,r)$ above and the fact that $\brho_r(x) \leq (2/p(x))^{1/d}$ from Lemma \ref{lemma:bar-rho-epsilon}(ii),
\begin{align} \label{eq:bound-bias-1}
	 \E Y_j \leq   4 c_{\max}^{d+1} m_0[k_2] p \brho_r^{d+1}(x) \sqrt{\epsilon}  \leq C_{E,1}(\sM, k_0, \phi)  p(x)^{-1/d} \sqrt{\epsilon},  
\end{align}
where $C_{E,1}(\sM, k_0, \phi) =  2^{(d+1)/d+2} c_{\max}^{d+1}  m_0[k_2]$. We used $_E$ in the subscript of $C_{E,1}$ to stand for ``expectation''.

For $\var(Y_j)$, we have 
\begin{align*} 
	\epsilon^{d/2} \var(Y_j) &\leq  \epsilon^{d/2} \E Y_j^2  \leq \int_\sM \epsilon^{-d/2}k_1^2\left(  \frac{\| x - y\|^2}{ \epsilon \phi( \brho_r(x), \brho_r(y))^2 } \right)  \|x-y\|^2 p(y) dV(y)  \notag \\
	&= \epsilon   \int_\sM \epsilon^{-d/2}k_1^2\left(  \frac{\| x - y\|^2}{ \epsilon \phi( \brho_r(x), \brho_r(y))^2 } \right)  \frac{\|x-y\|^2}{\epsilon \phi( \brho_r(x), \brho_r(y))^2 } \phi( \brho_r(x), \brho_r(y))^2 p(y) dV(y). 
\end{align*}
Note that $k_1^2$ satisfies Assumption \ref{assump:k0-smooth} with $a_0[k_1^2] \coloneqq a_0[k_1]^2$ and the decay constant $a[k_1^2] \coloneqq 2 a[k_1] $.
We define 
\begin{align} \label{eq:def-k3}
	k_3(\eta) \coloneqq  \frac{2}{e}\frac{a_0[k_1^2]}{ a[k_1^2] }\exp(-\frac{a[k_1^2]}{2}\eta), \quad \eta\geq 0,
\end{align}
then by the fact that $\exp(-\frac{a[k_1^2]}{2}\eta) \eta \le \frac{2}{ea[k_1^2]}, \forall \eta \ge 0$, we have that $k_1^2(\eta) \eta \leq a_0[k_1^2]\exp(-a[k_1^2]\eta) \eta \leq k_3(\eta)$. 
$k_3$ is monotonic, satisfies Assumption \ref{assump:k0-smooth} with $a[k_3] \coloneqq a[k_1^2] / 2 = a[k_0] / 4$, and $m_0[k_3] > 0$ is a constant depending on $\sM$ and $k_0$.
Since $k_3$ is monotone, $\var(Y_j)$ can be further bounded by
\begin{align} \label{eq:bound-var2}
	\epsilon^{d/2} \var(Y_j)  &\leq \epsilon  c_{\max}^2  \int_\sM \epsilon^{-d/2}k_3\left(  \frac{\| x - y\|^2}{ \epsilon \phi( \brho_r(x), \brho_r(y))^2 } \right)   \max\{ \brho_r(x), \brho_r(y)\}^2 p(y) dV(y) \notag \\
	&\leq \epsilon  c_{\max}^2  \int_\sM \epsilon^{-d/2}k_3\left( \frac{\| x - y\|^2}{ \epsilon  c_{\max}^2\max\{ \brho_r(x), \brho_r(y) \}^2 } \right) (\brho_r(x)^2 + \brho_r(y)^2 ) p(y) dV(y).
\end{align}
We now apply Lemma \ref{lemma:right-operator-degree}(ii) to compute the integral in \eqref{eq:bound-var2}.
Let $ \epsilon_4 \coloneqq \epsilon_D(k_3, 1) / c_{\max}^2 > 0$, where $\epsilon_D$ is defined in \eqref{eq:def-epsD}. Then,
for $\epsilon < \epsilon_4$, we have $c_{\max}^2 \epsilon <  \epsilon_D(k_3, 1)$, which is the small-$\epsilon$ threshold required in Lemma \ref{lemma:right-operator-degree}(ii) (with $\epsilon$ in the lemma being $\epsilon c_{\max}^2$ and $k_0$ in the lemma being $k_3$).
We apply \eqref{eq:ker-expan-degree-slowI} in Lemma \ref{lemma:right-operator-degree}(ii) (specifically, we use $i=0$ to analyze the integral involving $\brho_r(x)^2$ and $i=2$ to analyze the integral involving $\brho_r(y)^2$ in \eqref{eq:bound-var2}), and have
\begin{align*} 
	\epsilon^{d/2} \var(Y_j)   &\leq   2\epsilon c_{\max}^{d+2} \left( m_0[k_3] p \brho_r^{d+2}(x) + R_{\epsilon,2}(x,r) \right), \quad \sup_{x\in \sM,r \in [0,\tilde{r}_0]}|R_{\epsilon,2}(x,r)| 
		= 
		{O^{[p]}(\sqrt{\epsilon})}, 
\end{align*}
where the constant in 
${O^{[p]}(\sqrt{\epsilon})}$ depends on $\phi$ only through $c_{\max}$.
There is $\epsilon_5 > 0$ such that for $\epsilon < \epsilon_5$ (this threshold is uniform for $x$ and $r$),
\[ 
R_{\epsilon,2}(x,r) \leq (2/3)^{(d+2)/d} m_0[k_3] p_{\max}^{-2/d} \leq  m_0[k_3] p \brho_r^{d+2}(x),\quad  \forall x \in \sM,  \quad \forall r \in [0, \tilde{r}_0]. 
\]
The last inequality used the lower bound of $\brho_r$, i.e., $\brho_r(x) \geq (2/(3p(x)))^{1/d}$ from Lemma \ref{lemma:bar-rho-epsilon}(ii). By the upper bound of $R_{\epsilon,2}(x,r)$ above and the fact that $\brho_r(x) \leq (2/p(x))^{1/d}$ from Lemma \ref{lemma:bar-rho-epsilon}(ii),
\begin{align} \label{eq:bound-var-6}
	\var(Y_j) \leq   4 c_{\max}^{d+2} m_0[k_3] p \brho_r^{d+2}(x) \epsilon^{1-d/2}  \leq  \mu_Y(x) \coloneqq C_{V,1}(\sM, k_0, \phi)  p(x)^{-2/d} \epsilon^{1-d/2},
\end{align}
where $C_{V,1}(\sM, k_0, \phi) =  2^{(d+2)/d+2} c_{\max}^{d+2}  m_0[k_3]$.
Here, we used  $_V$ in the subscript of $C_{V,1}$ to represent ``variance''.

Next, we show the boundedness of $|Y_j|$. Since $k_1(\eta) \sqrt{\eta} \le k_2(\eta)$, we have
\begin{align*}
		|Y_j|  &= \epsilon^{-d/2}k_1\left(  \frac{\| x - x_j\|^2}{ \epsilon \phi( \brho_r(x), \brho_r(x_j))^2 } \right)  \mathbf{1}_{ \{x_j \in B_ {\delta_\epsilon[k_1]}(x) \}  }  \|x-x_j\| \\
        &\leq \epsilon^{1/2-d/2}k_1\left(  \frac{\| x - x_j\|^2}{ \epsilon \phi( \brho_r(x), \brho_r(x_j))^2 } \right)  \frac{\|x-x_j\|}{\sqrt{\epsilon} \phi( \brho_r(x), \brho_r(x_j)) } \phi( \brho_r(x), \brho_r(x_j))\\
		&\leq  \epsilon^{1/2-d/2}k_2\left(  \frac{\| x - x_j\|^2}{ \epsilon \phi( \brho_r(x), \brho_r(x_j))^2 } \right) \phi( \brho_r(x), \brho_r(x_j)) \\
		&\leq a_0[k_2] c_{\max} \rho_{\max} \epsilon^{1/2-d/2},
\end{align*}
where in the last inequality, we used the bound of $\phi( \brho_r(x), \brho_r(x_j))$ in \eqref{eq:bound-phi-denominator} and the fact that $k_2$ satisfies Assumption \ref{assump:k0-smooth}.
Thus,
	\begin{align} \label{eq:bound1}
				|Y_j|   \leq L_Y \coloneqq C_{B,1}(\sM, p, k_0, \phi) \epsilon^{1/2-d/2},
		\end{align}
	where  $C_{B,1}(\sM, p, k_0, \phi ) 
	= a_0[k_2] c_{\max}\rho_{\max}.$
Here, we used $_B$ in the subscript of $C_{B,1}$ to denote ``boundedness''.

By the condition 
 that $\epsilon^{d/2} = \Omega(\log N/N)$ in the proposition, 
 when $N$ is sufficiently large 
 (whose threshold depends on $(\sM, p, k_0, \phi)$), 
 we have $ 40 \frac{\log N}{N}  \leq 9 \frac{C_{V,1}}{(C_{B,1})^2} p_{\max}^{-2/d} \epsilon^{d/2} \leq 9 \frac{\mu_Y(x)}{ L_Y^2 }$ for any $x \in \sM$.
Then, by Bernstein inequality, w.p. $\ge 1-2N^{-10}$,
\begin{align} \label{eq:concen1}
	0 \leq  \frac{1}{N} \sum_{j = 1}^N Y_j   \leq \E Y_j + \sqrt{40 \frac{\log N}{N}  \mu_Y(x)} \leq p(x)^{-1/d}  \left( C_{E,1} \sqrt{\epsilon} + \sqrt{40 C_{V,1}}  \sqrt{ \frac{\log N}{N \epsilon^{d/2-1}}}   \right).
\end{align}
Since $\epsilon^{d/2}= \Omega(\log N/N)$, for sufficiently large $N$, $ \frac{\log N}{N \epsilon^{d/2-1}} \leq \epsilon$. This implies that
\begin{align} \label{eq:concen2}
	0 \leq   \frac{1}{N} \sum_{j = 1}^N Y_j   \leq (C_{E,1} + \sqrt{40 C_{V,1}}) p(x)^{-1/d}  \sqrt{\epsilon} = O (p(x)^{-1/d}  \sqrt{\epsilon}) .
\end{align}
The constant in the big-$O$ notation depends on $\sM, k_0, \phi$ through $C_{E,1}$ and $C_{V,1}$ and is uniform for $x$ and $r$. 
Thus, by plugging the upper bound of $\frac1N \sum_{j=1}^N Y_j$ in \eqref{eq:concen2} back to the r.h.s. of \eqref{eq:k1-trunc-independ-sum}, \eqref{eq:bound-k1-sum} holds. 

Since $\epsilon = o(1)$, 
we have $\epsilon < \min\{  \epsilon_1, \epsilon_2, \dots, \epsilon_5 \}$ when $N$ exceeds some threshold (which depends on $(\sM, p, k_0, \phi)$ and is uniform for $x$ and $r$),
and then the bounds in \eqref{eq:k1-trunc-independ-sum}\eqref{eq:bound-bias-1}\eqref{eq:bound-var-6}\eqref{eq:bound1} hold;
The other large-$N$ thresholds are required in \eqref{eq:concen1}\eqref{eq:concen2}, which are uniform for $x$ and $r$.
Thus, for the claim in \eqref{eq:bound-k1-sum} to hold,
all the involved large-$N$ thresholds at most depend on $(\sM, p, k_0, \phi)$.
\end{proof}

To prove Lemma \ref{lemma:step1-diff}, we use the following bound, which holds by the definitions of $L_\un$ in \eqref{eq:def-L-un} and $\bar{L}_\un$ in \eqref{eq:def-Wbar-GL} and the triangle inequality,
\begin{align}  \label{eq:bound-L-Lbar}
	\left|L_\un f(x)-\bar{L}_\un f(x)\right|
	\leq M_1(x) + M_2(x),
\end{align}
where
\begin{align*}
	M_1(x)   &\coloneqq    \frac{1}{\frac{m_2}{2}N\epsilon \hrho(x)^2 }\sum_{j=1}^N 
	\epsilon^{-d/2}  \left| k_0 \left(  \frac{\| x - x_j\|^2}{ \epsilon \phi(\hrho(x), \hrho(x_j))^2 } \right)- k_0 \left(  \frac{\| x - x_j\|^2}{ \epsilon \phi( \brho_{r_k}(x), \brho_{r_k}(x_j))^2 } \right)\right|     \left|f(x_j)-f(x)\right|, \\
	M_2(x)   &\coloneqq    \frac{1}{\frac{m_2}{2}N\epsilon  } \left |  \frac{1}{\hrho(x)^2} - \frac{1}{\brho_{r_k}(x)^2}  \right|    \sum_{j=1}^N 
	\epsilon^{-d/2}   k_0 \left(  \frac{\| x - x_j\|^2}{ \epsilon \phi( \brho_{r_k}(x), \brho_{r_k}(x_j))^2 } \right)     \left|f(x_j)-f(x)\right|.
\end{align*}

\begin{proof}[Proof of Lemma \ref{lemma:step1-diff}]

	Since $r_k \leq r_0$, $\brho_{r_k}(x)$ for all $x\in \calM$  is well-defined by Lemma \ref{lemma:bar-rho-epsilon}(i), and consequently,  $\bar{L}_\un$ is also well-defined.
	By \eqref{eq:bound-L-Lbar}, it suffices to bound the two terms $M_1(x)$ and $M_2(x)$ on the r.h.s.
	
	We first bound $M_2(x)$.
 The lemma assumes that  $\varepsilon_{\rho,k} \leq \min\{ \delta_\phi, 0.05 / L_\phi\}$
 and $\phi$ satisfies Assumption \ref{assump:phi-diff}(i)(iii),
 and this means that Lemma \ref{lemma:lb-ub-E_phi} holds. 
 Therefore,
	by the inequalities \eqref{eq:bound-phi-rela-left}\eqref{eq:bound-phi-rela-0.1} in Lemma \ref{lemma:lb-ub-E_phi}, $\varepsilon_{\rho,k} \leq E_{\phi,k} \leq 0.1$.
	Then, by Lemma \ref{lemma:rela-err-uv-0.1} applied to $u = \brho_{r_k}(x)$ and $v = \hrho(x)$, we have
	\begin{equation}  \label{eq:bound-1/hrho-rela}
		\left |  \frac{1}{\hrho(x)^2} - \frac{1}{\brho_{r_k}(x)^2}  \right| 
		\leq  \frac{3\varepsilon_{\rho,k}}{\brho_{r_k}(x)^2}, \quad \forall x \in \sM.  
	\end{equation}
	Therefore, by Assumption \ref{assump:k0-smooth}(ii), 
	\begin{align} \label{eq:bound-circ1}
		M_2(x) \leq  C_1  \varepsilon_{\rho,k}\,  \frac{1}{N\epsilon \brho_{r_k}(x)^2}\sum_{j=1}^N \epsilon^{-d/2} a_0[k_0]\exp\left( - \frac{ a[k_0]\, \| x - x_j\|^2}{ \epsilon \phi( \brho_{r_k}(x), \brho_{r_k}(x_j))^2 } \right) \left|f(x_j)-f(x)\right|,
	\end{align}
	where $C_1 = 6 / m_2$ is a constant depending on $k_0$. 
	
	To bound $M_1(x)$, we apply Taylor expansion to $k_0$ and obtain that
	\begin{align*}
		& \epsilon^{-d/2}  \left| k_0 \left(  \frac{\| x - x_j\|^2}{ \epsilon  } \frac{1}{\phi(\hrho(x), \hrho(x_j))^2} \right)- 
  k_0 \left(  \frac{\| x - x_j\|^2}{ \epsilon  } \frac{1}{\phi( \brho_{r_k}(x), \brho_{r_k}(x_j))^2} \right)\right|    \\
		&\leq \epsilon^{-d/2}   \frac{\| x - x_j\|^2}{\epsilon}     \left|   \frac{ 1 }{  \phi(\hrho(x), \hrho(x_j))^2} - \frac{1}{\phi( \brho_{r_k}(x), \brho_{r_k}(x_j))^2} \right|       \left| k_0' \left(  \frac{\| x - x_j\|^2}{ \epsilon } \xi  \right) \right|,
	\end{align*}
	where 
	\[ \xi \geq \min \left\{ \frac{ 1 }{  \phi(\hrho(x), \hrho(x_j))^2}, \frac{1}{\phi( \brho_{r_k}(x), \brho_{r_k}(x_j))^2}  \right\}. \]
	Since $E_{\phi,k} \leq 0.1$, by Lemma \ref{lemma:rela-err-uv-0.1} applied to $u = \phi( \brho_{r_k}(x), \brho_{r_k}(x_j))$ and $v = \phi(\hrho(x), \hrho(x_j))$, we have 
	\begin{align} \label{eq:bound-1/phi-hrho-rela}
		& \left|   \frac{ 1 }{  \phi(\hrho(x), \hrho(x_j))^2} - \frac{1}{\phi( \brho_{r_k}(x), \brho_{r_k}(x_j))^2} \right| \leq \frac{ 3 E_{\phi,k} }{ \phi(\brho_{r_k}(x), \brho_{r_k}(x_j))^2} \notag \\
		&\leq   \frac{ 6 L_\phi \varepsilon_{\rho,k} }{ \phi(\brho_{r_k}(x), \brho_{r_k}(x_j))^2}  \leq \frac{ 0.3 }{ \phi(\brho_{r_k}(x), \brho_{r_k}(x_j))^2}, \quad \forall x\in\calM,  \quad \forall j=1,\cdots, N.
	\end{align}
	The last two inequalities used \eqref{eq:bound-phi-rela-0.1} in Lemma \ref{lemma:lb-ub-E_phi} and the condition $\varepsilon_{\rho,k} \leq 0.05/L_\phi$, respectively.
	As a result, 
	\[ \xi \geq \frac{1}{2  \phi( \brho_{r_k}(x), \brho_{r_k}(x_j))^2 }, \]
	and
	\begin{align*} 
		 \left|  \frac{\phi( \brho_{r_k}(x), \brho_{r_k}(x_j))^2}{  \phi(\hrho(x), \hrho(x_j))^2} -  1 \right| \leq 6 L_\phi \varepsilon_{\rho,k},
   \quad 
   \forall x\in\calM, \quad \forall j=1,\cdots, N.
	\end{align*}
	Therefore, by the exponential decay of $|k_0'|$ (Assumption \ref{assump:k0-smooth}(ii)) and the monotonicity of the exponential function $\exp(-a[k_0] \eta)$, we have
	\begin{align} \label{eq:bound-k0-diff-1}
		& \left| k_0 \left(  \frac{\| x - x_j\|^2}{ \epsilon \phi(\hrho(x), \hrho(x_j))^2 } \right)- k_0 \left(  \frac{\| x - x_j\|^2}{ \epsilon \phi( \brho_{r_k}(x), \brho_{r_k}(x_j))^2 } \right)\right|  \notag \\
		&\leq  6L_\phi \varepsilon_{\rho,k} \, a_1[k_0] \frac{\| x - x_j\|^2}{ \epsilon \phi( \brho_{r_k}(x), \brho_{r_k}(x_j))^2 }    \exp \left(  -\frac{a[k_0]}{2}\frac{\| x - x_j\|^2}{ \epsilon \phi( \brho_{r_k}(x), \brho_{r_k}(x_j))^2 } \right).
	\end{align}
	Additionally, since $\varepsilon_{\rho,k} \leq 0.1$, we have
	\[ \frac{1}{\hrho(x)^2} \leq \frac{1}{0.9^2 \brho_{r_k}(x)^2 }, \quad \forall x \in \sM. \]
	Therefore, we have the following bound for $M_1(x)$,
	\begin{align*}
		M_1(x) \leq  &\frac{C_2 \varepsilon_{\rho,k}}{N\epsilon \brho_{r_k}(x)^2}\sum_{j=1}^N \epsilon^{-d/2} a_1[k_0] \exp \left(  -\frac{a[k_0]}{2}\frac{\| x - x_j\|^2}{ \epsilon \phi( \brho_{r_k}(x), \brho_{r_k}(x_j))^2 } \right) \frac{\| x - x_j\|^2}{ \epsilon \phi( \brho_{r_k}(x), \brho_{r_k}(x_j))^2 }  \\
		&\left|f(x_j)-f(x)\right|,
	\end{align*}
	where $C_2 = 18 L_\phi / m_2$ is a constant depending on $k_0$ and $\phi$. 
	
	Together with the bound of \circled{2} in \eqref{eq:bound-circ1} and using the fact that $\exp(-a[k_0]\eta) \leq \exp(-\frac{a[k_0]}{2}\eta)$, we have that
	\begin{align*}
		& \left|L_\un f(x)-\bar{L}_\un f(x)\right| \leq M_1(x) + M_2(x) \notag \\
		&\leq C\varepsilon_{\rho,k}  \; \frac{1}{N\epsilon \brho_{r_k}(x)^2}\sum_{j=1}^N \epsilon^{-d/2}k_1\left(  \frac{\| x - x_j\|^2}{ \epsilon \phi( \brho_{r_k}(x), \brho_{r_k}(x_j))^2 } \right) \left|f(x_j)-f(x)\right|,    
	\end{align*}
	where $k_1$ is as defined in \eqref{eq:def-k1}, and $C = C_1 + C_2 = 6 (1 + 3L_\phi) / m_2$.
\end{proof}

\subsection{Step 2: convergence of $\bar{L}_\un $}

\begin{proof}[Proof of Lemma \ref{lemma:step2-diff}]

 We prove the lemma in five steps as follows.
	
\vspace{5pt}
\noindent
 \underline{Step 1.}
		Local truncation on $\sM$.  

We truncate the integral domain to be the intersection of $\calM$ with a Euclidean ball centered at $x$ and of radius $\delta_\epsilon[k_0]$.
	Recall the definition of $\delta_\epsilon$ in \eqref{eq:def-delta-eps} and of $\epsilon_D$ in \eqref{eq:def-epsD}.
	By the definition of $\delta_\epsilon$, for $\epsilon < \epsilon_D(k_0, c_{\max})$, we have $\delta_\epsilon[k_0] < \delta_1(\sM)$, where $\delta_1(\sM)$ is introduced in Lemma \ref{lemma:M-metric}(iii).
	In the rest of the proof of the lemma, we omit the dependence on $k_0$ in $\delta_\epsilon[k_0]$.

	By the choice of $\delta_\epsilon$, Assumption \ref{assump:k0-smooth}(ii) and the bound of $\phi( \brho_r(x), \brho_r(y) )$ in \eqref{eq:bound-phi-denominator}, we have that  when $y \notin B^m_{\delta_\epsilon}(x)$, $\epsilon^{-d/2}  k_0(  \frac{\| x - y\|^2}{  \epsilon  \phi( \brho_r(x), \brho_r(y) )^2  } ) = O(\epsilon^{10})$. Hence, by the compactness of $\sM$,
	\begin{align} \label{eq:lemma-bias-fast-step1}
		&\brho_r(x)^{-2} \epsilon^{-d/2} \int_{\sM} 
		{  k_0 \left(  \frac{\| x - y\|^2}{  \epsilon  \phi( \brho_r(x), \brho_r(y) )^2  } \right)}   (f(y)-f(x)) p(y)  dV(y) =  \circled{1} + O^{[f,p]}(\epsilon^{10}), 
	\end{align}
	where
	\begin{align*}
		\circled{1} \coloneqq  \brho_r(x)^{-2} \epsilon^{-d/2} \int_{B^m_{\delta_\epsilon}(x) \cap \sM} 
		{  k_0 \left(  \frac{\| x - y\|^2}{  \epsilon  \phi( \brho_r(x), \brho_r(y) )^2  } \right)}   (f(y)-f(x)) p(y)  dV(y).   
	\end{align*}

\vspace{5pt}
\noindent
 \underline{Step 2.}
		Change of variable from $y$ to the projected coordinate $u$.

	Since $\delta_\epsilon < \delta_1(\sM)$, $\sM \cap B^m_{\delta_\epsilon}(x)$ is isomorphic to a ball in $\R^d$. Let $u$ and $s$ be the projected and normal coordinates of $y$ introduced in Appendix \ref{app:diff-geo-lemmas}. Namely, $u = \varphi_x(y), s = \exp_x^{-1}(y)$.
	We will show that the change of variable from $y$ to $u$ results in
	\begin{equation} \label{eq:lemma-bias-fast-step2-1}
		\circled{1} = \circled{2}, 
	\end{equation}
	where \circled{2} is as defined in \eqref{eq:lemma-bias-fast-step2}. To achieve this, we need the Taylor expansion of the functions on $\sM$ with respect to $u$, which will be detailed below.

	We first define the functions $\tilde{p} \coloneqq p \circ \exp_x$ and $\tilde{f} \coloneqq f \circ \exp_x$ mapping from $\R^d$ to $\R$, where we identify $T_x\sM$ with $\R^d$.
	Additionally, we introduce the ratio function $\beta_x: \sM \to \R_+$:
	\begin{equation} \label{eq:def-beta}
		\beta_x(y) :=  \brho_r(y)/  \brho_r(x),
	\end{equation}
	where the dependence on $r$ is omitted for conciseness.
	We also define $\tilde{\beta}_x: \R^d \to \R_+$ by $\tilde{\beta}_x \coloneqq \beta_x \circ \exp_x$. By definition, $\tilde\beta_x(0) = 1$.
	
	We apply Taylor expansion to the functions $\tilde{p}$, $\tilde{f}$, and $\tilde{\beta}_x$.
	By applying Taylor expansion to $\tilde{f}$, 
	\begin{align}  \label{eq:f-taylor-4th-pre}
		\tilde{f}(s) = \tilde{f}(0) + \nabla \tilde{f}(0) \cdot s + \frac{1}{2}  s^T \Hess_{\tilde{f}}(0) s + Q_{x,3}^{(f)}(s)  + O^{[f]}( \|s\|^4),
	\end{align}
	where $\Hess_{\tilde{f}}(0)$ denotes the $d$-by-$d$ Hessian matrix of the function $\tilde{f}$ at $0$.
	The constant in the big-$O$ notation depends on $\sM$ and $\| \nabla^l f \|_\infty, l\leq 4$, because by the definition $\tilde{f} = f \circ \exp_x $, the fourth derivative of $\tilde{f}$ depends on $\sM$ and the derivatives of $f$ and $\exp_x$ up to the fourth order. 
	In \eqref{eq:f-taylor-4th-pre}, $Q_{x,3}^{(f)}$ denotes a homogeneous polynomial of degree $3$, with coefficients depending on $f$ and its derivatives evaluated at $x$. In the latter context, we use similar notations to denote homogeneous polynomials whose coefficients depend on a function on $\sM$.	

	By Lemma \ref{lemma:M-coord-expansion}, $s_i = u_i + Q_{x,3}^{(N)}(u) +  O(\|u\|^4)$.  Therefore, 
	\begin{equation} \label{eq:f-taylor-4th}
	f(y) - f(x) = \nabla \tilde{f}(0) \cdot u  +  \frac{1}{2} u^T \Hess_{\tilde{f}}(0) u  +  Q_{x,3}^{(f)}(u) + O^{[f]}( \|u\|^4),   
	\end{equation}
	where the constant in the big-$O$ notation depends on both $\sM$ and $\|\nabla^l f\|_\infty, l\leq 4$. 
	Here, $Q_{x,3}^{(f)}(u)$ is a different polynomial from the one in \eqref{eq:f-taylor-4th-pre} whose coefficients additionally depend on $\sM$. As the specific coefficients do not impact the limiting operator, we use the same symbol to avoid introducing new notations.

	For $\tilde{p}$, by a similar argument used to obtain \eqref{eq:f-taylor-4th}, we have its third order expansion as follows,
	\begin{equation} \label{eq:p-taylor}
		p(y) = \tilde{p}(0) + \nabla \tilde{p}(0) \cdot s + Q_{x,2}^{(p)}(s) + O^{[p]}( \|s\|^3) = \tilde{p}(0) + \nabla \tilde{p}(0) \cdot u  + Q_{x,2}^{(p)}(u) + O^{[p]}( \|u\|^3),   
	\end{equation}
	where the constant in the big-$O$ notation depends on both $\sM$ and $\|\nabla^l p\|_\infty, l\leq 3$.

	By Lemma \ref{lemma:bar-rho-epsilon}(i), $\brho_{r}$ is $C^3$ on $\sM$, which implies that $\beta_x$ is also $C^3$ on $\sM$.  By a similar argument used to obtain \eqref{eq:f-taylor-4th}, we apply Taylor expansion to $\tilde{\beta}_x$ up to the third order and obtain
	\begin{equation} \label{eq:beta-taylor}
		\beta_x(y) = 1 + \nabla \tilde\beta_x(0) \cdot u + Q_{x,2}^{(\beta)}(u) + O^{[p]}(\|u\|^3).
	\end{equation}
	By Lemma \ref{lemma:bar-rho-epsilon}(iv), the constant in the big-$O$ notation depends on $\sM$, $p_{\min}, p_{\max}$, and $\|\nabla^l p\|_\infty, l\leq 5$.
	The superscript $^{(\beta)}$ distinguishes the quadratic polynomial in the expansion of $\beta_x$.
	Then, based on \eqref{eq:beta-taylor}, we consider the expansion of $\phi(\brho_{r}(x), \brho_{r}(y))$ at $\phi(\brho_{r}(x), \brho_{r}(x)) = \brho_{r}(x)$. 
	Since $\phi$ is $C^3$ on $\R_+ \times \R_+$, $\partial_1 \phi(\eta,\eta) = \frac{1}{2}, \forall \eta \in \R_+$, we have
	\begin{align}  \label{eq:phi/rho-taylor-pre}
		\phi(\brho_r(x), \brho_r(y) ) &= \phi(\brho_r(x), \brho_r(x)\beta_x(y) ) = \phi\left(\brho_r(x), \brho_r(x) \left(1 + \nabla \tilde\beta_x(0) \cdot u + Q_{x,2}^{(\beta)}(u) + O^{[p]}(\|u\|^3)  \right) \right) \notag \\
		&= \brho_r(x) + \frac{1}{2} \brho_r(x) \nabla \tilde\beta_x(0) \cdot u + Q_{x,2}^{(\phi)}(u) + O^{[p]}(\|u\|^3).
	\end{align}
	Here, the coefficients of $Q_{x,2}^{(\phi)}$ depend on both $\brho_r$ and $\phi$.
	The constant in the big-$O$ notation depends on $\sM$, $p_{\min}, p_{\max}$, $\|\nabla^l p\|_\infty, l \le 5$, 
	and 
	\begin{equation}\label{eq:def-C-phi-l}
	C_{\phi, l} \coloneqq \sup_{u, v \in I} | \frac{\partial^l}{\partial u^l} \phi(u, v)|, 
	\quad
	 I = [\rho_{\min}, \rho_{\max}],
	\end{equation}
		for $l \le 3$.
	Hence,
	\begin{align} \label{eq:phi/rho-taylor}
		\frac{\brho_r(x) }{ \phi(\brho_r(x), \brho_r(y)) } = 1 - \frac{1}{2}\nabla \tilde\beta_x(0)  \cdot u + Q_{x,2}^{(\rho/\phi)}(u) + O^{[p]}(\|u\|^3),
	\end{align}
	where $Q_{x,2}^{(\rho/\phi)}$ is another quadratic polynomial, distinct from $Q_{x,2}^{(\beta)}$ and $Q_{x,2}^{(\phi)}$.
	Consequently, combining \eqref{eq:phi/rho-taylor} with \eqref{eq:metric-comp1}, we obtain
	\begin{equation} \label{eq:phi-taylor}
		\frac{ \|x-y\|^2 }{ (\phi(\brho_r(x), \brho_r(y)) / \brho_r(x))^2 } = \|u\|^2 -  \|u\|^2 \nabla \tilde\beta_x(0) \cdot u + Q_{x,4}^{(\rho)}(u) + O^{[p]}(\|u\|^5),    
	\end{equation}
	where the coefficients of $Q_{x,4}^{(\rho)}(u) $ depend on $\brho_r$, $\phi$, and also $\sM$.
	
	Plugging the volume comparison \eqref{eq:volume-comp} and the expansions in  \eqref{eq:f-taylor-4th}, \eqref{eq:p-taylor}, and \eqref{eq:phi-taylor} into \circled{1}, we have that
	\begin{align*}
		\circled{1}  &=   \brho_r(x)^{-2}  \epsilon^{-d/2} \int_{\varphi_x\left( B^m_{\delta_\epsilon}(x) \cap \sM \right) } 
		k_0 \left(    \frac{1}{  \epsilon \brho_r(x)^2    } \left(  \|u\|^2 -  \|u\|^2 \nabla \tilde\beta_x(0) \cdot u + Q_{x,4}^{(\rho)}(u) + O^{[p]}(\|u\|^5) \right) \right)   \\
		& \quad   \left(\nabla \tilde{f}(0) \cdot u + \frac{1}{2} u^T \Hess_{\tilde{f}}(0) u +  Q_{x,3}^{(f)}(u) + O^{[f]}(\|u\|^4) \right)  \left(\tilde{p}(0)  +  \nabla \tilde{p}(0) \cdot u +  Q_{x,2}^{(p)}(u) + O^{[p]}(\|u\|^3)\right)  \\
		&   \quad    \left( 1 + Q_{x,2}^{(V)}(u) + O(\|u\|^3)\right)du.
	\end{align*}	
	By Assumption \ref{assump:k0-smooth}, $k_0$ is $C^3$. We apply Taylor expansion to $k_0$ up to the third order and obtain the expansion of the kernel as
	\begin{align*}
		& k_0 \left(    \frac{1}{  \epsilon \brho_r(x)^2    } \left(  \|u\|^2 -  \|u\|^2 \nabla \tilde\beta_x(0) \cdot u + Q_{x,4}^{(\rho)}(u) + O^{[p]}(\|u\|^5) \right) \right)  \\
		&= k_0 \left(  \frac{\|u\|^2}{ \epsilon \brho_r(x)^2 } \right)	+  \left( -  \|u\|^2 \nabla \tilde\beta_x(0) \cdot u + Q_{x,4}^{(\rho)}(u) + O^{[p]}(\|u\|^5)  \right) \frac{1}{\epsilon \brho_r(x)^2 } k_0' \left(  \frac{\|u\|^2}{ \epsilon \brho_r(x)^2 } \right)	\\
		&\quad + {\frac{1}{2}} \left( \|u\|^4  (\nabla \tilde\beta_x(0) \cdot u)^2 + O^{[p]}(\|u\|^7) \right) \frac{1}{\epsilon^2 \brho_r(x)^4 } k_0'' \left(  \frac{\|u\|^2}{ \epsilon \brho_r(x)^2 } \right) + O^{[p]}\left( \frac{\|u\|^9}{\epsilon^3} |k_0^{(3)} (\xi)| \right),
	\end{align*}
	where 
	\[ \xi \ge \min \left\{  \frac{\| x - y\|^2}{  \epsilon  \phi( \brho_r(x), \brho_r(y) )^2  } , \; \frac{\|u\|^2}{ \epsilon \brho_r(x)^2 } \right\}  \geq  \frac{\|u\|^2}{ \epsilon c_{\max}^2 \rho_{\max}^2}.  \]
	Here, the second inequality used Assumption \ref{assump:phi-diff}(iv) and the upper bound of $\brho_r$ in Lemma \ref{lemma:bar-rho-epsilon}(ii).
	Besides, since $|k^{(3)}|$ has exponential decay by Assumption \ref{assump:k0-smooth}(ii) and $\exp(-a[k_0]\eta)$ is monotonic, we have
	\[ \frac{\|u\|^9}{\epsilon^3} |k_0^{(3)} (\xi) | \leq  a_3[k_0]\frac{\|u\|^9}{\epsilon^3} \exp \left( -  \frac{ a[k_0] \|u\|^2}{ \epsilon c_{\max}^2 \rho_{\max}^2} \right).  \]
	Therefore, 
	$\circled{1} = \circled{2}$, where 
	\begin{align} \label{eq:lemma-bias-fast-step2}
		\circled{2} &\coloneqq  \brho_r(x)^{-2}  \epsilon^{-d/2} \int_{\varphi_x\left( B^m_{\delta_\epsilon}(x) \cap \sM \right) } 
		\bigg(  k_0 \left(  \frac{\|u\|^2}{ \epsilon \brho_r(x)^2 } \right)	 +  \left( -  \|u\|^2 \nabla \tilde\beta_x(0) \cdot u + Q_{x,4}^{(\rho)}(u)  \right) \frac{1}{\epsilon \brho_r(x)^2 }  \notag	\\
		&\quad  k_0' \left(  \frac{\|u\|^2}{ \epsilon \brho_r(x)^2 } \right) +  {\frac{1}{2}} \frac{\|u\|^4  (\nabla \tilde\beta_x(0) \cdot u)^2 }{\epsilon^2 \brho_r(x)^4 } k_0'' \left(  \frac{\|u\|^2}{ \epsilon \brho_r(x)^2 } \right) + O^{[p]}\left( R_{\epsilon, x}(u) \right)  \bigg) \bigg(\nabla \tilde{f}(0) \cdot u   \notag \\
		&\quad + \frac{1}{2} u^T \Hess_{\tilde{f}}(0) u  +  Q_{x,3}^{(f)}(u) + O^{[f]}(\|u\|^4) \bigg)   \left(\tilde{p}(0)  +  \nabla \tilde{p}(0) \cdot u +  Q_{x,2}^{(p)}(u) + O^{[p]}(\|u\|^3)\right)\notag \\
		&\quad  \left( 1 + Q_{x,2}^{(V)}(u) + O(\|u\|^3)\right)du.
	\end{align}	
	Here, 
	\[ R_{\epsilon, x}(u) \coloneqq \frac{\|u\|^5}{\epsilon} \left|k_0' \left(  \frac{\|u\|^2}{ \epsilon \brho_r(x)^2 } \right) \right| +  \frac{\|u\|^7}{\epsilon^2} \left|k_0''\left(  \frac{\|u\|^2}{ \epsilon \brho_r(x)^2 } \right) \right| +  \frac{\|u\|^9}{\epsilon^3} \exp \left( -  \frac{ a[k_0]  \|u\|^2}{ \epsilon c_{\max}^2 \rho_{\max}^2} \right).  \]

\vspace{5pt}
\noindent
 \underline{Step 3.}
		Bound of residual terms and change of the integration domain to $\R^d$.  
	
	In this step, we first prove that the integrals that involve the residual terms, namely, the terms in the big-$O$ notations in the expansions of $\tilde{f}, \tilde{p}$, the kernel, and the volume comparison, contribute to an error of $O^{[f,p]}(\epsilon^2)$.
	Moreover, the error of changing the integration domain from $\varphi_x\left( B^m_{\delta_\epsilon}(x) \cap \sM \right)$ to $\R^d$ is of $o^{[f,p]}(\epsilon^2)$.
	Specifically, we will show that
	\begin{align} \label{eq:lemma-bias-fast-step3-1}
		\circled{2} = \circled{3} + O^{[f,p]}(\epsilon^2), 
	\end{align}
	where
	\begin{align*} 
		\circled{3} &\coloneqq  \brho_r(x)^{-2}  \epsilon^{-d/2} \int_{ \R^d } 
		\Bigg( \left( k_0 \left(  \frac{\|u\|^2}{ \epsilon \brho_r(x)^2 } \right)	\right) +  \left( -  \|u\|^2 \nabla \tilde\beta_x(0) \cdot u + Q_{x,4}^{(\rho)}(u)  \right) \frac{1}{\epsilon \brho_r(x)^2 } k_0' \left(  \frac{\|u\|^2}{ \epsilon \brho_r(x)^2 } \right)	\notag \\
		&\quad +  {\frac{1}{2}} \frac{\|u\|^4  (\nabla \tilde\beta_x(0) \cdot u)^2 }{\epsilon^2 \brho_r(x)^4 } k_0'' \left(  \frac{\|u\|^2}{ \epsilon \brho_r(x)^2 } \right)  \Bigg) \left(\nabla \tilde{f}(0) \cdot u + \frac{1}{2} u^T \Hess_{\tilde{f}}(0) u +  Q_{x,3}^{(f)}(u)  \right) \notag \\
		&\quad \left(\tilde{p}(0)  +  \nabla \tilde{p}(0) \cdot u +  Q_{x,2}^{(p)}(u) \right) \left( 1 + Q_{x,2}^{(V)}(u)\right)du.
	\end{align*}

	To bound the integrals involving the residual terms, we use the exponential decay of $| k_0^{(l)}(\eta)| (l \le 2)$, which is by Assumption \ref{assump:k0-smooth}(ii). It implies that for $l\leq 2$ and $j \leq 19$, we have the following bounds for integrals of the form below,
	\begin{align}  \label{eq:bound-integrals-eps}
		& \epsilon^{-d/2} \int_{\R^d} \|u\|^j  \left| k_0^{(l)}  \left( \frac{\|u\|^2}{ \epsilon \brho_r(x)^2 }  \right) \right| du = \brho_r(x)^{d+j} \epsilon^{j/2} \int_{\R^d} \|v\|^j  \left| k_0^{(l)}  \left( \|v\|^2 \right) \right| dv	= O^{[p]}(\epsilon^{j/2}), \notag \\
		& \epsilon^{-d/2} \int_{\R^d} \|u\|^j  \exp \left( - \frac{a[k_0] \|u\|^2}{ \epsilon c_{\max}^2\rho_{\max}^2 }  \right) du	=  (c_{\max}\rho_{\max})^{d+j} \epsilon^{j/2} \int_{\R^d} \|v\|^j  
		\exp \left( - a[k_0] {\|v\|^2} \right)  dv	= O^{[p]}(\epsilon^{j/2}). 
	\end{align}
	By the bound in \eqref{eq:bound-integrals-eps} applied to each of the integrals containing residual terms in \circled{2} in \eqref{eq:lemma-bias-fast-step2}, e.g.,
	\begin{align*}
		& \epsilon^{-d/2} \int_{\varphi_x\left( B^m_{\delta_\epsilon}(x) \cap \sM \right) } O^{[p]}\left( \frac{\|u\|^5}{\epsilon} k_0'\left( \frac{\|u\|^2}{\epsilon {\brho_r(x)^2}} \right) \right)    \,( \nabla \tilde{f}(0) \cdot u ) \, \tilde{p}(0) du \\
		& = O^{[f,p]} \left(\epsilon^{-d/2-1} \int_{\R^d} \|u\|^6  \left| k_0'  \left( \frac{\|u\|^2}{ \epsilon \brho_r(x)^2 }  \right) \right| du \right) = O^{[f,p]}(\epsilon^2),
	\end{align*}
	we have that these terms are all of $O^{[f,p]}(\epsilon^2)$.
	Therefore,
	\begin{align*}
		\circled{2} &= \brho_r(x)^{-2}   \epsilon^{-d/2} \int_{\varphi_x\left( B^m_{\delta_\epsilon}(x) \cap \sM \right) } 
		\Bigg( \left( k_0 \left(  \frac{\|u\|^2}{ \epsilon \brho_r(x)^2 } \right)	\right) +  \left( -  \|u\|^2 \nabla \tilde\beta_x(0) \cdot u + Q_{x,4}^{(\rho)}(u)  \right) \frac{1}{\epsilon \brho_r(x)^2 } \\
		&\quad  k_0' \left(  \frac{\|u\|^2}{ \epsilon \brho_r(x)^2 } \right) +  {\frac{1}{2}} \frac{\|u\|^4  (\nabla \tilde\beta_x(0) \cdot u)^2 }{\epsilon^2 \brho_r(x)^4 } k_0'' \left(  \frac{\|u\|^2}{ \epsilon \brho_r(x)^2 } \right)  \Bigg) \left(\nabla \tilde{f}(0) \cdot u + \frac{1}{2} u^T \Hess_{\tilde{f}}(0) u +  Q_{x,3}^{(f)}(u)  \right) \\
		&\quad  \left(\tilde{p}(0)  +  \nabla \tilde{p}(0) \cdot u +  Q_{x,2}^{(p)}(u) \right) \left( 1 + Q_{x,2}^{(V)}(u)\right)du + O^{[f,p]}(\epsilon^2).
	\end{align*}
	
	Next, we bound the error of changing the domain of integration from $\varphi_x\left( B^m_{\delta_\epsilon}(x) \cap \sM \right)$ to $\R^d$.
	Since $\delta_\epsilon < \delta_1(\sM)$, by the metric comparison in \eqref{eq:metric-comp2} in Lemma \ref{lemma:M-metric}(iii), we have
	\begin{equation*} 
		\varphi_x \left( B^m_{\delta_\epsilon}(x) \cap \sM \right) \supset B^d_{ \delta_\epsilon / 1.1 }(0).
	\end{equation*}
	As a result, it suffices to bound integrals of the form
	\[  \epsilon^{-d/2} \int_{ \R^d \setminus  B^d_{ \delta_\epsilon / 1.1 }(0)} \| u \|^j \left| k_0^{(l)} \left( \frac{\|u\|^2}{ \epsilon \brho_r(x)^2 } \right) \right| du, \quad j \leq 13, \quad l\leq 2.
 \] 
	By Assumption \ref{assump:k0-smooth}(ii) and the choice of $\delta_\epsilon$, for $u$ such that $\|u\| \geq \delta_\epsilon / 1.1$, we have
	\begin{align*}
		\left| k_0^{(l)} \left( \frac{\|u\|^2}{ \epsilon \brho_r(x)^2 } \right) 
		\right|  & \leq a_l[k_0] \exp \left( - \frac{a[k_0] \|u\|^2}{ \epsilon \rho_{\max}^2 }  \right) \leq  a_l[k_0] \exp \left( - \frac{a[k_0] \delta_\epsilon ^2}{ 2 \times 1.1^2 \epsilon \rho_{\max}^2 }  \right) 
		 \exp \left( - \frac{a[k_0] \|u\|^2}{ 2 \epsilon \rho_{\max}^2 }  \right) \\
		&\leq a_l[k_0] \epsilon^5 \exp \left( - \frac{a[k_0] \|u\|^2}{ 2 \epsilon \rho_{\max}^2 }  \right), \quad l = 0,1,2.
	\end{align*}
	Therefore, for $j \leq 13, l\leq 2$,
	\begin{align} \label{eq:bound-integrals-eps-change-Rd}
		&\epsilon^{-d/2} \int_{ \R^d \setminus  B^d_{ \delta_\epsilon / 1.1 }(0)} \| u \|^j \left| k_0^{(l)} \left( \frac{\|u\|^2}{ \epsilon \brho_r(x)^2 } \right) \right| du \leq a_l[k_0]   \epsilon^{5-d/2} \int_{ \R^d } \| u \|^j \exp 
		\left( - \frac{a[k_0] \|u\|^2}{ 2 \epsilon \rho_{\max}^2 }  \right) du \notag \\
		&= a_l[k_0] (\sqrt{2}\rho_{\max})^{d+j} \epsilon^{5+j/2} \int_{ \R^d } \| v \|^j \exp \left( - a[k_0] \|v\|^2  \right) dv = O^{[p]}(\epsilon^{5 + j/2}). 
	\end{align}
	Consequently, the error introduced by changing the integration domain from $\varphi_x\left( B^m_{\delta_\epsilon}(x) \cap \sM \right)$ to $\R^d$ can be bounded by $O^{[f,p]}(\epsilon^5)$, which is $o^{[f,p]}(\epsilon^2)$. Thus, \eqref{eq:lemma-bias-fast-step3-1} is proved.

\vspace{5pt}
\noindent
 \underline{Step 4.}	Collection of leading terms.

	Subsequently, we collect the leading terms in \circled{3} and will show that
	\begin{align} \label{eq:lemma-bias-fast-step4}
		\circled{3} = \epsilon \frac{m_2[k_0]}{2}  \sL_pf(x) + O^{[f,p]}(\epsilon^2 + \epsilon r),
	\end{align}
	where $\sL_p$ is as defined in \eqref{eq:def-sL}.
	
	Due to the symmetry of the kernels $k_0(\|u\|^2 / (\epsilon \brho_{r}(x)^2))$ and $k_0'(\|u\|^2 / (\epsilon \brho_{r}(x)^2))$, the odd moments in \circled{3} vanish.
	In addition to the terms of $O^{[f,p]}(\epsilon^2)$, which are bounded using \eqref{eq:bound-integrals-eps}, the remaining leading terms are
	\begin{align} \label{eq:lemma-bias-fast-step4-8}
		\circled{3} = \circled{4}_1 + \circled{4}_2 + \circled{4}_3 + O^{[f,p]}(\epsilon^2), 
	\end{align}
	where
	\begin{align} 
		\circled{4}_1  &\coloneqq \tilde{p}(0) \; \brho_r(x)^{-2}  \epsilon^{-d/2} \int_{ \R^d }   k_0 \left(  \frac{\|u\|^2}{ \epsilon \brho_r(x)^2 } \right)   \frac{1}{2} u^T \Hess_{\tilde{f}}(0) u  du, \label{eq:lemma-bias-fast-step4-2} \\
		\circled{4}_2  &\coloneqq  \brho_r(x)^{-2}  \epsilon^{-d/2} \int_{ \R^d }   k_0 \left(  \frac{\|u\|^2}{ \epsilon \brho_r(x)^2 } \right)   (\nabla \tilde{f}(0) \cdot u )  (\nabla \tilde{p}(0) \cdot u) du,  \label{eq:lemma-bias-fast-step4-3} \\ 
		\circled{4}_3  &\coloneqq  -\tilde{p}(0) \; \brho_r(x)^{-2}  \epsilon^{-d/2} \int_{ \R^d }   k_0' \left(  \frac{\|u\|^2}{ \epsilon \brho_r(x)^2 } \right) \frac{\|u\|^2}{\epsilon \brho_r(x)^2}   (\nabla \tilde{f}(0) \cdot u )  (\nabla \tilde{\beta}_x(0) \cdot u) du. \label{eq:lemma-bias-fast-step4-4} 
	\end{align}
	We compute the three terms $\circled{4}_1$, $\circled{4}_2$, and $\circled{4}_3$ separately as below.
	First, by changing the variable from $u$ to $v \coloneqq u / (\sqrt{\epsilon} \brho_{r}(x))$ and noting that $\tilde{p}(0) = p(x)$, we have
	\begin{align*}
		\circled{4}_1  &= p(x) \; \frac{1}{2} \epsilon \brho_{r}(x)^d \int_{ \R^d } k_0(\|v\|^2) v^T \Hess_{\tilde{f}}(0) v  dv = p(x) \; \frac{1}{2} \epsilon \brho_{r}(x)^d \sum_{i,j=1}^d \frac{\partial^2 \tilde{f}}{\partial s_i \partial s_j}(0) \int_{ \R^d } k_0(\|v\|^2)  v_iv_j    dv \\
		&=  p(x) \; \frac{1}{2} \epsilon \brho_{r}(x)^d \sum_{i=1}^d \frac{\partial^2 \tilde{f}}{\partial s_i^2}(0)  \; \frac{1}{d}\int_{ \R^d } k_0(\|v\|^2) \|v\|^2 dv =  \epsilon \frac{m_2[k_0]}{2} p(x)  \brho_{r}(x)^d \Delta f(x).
	\end{align*}
	In the last equality, we used the definition of $m_2[k_0]$ and that $\sum_{i=1}^d \frac{\partial^2 \tilde{f}}{\partial s_i^2}(0) = \Delta f(x)$.
	Recall that $\brho = p^{-1/d}$.
	By Lemma \ref{lemma:bar-rho-epsilon}(iii), $\|\brho_{r}^d - p^{-1} \|_\infty = \|\brho_{r}^d - \brho^d \|_\infty \leq d \rho_{\max}^{d-1} C_0(p) r$, which implies that $\|p \brho_{r}^d - 1 \|_\infty \leq d p_{\max} \rho_{\max}^{d-1} C_0(p) r$. Therefore,
	\begin{align} \label{eq:lemma-bias-fast-step4-5}
		\circled{4}_1  = \epsilon \frac{m_2[k_0]}{2}  \Delta f(x) + O^{[f,p]} (\epsilon r). 
	\end{align}
	
	Similarly, we obtain that
	\begin{align} \label{eq:lemma-bias-fast-step4-6}
		\circled{4}_2  &= \epsilon \brho_{r}(x)^d \int_{ \R^d } k_0(\|v\|^2) (\nabla \tilde{f}(0) \cdot v )  (\nabla \tilde{p}(0) \cdot v) dv = \epsilon \brho_{r}(x)^d  \sum_{i,j=1}^d \frac{\partial \tilde{f}}{\partial s_i}(0)  \frac{\partial \tilde{p}}{\partial s_j}(0) \int_{ \R^d } k_0(\|v\|^2)v_iv_j    dv  \notag \\
		&=\epsilon \brho_{r}(x)^d \sum_{i=1}^d \frac{\partial \tilde{f}}{\partial s_i}(0)  \frac{\partial \tilde{p}}{\partial s_{{i}}}(0)   \; \frac{1}{d}\int_{ \R^d } k_0(\|v\|^2) \|v\|^2 dv = \epsilon m_2[k_0]   \brho_{r}(x)^d  \nabla f(x) \cdot \nabla p(x)  \notag \\
		&= \epsilon m_2[k_0]  \nabla f(x) \cdot \frac{\nabla p(x) }{p(x)} + O^{[f,p]} (\epsilon r).
	\end{align}
	In the second last equality, we used that $\sum_{i=1}^d \frac{\partial \tilde{f}}{\partial s_i}(0)  \frac{\partial \tilde{p}}{\partial s_i}(0) = \nabla f(x) \cdot \nabla p(x)$. In the last equality, we used that $\|\brho_{r}^d - p^{-1} \|_\infty \leq d \rho_{\max}^{d-1} C_0(p) r$, which has been derived in the analysis of $\circled{4}_1$.
	
	Then, for $\circled{4}_3$, we have
	\begin{align*}
		\circled{4}_3 &= - p(x) \; \epsilon \brho_{r}(x)^d \int_{ \R^d } k_0'(\|v\|^2) \|v\|^2 (\nabla \tilde{f}(0) \cdot v )  (\nabla \tilde{\beta}_x(0) \cdot v) dv \\
		&= - p(x) \; \epsilon \brho_{r}(x)^d \sum_{i,j=1}^d \frac{\partial \tilde{f}}{\partial s_i}(0)  \frac{\partial \tilde{\beta}_x}{\partial s_j}(0) \int_{ \R^d } k_0'(\|v\|^2) \|v\|^2 v_iv_j    dv \\
		&= - p(x)  \; \epsilon \brho_{r}(x)^d \sum_{i=1}^d \frac{\partial \tilde{f}}{\partial s_i}(0)  \frac{\partial \tilde{\beta}_x}{\partial s_i}(0) \; \frac{1}{d} \int_{ \R^d } k_0'(\|v\|^2) \|v\|^4 dv. 
	\end{align*}
	We can directly compute the integral $\frac{1}{d} \int_{ \R^d } {k_0'(\|v\|^2)} \|v\|^4 dv$:
	\begin{align} \label{eq:moment-comp}
		&  \frac{1}{d}   \int_{ \R^d }  k_0' \left(  \|v\|^2 \right) \|v\|^4 dv =  \frac{1}{d}  \int_{S^{d-1}} \int_{0}^\infty  t^{d+3} k_0'( t^2 ) dt d\theta  = \frac{1}{2d }\int_{S^{d-1}} \int_{0}^\infty  s^{d/2+1} k_0'( s ) ds d\theta  \notag \\
		& = -  \frac{d+2}{4d} \int_{S^{d-1}} \int_{0}^\infty  s^{d/2} k_0( s ) ds d\theta = - \frac{d+2}{2d} \int_{S^{d-1}} \int_{0}^\infty  t^{d+1} k_0( t^2 ) dt d\theta  \notag \\
		& = - \frac{d+2}{2d} \int_{ \R^d }  k_0 \left(  \|v\|^2 \right) \|v\|^2 dv  = - \frac{d+2}{2} m_2[k_0].
	\end{align}
	Then, combined with the fact that $\sum_{i=1}^d \frac{\partial \tilde{f}}{\partial s_i}(0)  \frac{\partial \tilde{\beta}_x}{\partial s_i}(0) = \nabla f(x) \cdot \nabla \beta_x(x) $, we have
	\begin{align*}
		\circled{4}_3 
		&= \frac{d+2}{2}\epsilon m_2[k_0] p(x)\brho_{r}(x)^d  \nabla f(x) \cdot \nabla \beta_x(x).
	\end{align*}
	By the definition of $\beta_x$ in \eqref{eq:def-beta}, $\nabla \beta_x(x) = \nabla \brho_{r}(x) /  \brho_{r}(x)$.
	Furthermore, by Lemma \ref{lemma:bar-rho-epsilon}(iv),
	$\|\nabla \brho_{r} - \nabla \brho\|_\infty \leq C_1(p)r$.
	Therefore, by the triangle inequality, 
	$\|  \nabla f \cdot \nabla \brho_{r} /  \brho_{r}  - \nabla f \cdot \nabla \brho /  \brho \|_\infty \leq \|\nabla f\|_\infty ( C_1(p) / \rho_{\min} + C_0(p)  \| \nabla \brho\|_\infty /  \rho_{\min} ^2) r $. Thus, together with the fact that $\|p \brho_{r}^d - 1 \|_\infty \leq d p_{\max} \rho_{\max}^{d-1} C_0(p) r$ established in the analysis of $\circled{4}_1$ above, we have
	\begin{align} \label{eq:lemma-bias-fast-step4-7}
		\circled{4}_3 &= \frac{d+2}{2}\epsilon m_2[k_0]  \nabla f(x) \cdot \frac{\nabla  \brho(x)}{\brho(x)} + O^{[f,p]}(\epsilon r) \notag \\
		&= - \frac{d+2}{2d}\epsilon m_2[k_0]  \nabla f(x) \cdot \frac{\nabla  p(x)}{p(x)} + O^{[f,p]}(\epsilon r),
	\end{align}
	where in the second equality, we used that $\brho = p^{-1/d}$.
 
	Therefore, by inserting \eqref{eq:lemma-bias-fast-step4-5}\eqref{eq:lemma-bias-fast-step4-6}\eqref{eq:lemma-bias-fast-step4-7} back to \eqref{eq:lemma-bias-fast-step4-8}, we have
	\begin{align*}
		\circled{3} &= \circled{4}_1 + \circled{4}_2 + \circled{4}_3 + O^{[f,p]}(\epsilon^2) \\
		&=   \epsilon \frac{m_2[k_0]}{2}   \left( \Delta f(x) + 2\nabla f(x) \cdot \frac{\nabla p(x) }{p(x)}  -   \frac{d+2}{d}
		\nabla f(x) \cdot \frac{\nabla  p(x)}{p(x)}  \right)+ O^{[f,p]}(\epsilon^2 + \epsilon r) \\
		&= \epsilon \frac{m_2[k_0]}{2}  \sL_pf(x) + O^{[f,p]}(\epsilon^2 + \epsilon r),
	\end{align*}
	where the last equality uses the definition of $\sL_p$ in \eqref{eq:def-sL}. Then, \eqref{eq:lemma-bias-fast-step4} is proved.
	
\vspace{5pt}
\noindent
 \underline{Step 5.} Final error bound.   
        Throughout the proof, the only needed threshold for $\epsilon$ is $\epsilon <  \epsilon_D(k_0, c_{\max})$ used in Step 1, which satisfies the declared dependence in the lemma. 
        Then, when $\epsilon < \epsilon_D(k_0, c_{\max})$,
        \eqref{eq:lemma-bias-fast-step1}\eqref{eq:lemma-bias-fast-step2-1}\eqref{eq:lemma-bias-fast-step3-1}\eqref{eq:lemma-bias-fast-step4} from Steps 1-4 all hold, and this proves the lemma. 
        The constant in the big-$O$ notation of the final error term depends on $f$ and $p$ and is uniform for $x$ and $r$, since this applies to the residual terms in \eqref{eq:lemma-bias-fast-step1}\eqref{eq:lemma-bias-fast-step3-1}\eqref{eq:lemma-bias-fast-step4}. 
\end{proof}

\begin{proof}[Proof of Proposition \ref{prop:step2-diff}]
Since $k=o(N)$, when $N \geq N_{\rho,1}$, we have $r_k \leq \tilde{r}_0$.
Therefore, $\brho_{r_k}$ is well-defined.
Thus, $\bar{L}_\un$ as in \eqref{eq:def-Wbar-GL} is also well-defined.

First note that if $\| \nabla f \|_\infty = 0$, then $f$ is a constant function, and thus $\bar{L}_\un f \equiv 0$ and $\sL_pf \equiv 0$. Therefore, \eqref{eq:error-barL-LM-un} holds trivially.
We thus only  prove the case when $\| \nabla f \|_\infty > 0$.
We claim that  when $N$ is sufficiently large whose threshold depends on $(\sM, p, k_0, \phi)$, for any $r \in [0, \tilde{r}_0]$ and any $x \in \sM$, w.p. $\ge 1-2N^{-10}$,
\begin{align} \label{eq:step2-discrete-sum-diff-W}
	&\frac{1}{N}\sum_{j=1}^N \epsilon^{-d/2}  \brho_r(x)^{-2}  k_0 \left(  \frac{\| x - x_j\|^2}{ \epsilon \phi( \brho_r(x), \brho_r(x_j))^2 } \right) (f(x_j)-f(x))  \notag \\
	&=  \epsilon \frac{m_2}{2}  \sL_p  f(x) + O^{[f,p]}\left(\epsilon^2  +   \epsilon r \right)  +  O\left(  \|\nabla f\|_\infty p(x)^{1/d}  \sqrt{\frac{\log N}{N\epsilon^{d/2-1}}}\right).  
\end{align}
 Here, the threshold for $N$ and the constants in the big-$O$ notations are uniform for $x$ and $r$. 
We will apply the claim at $r=r_k$ and the fixed point $x\in\sM$ as in the proposition, and we denote the good event as $E_2$.

Observe that $\bar{L}_\un$ equals the l.h.s. of \eqref{eq:step2-discrete-sum-diff-W} divided by $\epsilon {m_2}/{2}$, where the $r$ in \eqref{eq:step2-discrete-sum-diff-W} takes the value of $r_k$.
As a result,  we reduce the proof of the proposition to that of the claim \eqref{eq:step2-discrete-sum-diff-W},
which holds under the good event in the claim (applied to $r=r_k$ and $x$) and when $N$ exceeds the needed threshold (and $N_{\rho,1}$).

Below, we prove the claim.
Recall the definition of $\delta_\epsilon[\cdot]$ in \eqref{eq:def-delta-eps}, and we denote $B^m_ {\delta_\epsilon[k_0]}$ as $B_ {\delta_\epsilon[k_0]}$.
Then, by a similar argument to \eqref{eq:k1-trunc}, we have that for any $x, y \in \sM$ and any $r \in [0, \tilde{r}_0]$,
 \begin{align} \label{eq:k0-trunc}
	\epsilon^{-d/2}k_0\left(  \frac{\| x - y\|^2}{ \epsilon \phi( \brho_r(x), \brho_r(y))^2 } \right) = \epsilon^{-d/2}k_0\left(  \frac{\| x - y\|^2}{ \epsilon \phi( \brho_r(x), \brho_r(y))^2 } \right) \mathbf{1}_{ \{y \in B_ {\delta_\epsilon[k_0]}(x) \} }    + O( \epsilon^{10} ),
\end{align}
where the constant in $O( \epsilon^{10} )$ depends on $\sM$ and is uniform for $x$, $y$, and $r$.
This implies that
\begin{align}  \label{eq:k0-trunc-1}
	& \frac{1}{N}\sum_{j=1}^N \epsilon^{-d/2}  \brho_r(x)^{-2}  k_0 \left(  \frac{\| x - x_j\|^2}{ \epsilon \phi( \brho_r(x), \brho_r(x_j))^2 } \right) (f(x_j)-f(x))  \notag \\
	&=  \frac{1}{N}\sum_{j=1}^N \epsilon^{-d/2}  \brho_r(x)^{-2}  k_0 \left(  \frac{\| x - x_j\|^2}{ \epsilon \phi( \brho_r(x), \brho_r(x_j))^2 } \right)  \mathbf{1}_{  \{x_j \in B_ {\delta_\epsilon[k_0]}(x) \} }  (f(x_j)-f(x))    + O^{[f,p]}( \epsilon^{10} ).  
\end{align}
We analyze the leading term in \eqref{eq:k0-trunc-1}, which is an independent sum.
We define
\begin{align*}
	Y_j \coloneqq \brho_r(x)^{-2} \epsilon^{-d/2}k_0\left(  \frac{\| x - x_j\|^2}{ \epsilon \phi( \brho_r(x), \brho_r(x_j))^2 } \right)  \mathbf{1}_{  \{x_j \in B_ {\delta_\epsilon[k_0]}(x) \} }(f(x_j)-f(x)) .
\end{align*}
Note that $\{Y_j\}_{j=1}^N$ are i.i.d random variables. 
For $\E Y_j$, by the truncation in \eqref{eq:k0-trunc}, we have
\begin{align*}
	\E Y_j &= \brho_r(x)^{-2}\int_\sM \epsilon^{-d/2}k_0\left(  \frac{\| x - y\|^2}{ \epsilon \phi( \brho_r(x), \brho_r(y))^2 } \right) \mathbf{1}_{  \{y \in B_ {\delta_\epsilon[k_0]}(x) \} }  (f(y)-f(x))p(y)dV(y) \\
	&= \brho_r(x)^{-2}\int_\sM \epsilon^{-d/2}k_0\left(  \frac{\| x - y\|^2}{ \epsilon \phi( \brho_r(x), \brho_r(y))^2 } \right)  (f(y)-f(x))p(y)dV(y)  + O^{[f,p]}( \epsilon^{10} ).  
\end{align*}
Let $ \epsilon_1 \coloneqq  \epsilon_D(k_0, c_{\max}) > 0$, where $\epsilon_D$ is defined in \eqref{eq:def-epsD}.
Then, $\epsilon_1$ equals the small-$\epsilon$ threshold required in Lemma \ref{lemma:step2-diff}.
Therefore, for $\epsilon < \epsilon_1$, by Lemma \ref{lemma:step2-diff}, we have that 
\begin{align} \label{eq:bound-bias-8}
	\E Y_j	= \epsilon \frac{m_2}{2}  
	\sL_p  f(x) 
	+ O^{[f,p]}\left(\epsilon^2 +  \epsilon r\right) + O^{[f,p]}( \epsilon^{10} ) = \epsilon \frac{m_2}{2}  
	\sL_p  f(x) 
	+ O^{[f,p]}\left(\epsilon^2 +  \epsilon r\right) .    
\end{align}

For $\var(Y_j)$, 
by the definition of $\epsilon_D$ in \eqref{eq:def-epsD}, we have that
when $\epsilon < \epsilon_1$, $\delta_\epsilon[k_0] < \delta_1(\sM)$ holds, where $\delta_1(\sM)$ is introduced in Lemma \ref{lemma:M-metric}(iii). 
Therefore, 
by Lemma \ref{lemma:f-diff-bound} (applied to $f$) and the metric comparison in \eqref{eq:metric-comp2},
\begin{align*}
	&\epsilon^{d/2} \var(Y_j) \leq \epsilon^{d/2} \E Y_j^2 \\
	&= \brho_r(x)^{-4}\int_\sM \epsilon^{-d/2}k_0^2\left(  \frac{\| x - y\|^2}{ \epsilon \phi( \brho_r(x), \brho_r(y))^2 } \right) \mathbf{1}_{  \{y \in B_ {\delta_\epsilon[k_0]}(x) \} }  (f(y)-f(x))^2p(y)dV(y) \\
	&\leq 1.1^2 \| \nabla f\|_\infty^2  \brho_r(x)^{-4} \int_\sM \epsilon^{-d/2}k_0^2\left(  \frac{\| x - y\|^2}{ \epsilon \phi( \brho_r(x), \brho_r(y))^2 } \right)  \|x - y\|^2p(y)dV(y) \\
	&= 1.1^2 \| \nabla f\|_\infty^2  \brho_r(x)^{-4}
		 \circled{1}, 
\end{align*}
where
\begin{align*}
	\circled{1} \coloneqq \int_\sM \epsilon^{-d/2}k_0^2\left(  \frac{\| x - y\|^2}{ \epsilon \phi( \brho_r(x), \brho_r(y))^2 } \right)  \|x - y\|^2p(y)dV(y).
\end{align*}
$k_0^2$ satisfies Assumption \ref{assump:k0-smooth} with $a_0[k_0^2] \coloneqq a_0[k_0]^2$ and the decay constant $a[k_0^2] \coloneqq 2 a[k_0]$.
Furthermore, $k_0^2(\eta) \eta \leq k_4(\eta)$, where 
\begin{align} \label{eq:def-k4}
	k_4(\eta) \coloneqq \frac{2}{e} \frac{ a_0[k_0^2]}{ a[k_0^2] }\exp(-\frac{a[k_0^2]}{2}\eta), \quad \eta\geq 0.
\end{align}
This is because from the facts that $ \exp(-\frac{a[k_0^2]}{2}\eta) \eta \le \frac{2}{e a[k_0^2]}$ and that $k_0^2$ satisfies Assumption \ref{assump:k0-smooth}, we have $k_0^2(\eta) \eta \leq  a_0[k_0^2] \exp(-a[k_0^2]\eta) \eta \le k_4(\eta)$.
$k_4$ is monotonic and satisfies Assumption \ref{assump:k0-smooth} with $a[k_4] \coloneqq a[k_0^2] / 2 = a[k_0] $. Besides, $m_0[k_4] > 0$ is a constant depending on $\sM$ and $k_0$.
By similar arguments used to obtain \eqref{eq:bound-var2}, 
we have that
\begin{align} \label{eq:bound-var-18}
		\circled{1} \leq \epsilon  c_{\max}^2  \int_\sM \epsilon^{-d/2}k_4\left( \frac{\| x - y\|^2}{ \epsilon  c_{\max}^2\max\{ \brho_r(x), \brho_r(y) \}^2 } \right) (\brho_r(x)^2 + \brho_r(y)^2 ) p(y) dV(y).
\end{align}
We now apply Lemma \ref{lemma:right-operator-degree}(ii) to compute the integral in \eqref{eq:bound-var-18}.
Let $ \epsilon_2 \coloneqq \epsilon_D(k_4, 1) / c_{\max}^2 > 0$, where $\epsilon_D$ is defined in \eqref{eq:def-epsD}.
Then, 
for $\epsilon < \epsilon_2$, we have $c_{\max}^2 \epsilon < \epsilon_D(k_4, 1)$, which is the small-$\epsilon$ threshold required in Lemma \ref{lemma:right-operator-degree}(ii) (with $\epsilon$ in the lemma being $\epsilon c_{\max}^2$ and $k_0$ in the lemma being $k_4$).
We apply \eqref{eq:ker-expan-degree-slowI} in Lemma \ref{lemma:right-operator-degree}(ii) 
	(specifically, we use $i=0$ to analyze the integral involving $\brho_r(x)^2$ and $i=2$ to analyze the integral involving $\brho_r(y)^2$ in
\eqref{eq:bound-var-18}), and have
\begin{align*} 
	\circled{1}   &\leq 2 \epsilon  c_{\max}^{d+2} \left( m_0[k_4] p \brho_r^{d+2}(x) + R_\epsilon(x,r) \right), \quad \sup_{x\in \sM,r \in [0,\tilde{r}_0]}|R_\epsilon(x,r)| 
	= 
	{O^{[p]}(\sqrt{\epsilon})}. 
\end{align*}
where the constant in 
${O^{[p]}(\sqrt{\epsilon})}$ depends on $\phi$ only through $c_{\max}$.
There is $\epsilon_3 > 0$ such that for $\epsilon < \epsilon_3$ (this threshold is uniform for $x$ and $r$), 
\[ R_\epsilon(x,r) \leq (2/3)^{(d+2)/d} m_0[k_4] p_{\max}^{-2/d} \leq  m_0[k_4] p \brho_r^{d+2}(x),\quad  \forall x \in \sM, \forall r \in [0, \tilde{r}_0]. \]
The last inequality used the lower bound of $\brho_r$, i.e., $\brho_r(x) \geq (2/(3p(x)))^{1/d}$ from Lemma \ref{lemma:bar-rho-epsilon}(ii). By the upper bound of $R_\epsilon(x,r)$ above and the fact that $\brho_r(x) \leq (2/p(x))^{1/d}$ from Lemma \ref{lemma:bar-rho-epsilon}(ii), $\circled{1} \le 4 c_{\max}^{d+2} m_0[k_4] p \brho_r^{{d+2}}(x)  \epsilon$. 
Therefore,
\begin{align}\label{eq:bound-var-7}
	\var(Y_j) &\leq 1.1^2  \| \nabla f\|_\infty^2\brho_r^{-4}(x)  \,\circled{1} \, \epsilon^{-d/2} \leq 1.1^2  \| \nabla f\|_\infty^2 \;  4 c_{\max}^{d+2} m_0[k_4] p \brho_r^{d-2}(x) \epsilon^{1-d/2} \notag \\
	&  \leq \mu_Y(x) \coloneqq C_{V,2}(\sM, k_0, \phi) \| \nabla f\|_\infty^2 p(x)^{2/d} \epsilon^{1-d/2},
\end{align}
where $C_{V,2}(\sM, k_0, \phi)  =  {4\max\{3/2,2^{(d-2)/d}\}} 1.1^2 c_{\max}^{d+2}m_0[k_4] $.
{Here we used $p(x)\brho_r(x)^{d-2}\leq \max\{3/2,2^{(d-2)/d}\}p(x)^{2/d}$; for $d=1$ this follows from the lower bound on $\brho_r$, and for $d\geq2$ from the upper bound.}

Moreover, since $\delta_\epsilon[k_0] < \delta_1(\sM)$, the metric comparison in \eqref{eq:metric-comp2} holds for $x,y$ such that $y \in B_ {\delta_\epsilon[k_0]}(x)$.  
Together with the {Lemma \ref{lemma:f-diff-bound} (applied to $f$)}
 and the lower bound of $\brho_r$ from Lemma \ref{lemma:bar-rho-epsilon}(ii), we have
\begin{align*}
	|Y_j|  &= \brho_r(x)^{-2}\epsilon^{-d/2}k_0\left(  \frac{\| x - x_j\|^2}{ \epsilon \phi( \brho_r(x), \brho_r(x_j))^2 } \right)  \mathbf{1}_{  \{x_j \in B_ {\delta_\epsilon[k_0]}(x) \} }  |f(x_j) - f(x)| \\
	&\le (2/3)^{-2/d}p(x)^{2/d}  1.1 \|\nabla f \|_\infty    \epsilon^{-d/2} k_0\left(  \frac{\| x - x_j\|^2}{ \epsilon \phi( \brho_r(x), \brho_r(x_j))^2 } \right) \|x-x_j\| \\
        &= (2/3)^{-2/d} p(x)^{2/d} 1.1 \|\nabla f \|_\infty    \epsilon^{1/2-d/2} k_0\left(  \frac{\| x - x_j\|^2}{ \epsilon \phi( \brho_r(x), \brho_r(x_j))^2 } \right) \frac{\|x-x_j\|}{\sqrt{\epsilon} \phi( \brho_r(x), \brho_r(x_j))} \phi( \brho_r(x), \brho_r(x_j)) \\
	&\le (2/3)^{-2/d} p(x)^{2/d} 1.1 \|\nabla f \|_\infty   \epsilon^{1/2-d/2}  \frac{a_0[k_0]}{ \sqrt{2e a[k_0] }} c_{\max} \rho_{\max}, 
	\end{align*}
where the last inequality is  due to the fact  that $k_0(\eta) \sqrt{\eta} \leq a_0[k_0] \exp(-a[k_0] \eta) \sqrt{\eta} \leq \frac{a_0[k_0]}{\sqrt{2e a[k_0]}}, \forall\eta \ge 0,$ and the upper bound of $\phi( \brho_r(x), \brho_r(x_j))$ in \eqref{eq:bound-phi-denominator}.
In other words, we have that 
\begin{align} \label{eq:bound-LY}
	|Y_j|   \leq L_Y(x)  \coloneqq C_{B,2}(\sM, p, k_0, \phi) \|\nabla f \|_\infty p(x)^{2/d} \epsilon^{1/2-d/2},
\end{align}
where  $C_{B,2}(\sM, p, k_0, \phi) = 1.1  (2/3)^{-2/d} c_{\max} \rho_{\max}  {a_0[k_0]}/{\sqrt{2e a[k_0]}}.$

Since we now consider the case where $ \|\nabla f \|_\infty > 0$, it holds that $L_Y(x)  > 0, \mu_Y(x) > 0, \forall x \in 
\sM$. 
By the condition $\epsilon^{d/2}= \Omega(\log N/N)$ in the theorem, 
when $N$ is sufficiently large whose threshold depends on $(\sM, p, k_0, \phi)$,
we have 
$160 \frac{\log N}{N} \leq 9 \frac{C_{V,2}}{(C_{B,2})^2} p_{\max}^{-2/d} \epsilon^{d/2} \leq 9 \frac{\mu_Y(x)}{ L_Y(x)^2 }$
for any $x \in \sM$. The inequality between the first and last expressions is equivalent, for $\tau \coloneqq \sqrt{40\mu_Y(x)\log N/N}$, to $2\tau L_Y(x) \leq 3\mu_Y(x)$. Since $Y_j$ is signed, $|Y_j-\E Y_j|\leq 2L_Y(x)$. Thus, Bernstein's inequality applies to $Y_j-\E Y_j$, and w.p. $\ge 1-2N^{-10}$,
\begin{align} \label{eq:concen3}
	 \frac{1}{N} \sum_{j = 1}^N Y_j &=  \E Y_j + O\left( \sqrt{ \frac{\log N}{N}  \mu_Y(x)}  \right)  \notag \\
	 &=   \epsilon \frac{m_2}{2}  \sL_p  f(x) + O^{[f,p]}\left(\epsilon^2  +   \epsilon r \right)  +  O\left(  \|\nabla f\|_\infty p(x)^{1/d}  \sqrt{\frac{\log N}{N\epsilon^{d/2-1}}}\right).
\end{align}

Combined with \eqref{eq:k0-trunc-1}, we have that
\begin{align*}
	&\frac{1}{N}\sum_{j=1}^N \epsilon^{-d/2}  \brho_r(x)^{-2}  k_0 \left(  \frac{\| x - x_j\|^2}{ \epsilon \phi( \brho_r(x), \brho_r(x_j))^2 } \right) (f(x_j)-f(x)) \\
	&=  \epsilon \frac{m_2}{2}  \sL_p  f(x) + O^{[f,p]}\left(\epsilon^2  +   \epsilon r \right)  +  O\left(  \|\nabla f\|_\infty p(x)^{1/d}  \sqrt{\frac{\log N}{N\epsilon^{d/2-1}}}\right) + O^{[f,p]}( \epsilon^{10} ) \\
	&=  \epsilon \frac{m_2}{2}  \sL_p  f(x) + O^{[f,p]}\left(\epsilon^2  +   \epsilon r \right)  +  O\left(  \|\nabla f\|_\infty p(x)^{1/d}  \sqrt{\frac{\log N}{N\epsilon^{d/2-1}}}\right).
\end{align*}

Since $\epsilon = o(1)$, when $N$ exceeds a certain threshold, we have
$ \epsilon < \min\{  \epsilon_1, \epsilon_2, \epsilon_3\}, $
which is a constant depending on $(\sM, p, k_0, \phi)$ and is uniform for $x$ and $r$, and then the bounds in \eqref{eq:bound-bias-8}\eqref{eq:bound-var-7}\eqref{eq:bound-LY} hold.
Another large-$N$ threshold is required by the Bernstein inequality used to prove \eqref{eq:concen3}, which is uniform for all $x$ and $r$ and depends on $(\sM, p, k_0, \phi)$.
The overall large-$N$ threshold required by \eqref{eq:step2-discrete-sum-diff-W} depends on $(\sM, p, k_0, \phi)$.
\end{proof}

\subsection{Step 3: final convergence bound of $L_\un$}

\begin{proof} [Proof of Theorem \ref{thm:conv-un-Laplacian-case1}]

We first prove \eqref{eq:fast-rate}, which will lead to \eqref{eq:fast-rate-N}.

\vspace{5pt}
\noindent
$\bullet$ Proof of \eqref{eq:fast-rate}: 
The proof is based on Propositions \ref{prop:step1-diff} and \ref{prop:step2-diff}. 
We collect the good events and large-$N$ thresholds from these two propositions: 
	
Let the large-$N$ threshold needed by Proposition \ref{prop:step1-diff} be denoted as 
$N_1$ (which depends on $(\sM, p, k_0, \phi)$ as shown in the proof of the proposition),
and the proposition holds w.p. $ \ge 1-3N^{-10}$, which we denote as the good event $E_1$;
	
Let the large-$N$ threshold needed by Proposition \ref{prop:step2-diff} be denoted as $N_2$ (which depends on $(\sM, p, k_0, \phi)$ as shown in the proof of the proposition),
and the proposition holds under the good event $E_2$ which has been defined in its proof and happens w.p. $ \ge 1-2N^{-10}$. 
	
Then, when $N \geq \max\{ N_1, N_2\}$ and under the good event $E_1 \cap E_2$ which happens w.p. $\ge 1-5N^{-10}$, by the triangle inequality, $\left| L_\un f(x)  -  \sL_p f(x)  \right|$  is bounded by the sum of \eqref{eq:error-L-barL-un} and \eqref{eq:error-barL-LM-un}, which gives 
\[
O\left(\|\nabla f \|_\infty p(x)^{1/d} \frac{\varepsilon_{\rho,k}}{\sqrt{\epsilon}}\right)
   + O( \|f\|_\infty p(x)^{2/d}  \varepsilon_{\rho,k} \epsilon^{9}) + O^{[f,p]}\left(\epsilon  + r_k\right)  +  O\left(  \|\nabla f\|_\infty p(x)^{1/d}  \sqrt{\frac{\log N}{N\epsilon^{d/2+1}}}\right).
\]
Observe that 
$O(\|f\|_\infty p(x)^{2/d} \varepsilon_{\rho,k} \epsilon^9) = O^{[f,p]}(\epsilon)$:
under the condition of Theorem \ref{thm:conv-un-Laplacian-case1},
Theorem \ref{thm:consist-hrho} applies, which gives that 
 $\varepsilon_{\rho,k}=O^{[p]}\left(\left(\frac{k}{N}\right)^{3/d}\right)+O\left(\sqrt{\frac{\log N}{k}}\right)$. 
 Since we also have  $k/N\rightarrow 0$ and $\log N/k\rightarrow 0$, it follows that
$\varepsilon_{\rho,k}=o(1)$,
and then $O(\|f\|_\infty p(x)^{2/d}\varepsilon_{\rho,k}\epsilon^9) =O^{[f,p]}(\epsilon^9) =O^{[f,p]}(\epsilon)$.
Therefore, the above estimate becomes  \eqref{eq:fast-rate}. 
Finally, the overall large-$N$ threshold depends on $(\sM, p, k_0, \phi)$. 
	
\vspace{5pt}
\noindent
$\bullet$ Proof of \eqref{eq:fast-rate-N}: Plugging the bound of $\varepsilon_{\rho,k}$ in \eqref{eq:vareps-rho-bound} and the definition of $r_k$ in \eqref{eq:def-rk} into \eqref{eq:fast-rate}, we have the following bound of $\left| L_\un f(x)  -  \sL_p f(x)  \right|$ in terms of $\epsilon, k, N$:
\begin{equation} \label{eq:fast-rate1}
	O^{[f,p]} \left(\epsilon  +   \frac{1}{\sqrt{\epsilon}}\left(\frac{k }{N }\right)^{3/d}   +   \left(\frac{k }{N } \right)^{2/d} \right) \,  + \, O\left( \|\nabla f\|_\infty p(x)^{1/d}   \left( \sqrt{\frac{\log N}{k\epsilon}} +   \;   \sqrt{\frac{\log N}{N\epsilon^{d/2+1}}}  \right)  \right).    
\end{equation}
The derivation of the rate \eqref{eq:fast-rate-N} from the bound \eqref{eq:fast-rate1} follows from Lemma \ref{lemma:fast-rate}.
\end{proof}

\section{Proofs in Section \ref{sec:theory-extend}}\label{app:proof-theory-extend}

\subsection{Proof of Theorem \ref{thm:conv-rw-Laplacian}}

\begin{proof}[Proof of Theorem \ref{thm:conv-rw-Laplacian}]
Recall that $L_\rw f(x) =  \frac{ L_\un f(x) }{  D(x) }$. 
We introduce another quantity $\bar{D}(x)$ by substituting $\hrho$ with $\brho_{r_k}$ in $D(x)$, namely 
\begin{align} \label{eq:def-Dbar}
	\bar{D}(x)  \coloneqq \frac{1}{m_0 N}\sum_{j=1}^N  \epsilon^{-d/2} k_0 \left(  \frac{\| x - x_j\|^2}{ \epsilon \phi( \brho_{r_k}(x), \brho_{r_k}(x_j))^2 }  \right), \quad x \in \sM.
\end{align}
To verify that $\bar D(x)$ is well-defined:
note that under the condition of the theorem, \eqref{eq:property-Nrho} holds, which implies that when $N \geq N_\rho$, $r_k \leq \tilde{r}_0$. Therefore, $\brho_{r_k}$ is well-defined by Lemma \ref{lemma:bar-rho-epsilon}(i), and consequently, $\bar{D}(x)$ is also well-defined.

Then, by the triangle inequality, 
	\begin{align} \label{eq:proof-GL-rw}
		\left| L_\rw f(x) -  \sL_pf(x) \right| & \leq 
		 \left| L_\un f(x)\right|  \left| \frac{1}{D(x)} - \frac{1}{\E \bar{D}(x)}\right| 
		+ \left| L_\un f(x) -  \sL_pf(x) \right| \frac{1}{\E \bar{D}(x)}  \notag \\
		&~~~
    + \left|\sL_pf(x) \right| \left| \frac{1}{\E \bar{D}(x)} - 1\right| 
		=: R_1(x) + R_2(x) + R_3(x).
	\end{align}
To bound $R_1(x)$, we need the upper bounds of $\left| L_\un f(x)\right|$ and $|D(x) - \E \bar{D}(x)|$, which we will establish below.
 	For $R_2(x)$, $\left| L_\un f(x) -  \sL_pf(x) \right|$ has been bounded
	 in Theorem \ref{thm:conv-un-Laplacian-case1};
	For $R_3(x)$, we will use the $O(1)$ upper bound of $|\sL_pf(x)|$ and also derive a bound  of $|\E \bar{D}(x) - 1|$.
 All the control of $D(x)$, $\bar D(x)$ will be using the fact that $D(x) \approx \bar D(x) \approx \E \bar D(x) \approx 1$,
 where $D(x) \approx \bar D(x)$ is by the replacement of $\hrho$ with $\brho_{r_k}$, and $\bar D(x) \approx \E \bar D(x)$ holds
 because $\bar D(x)$ is an independent sum,
 and lastly, $\E \bar D(x) \approx 1$ is by expansion of the kernel integral at small $\epsilon$.

\vspace{5pt}	
	\noindent \underline{Bound of $| D(x) - \bar{D}(x)|$.}  
	
	We introduce the following lemma which holds 	
	when $k_0$ is differentiable, to be proved after we finish the theorem proof. 
	\begin{lemma}\label{lemma:step1-diff-degree}
		Under the same conditions of Lemma \ref{lemma:step1-diff}, for any $\epsilon > 0$ and any $x \in \sM$, 
		\begin{align}   \label{eq:lemma-step1-diff-degree}
			\left| D(x) - \bar{D}(x)\right|  &\leq C \varepsilon_{\rho,k}\;\frac{1}{N}\sum_{j=1}^N \epsilon^{-d/2}\tilde{k}_1\left(  \frac{\| x - x_j\|^2}{ \epsilon \phi( \brho_{r_k}(x), \brho_{r_k}(x_j))^2 } \right),
		\end{align}
		where $C = 6L_\phi / m_0[k_0]$ is a constant depending on $k_0$ and $\phi$, and
		\begin{align} \label{eq:def-k1-tilde}
			\tilde{k}_1(\eta) \coloneqq a_1[k_0]  \eta \exp( - \frac{a[k_0]}{2} \eta), \quad \eta \geq 0.
		\end{align}
	\end{lemma}

	We claim that for $k_0$ satisfying Assumption \ref{assump:k0-smooth} and $\phi$ satisfying Assumption \ref{assump:phi-diff}(i)(iv),
	when $N$ is sufficiently large with the threshold depending on $(\sM, p, k_0, \phi)$, 
	for any $r \in [0, \tilde{r}_0]$ and any $x \in \sM$, w.p. $\ge 1-2N^{-10}$,
	\begin{align}\label{eq:bound-k1-sum-D}
		\frac{1}{N}\sum_{j=1}^N \epsilon^{-d/2} \tilde{k}_1\left(  \frac{\| x - x_j\|^2}{ \epsilon \phi( \brho_r(x), \brho_r(x_j))^2 } \right) =  O\left( 1 \right).  
	\end{align}
	We denote the good event as $E_3 $.
	Here, the threshold for $N$ and the constant in the big-$O$ notation are uniform for both $r$ and $x$. 
	The constant in $O(1)$ depends on $\phi$ only through $c_{\max}$.
	We will apply the claim at $r = r_k$ and the fixed point $x$ as in the theorem.

 Suppose $N$ is large enough, exceeding both $N_\rho$ and the large-$N$ threshold of the claim \eqref{eq:bound-k1-sum-D}, and we consider under the good event $E_\rho \cap E_3 $ that happens w.p. $\ge 1-3N^{-10}$.
This ensures that \eqref{eq:property-Nrho} holds, and then Lemma \ref{lemma:step1-diff-degree} applies,
and the claim \eqref{eq:bound-k1-sum-D}, assumed to be true, also applies at $r = r_k$ and $x$ as in the theorem. Together, \eqref{eq:lemma-step1-diff-degree} and \eqref{eq:bound-k1-sum-D} imply that 
	\begin{align} \label{eq:bound-D-Dbar}
		\left|D(x)-\bar{D}(x)\right|&= O\left( \varepsilon_{\rho,k} \right).
	\end{align}
	
	Below, we prove the claim \eqref{eq:bound-k1-sum-D}.
	We denote by
	\[ Y_j \coloneqq \epsilon^{-d/2} \tilde{k}_1\left(  \frac{\| x - x_j\|^2}{ \epsilon \phi( \brho_r(x), \brho_r(x_j))^2 } \right).\]
	Note that $\{Y_j\}_{j=1}^N$ are i.i.d. random variables.

	For $\E Y_j$, we first note that $\tilde{k}_1$  satisfies Assumption \ref{assump:k0-smooth} with $a_0[\tilde{k}_1] \coloneqq 4 a_1[k_0] / ( e a[k_0] )$ and the decay constant $a[\tilde{k}_1] \coloneqq a[k_0] / 4$. Then, by Assumption \ref{assump:k0-smooth}(ii),
	\begin{align} 
		\E Y_j  &= \int_{\sM}\epsilon^{-d/2} \tilde{k}_1\left(  \frac{\| x - y \|^2}{ \epsilon \phi( \brho_r(x), \brho_r(y))^2 } \right) p(y) dV(y) \notag  \\
		&\leq \int_{\sM}\epsilon^{-d/2} k_5 \left(  \frac{\| x - y \|^2}{ \epsilon c_{\max}^2 \max\{ \brho_r(x), \brho_r(y)\}^2 } \right) p(y) dV(y), \label{eq:bound-bias-degree1}
	\end{align}
	where
	\[ k_5(\eta) \coloneqq a_0[\tilde{k}_1]  \exp( - a[\tilde{k}_1] \eta ), \quad \eta \geq 0. \]
	Here,  in the second inequality, we used that $k_5$ is monotonic and that $\phi$ satisfies Assumption \ref{assump:phi-diff}(iv).
	$m_0[k_5] > 0$ is a constant depending on $\sM$ and $k_0$. 
 We now apply Lemma \ref{lemma:right-operator-degree}(ii) to compute the integral in \eqref{eq:bound-bias-degree1}.
 Recall the definition of $\epsilon_D$ in \eqref{eq:def-epsD}.
	Let $\epsilon_1 \coloneqq \epsilon_D(k_5, 1) / c_{\max}^2 > 0$,  then
	for $\epsilon < \epsilon_1$, we have 
	$c_{\max}^2 \epsilon < \epsilon_D(k_5, 1)$, which is the small-$\epsilon$ threshold required by Lemma \ref{lemma:right-operator-degree}(ii) (with $\epsilon$ in the lemma being $\epsilon c_{\max}^2$ and $k_0$ in the lemma being $k_5$).
	Thus, by \eqref{eq:ker-expan-degree-slowI} in Lemma \ref{lemma:right-operator-degree}(ii) (with $i=0$), we have
	\begin{align} \label{eq:bound-bias-5}
		\E Y_j \leq c_{\max}^d \left( m_0[k_5] p\brho_r^d(x) + R_{\epsilon,1}(x,r) \right), \quad \sup_{x\in \sM,r \in [0,\tilde{r}_0]}|R_{\epsilon,1}(x,r) | 
		= 
		{O^{[p]}(\sqrt{\epsilon})},
	\end{align}
	where the constant dependence in 
	$
	{O^{[p]}(\sqrt{\epsilon})}$ can be found in Lemma \ref{lemma:right-operator-degree}(ii) and this constant depends on $\phi$ only through $c_{\max}$.
	By a similar argument that obtains \eqref{eq:bound-bias-1}, 
	there is $\epsilon_2 > 0$ such that for $\epsilon < \epsilon_2$ (this threshold is uniform for $x$ and $r$),
	\begin{align} \label{eq:bound-bias-9}
		\E Y_j \leq C_{E,2} (\sM, k_0, \phi), 
	\end{align}
	where $C_{E,2} (\sM, k_0, \phi) = 4 c_{\max}^d  m_0[k_5]$. 
	
	For $\var(Y_j)$,
	note that $\tilde{k}_1^2$ satisfies Assumption \ref{assump:k0-smooth} with $a_0[\tilde{k}_1^2] \coloneqq a_0[\tilde{k}_1]^2$ and the decay constant $a[\tilde{k}_1^2] \coloneqq 2 a[\tilde{k}_1] = a[k_0] / 2 $.
	Then, by the same argument to bound $\E Y_j$ above, we have
	\begin{align} \label{eq:bound-var-19}
		\epsilon^{d/2} \var(Y_j) \leq \epsilon^{d/2} \E Y_j^2 \leq \int_{\sM}\epsilon^{-d/2} k_6 \left(  \frac{\| x - y \|^2}{ \epsilon c_{\max}^2 \max\{ \brho_r(x), \brho_r(y)\}^2 } \right) p(y) dV(y), 
	\end{align}
	where 
	\[ k_6(\eta) \coloneqq a_0[\tilde{k}_1^2]  \exp( - a[\tilde{k}_1^2] \eta ), \quad \eta \geq 0. \]
	$m_0[k_6] > 0$ is a constant depending on $\sM$ and $k_0$. 
	We apply Lemma \ref{lemma:right-operator-degree}(ii) to compute the integral in \eqref{eq:bound-var-19}.
	Let $\epsilon_3 \coloneqq \epsilon_D(k_6, 1) / c_{\max}^2 > 0$,  then
	for $\epsilon < \epsilon_3$, we have $c_{\max}^2 \epsilon < \epsilon_D(k_6, 1)$, which is the small-$\epsilon$ threshold required by Lemma \ref{lemma:right-operator-degree}(ii) (with $\epsilon$ in the lemma being $\epsilon c_{\max}^2$ and $k_0$ in the lemma being $k_6$).
	By \eqref{eq:ker-expan-degree-slowI} in Lemma \ref{lemma:right-operator-degree}(ii) (with $i=0$), we have
	\begin{align} \label{eq:bound-var-12}
		\epsilon^{d/2} \var(Y_j) \leq c_{\max}^d \left( m_0[k_6] p\brho_r^d(x) + R_{\epsilon,2}(x,r) \right), \quad \sup_{x\in \sM,r \in [0,\tilde{r}_0]} |R_{\epsilon,2}(x,r)| 
		=  
		{O^{[p]}(\sqrt{\epsilon})},
	\end{align}
	where the constant in 
	$
	{O^{[p]}(\sqrt{\epsilon})}$ depends on $\phi$ only through $c_{\max}$.
	By a similar argument as used to obtain \eqref{eq:bound-var-6}, there is $\epsilon_4>0$ such that for $\epsilon < \epsilon_4$ (this threshold is uniform for $x$ and $r$),
	\begin{align} \label{eq:bound-var-13}
		\var(Y_j) \leq \mu_Y \coloneqq C_{V,3} (\sM, k_0, \phi) \epsilon^{-d/2},  
	\end{align}
	where $C_{V,3} (\sM, k_0, \phi) = 4 c_{\max}^d  m_0[k_6]$. 
	
	Moreover, $|Y_j| \leq L_Y \coloneqq C_{B,3} (\sM, k_0) \epsilon^{-d/2}$, where $C_{B,3} (\sM, k_0) = a_0[\tilde{k}_1]$.
	By the condition $\epsilon^{d/2}= \Omega(\log N/N)$ in the theorem,
	when $N$ is sufficiently large whose threshold depends on $(\sM, k_0, \phi)$, we have $ 40 \frac{\log N}{N}  \leq 9 \frac{C_{V,3} }{(C_{B,3} )^2}  \epsilon^{d/2} $, and then by Bernstein inequality, w.p. $\ge 1-2N^{-10}$,
	\begin{align} \label{eq:concen8}
		0 \leq  \frac{1}{N} \sum_{j = 1}^N Y_j   \leq \E Y_j + \sqrt{40 \frac{\log N}{N}  \mu_Y} \leq C_{E,2}  + \sqrt{40 C_{V,3} }  \sqrt{ \frac{\log N}{N \epsilon^{d/2}}}.
	\end{align}
	Since $\epsilon^{d/2}= \Omega(\log N/N)$, 
	for sufficiently large $N$, we have 
	$ \frac{\log N}{N \epsilon^{d/2}} \leq 1$, implying that
	\begin{align} \label{eq:concen9}
		0 \leq \frac{1}{N} \sum_{j = 1}^N Y_j \leq C_{E,2}  + \sqrt{40 C_{V,3} }= O (1) .
	\end{align}
	The constant in the big-$O$ notation depends on $\sM, k_0, \phi$ through $C_{E,2} $ and $C_{V,3} $ and is uniform for $x$ and $r$.

	Since $\epsilon = o(1)$, when $N$ exceeds some threshold,
	$ \epsilon <  \min\{  \epsilon_1, \epsilon_2, \dots, \epsilon_4 \}$, 
	which is a constant depending on $(\sM, p, k_0, \phi)$ and is uniform for $x$ and $r$, and then \eqref{eq:bound-bias-9}\eqref{eq:bound-var-12} along with the bound for $Y_j$ hold.
	The other large-$N$ thresholds are required by \eqref{eq:concen8}\eqref{eq:concen9}, which at most depend on $(\sM, k_0, \phi)$ and are uniform for $x$ and $r$.
	The overall large-$N$ threshold required by the claim in \eqref{eq:bound-k1-sum-D} depends on ($\sM$, $p$, $k_0$, $\phi$).

	\vspace{5pt}
	\noindent
	\underline{Bounds of $| \bar{D}(x) - \E  \bar{D}(x) |$ and $|\E  \bar{D}(x) - 1 |$.}  
	We claim that 
	when $N$ is sufficiently large with the threshold depending on $(\sM, p, k_0, \phi)$, for any $r \in [0, \tilde{r}_0]$ and any $x \in \sM$, w.p. $\ge 1-2N^{-10}$,
	\begin{align}\label{eq:bound-k1-sum-D2}
		&\frac{1}{N}\sum_{j=1}^N \epsilon^{-d/2} k_0\left(  \frac{\| x - x_j\|^2}{ \epsilon \phi( \brho_r(x), \brho_r(x_j))^2 } \right) 
		= \int_\sM \epsilon^{-d/2} k_0\left(  \frac{\| x - y\|^2}{ \epsilon \phi( \brho_r(x), \brho_r(y))^2 } \right) p(y) dV(y)  + O\left( \sqrt{\frac{\log N}{N \epsilon^{d/2}}}\right), \notag \\
		&\int_\sM \epsilon^{-d/2} k_0\left(  \frac{\| x - y\|^2}{ \epsilon \phi( \brho_r(x), \brho_r(y))^2 } \right) p(y) dV(y)  
		=  m_0 p(x) \brho_r^d(x) + O^{[p]}(\epsilon) .
	\end{align}
	Here, the threshold for $N$ and the constants in the big-$O$ notations are uniform for $r$ and $x$. 
	We apply the claim at $r = r_k$ and $x\in \sM$ in the theorem, and denote the good event as $E_4 $.
	
	Observe that $\bar{D}(x)$ equals the l.h.s. of \eqref{eq:bound-k1-sum-D2} divided by $m_0$, where the $r$ in  \eqref{eq:bound-k1-sum-D2} takes the value of $r_k$.
	Besides, by Lemma \ref{lemma:bar-rho-epsilon}(iii), we have
	$ \| p\brho_{r_k}^d - 1 \|_\infty = O^{[p]}(r_k)$.  
	Putting together, we have
	\begin{align} \label{eq:bound-Dbar-1}
		\left| \bar{D}(x) - \E \bar{D}(x) \right|= O\left( \sqrt{\frac{\log N}{N \epsilon^{d/2}}}\right), \quad \left| \E \bar{D}(x) - 1\right| = O^{[p]}\left( \epsilon + r_k \right),
	\end{align}
	which holds when $N$ exceeds the required threshold (and $N_\rho$) and under the good event $E_4 $ that happens w.p. $\ge 1-2N^{-10}$.
		
	Below, we prove the claim.
	We denote by
	\[ Y_j \coloneqq \epsilon^{-d/2}k_0\left(  \frac{\| x - x_j\|^2}{ \epsilon \phi( \brho_r(x), \brho_r(x_j))^2 } \right).\]
	Note that $\{Y_j\}_{j=1}^N$ are i.i.d. random variables.

	For $\E Y_j$, let $\epsilon_1 \coloneqq  \epsilon_D(k_0, c_{\max}) > 0$, where $\epsilon_D$ is defined in \eqref{eq:def-epsD}. $\epsilon_1$ equals the small-$\epsilon$ threshold required by Lemma \ref{lemma:right-operator-degree}(i).
	Thus, for $\epsilon < \epsilon_1$, by \eqref{eq:ker-expan-degree-fast} in Lemma \ref{lemma:right-operator-degree}(i) (with $i=0$), 
	\begin{align} \label{eq:bound-bias-10}
		\E Y_j  &= \int_{\sM}\epsilon^{-d/2} k_0\left(  \frac{\| x - y \|^2}{ \epsilon \phi( \brho_r(x), \brho_r(y))^2 } \right) p(y) dV(y) 
		= m_0 p \brho_r^d(x) + O^{[p]}(\epsilon).
	\end{align}
	
	For $\var(Y_j)$, note that $k_0^2$ satisfies Assumption \ref{assump:k0-smooth} as explained in the proof of Proposition \ref{prop:step2-diff}, and it satisfies Assumption \ref{assump:k0-smooth}(ii) with $a_0[k_0^2] = a_0[k_0]^2$ and $a[k_0^2] = 2 a[k_0]$. 
	We define
	\[ k_7(\eta) \coloneqq a_0[k_0^2]  \exp( - a[k_0^2] \eta ), \quad \eta \geq 0. \]
	Using the same argument that obtains \eqref{eq:bound-var-13}, there exists $\epsilon_2 > 0$ such that for $\epsilon < \epsilon_2$ (this threshold is uniform for $x$ and $r$), the following holds:
	\begin{align} \label{eq:bound-var-14}
		 \var(Y_j) \leq \mu_Y \coloneqq C_{V,4} (\sM, k_0, \phi) \epsilon^{-d/2},  
	\end{align}
	where $C_{V,4} (\sM, k_0, \phi) = 4 c_{\max}^d  m_0[k_7]$.
	
	Besides, $|Y_j| \leq L_Y \coloneqq C_{B,4} (k_0) \epsilon^{-d/2}$, where $C_{B,4} (k_0) = \|k_0\|_\infty$.
	By the same concentration argument to obtain \eqref{eq:bound-k1-sum-D},
	when $N$ is sufficiently large whose threshold depends on $(\sM, k_0, \phi)$, we have $ 40 \frac{\log N}{N}  \leq 9 \frac{C_{V,4} }{(C_{B,4} )^2}  \epsilon^{d/2} $, and then w.p. $\ge 1-2N^{-10}$,
	\begin{align} \label{eq:concen10}
		\frac{1}{N} \sum_{j = 1}^N Y_j   =  \E Y_j + O\left( \sqrt{\frac{\log N}{N}  \mu_Y} \right) =  \E Y_j  + O\left( \sqrt{\frac{\log N}{N \epsilon^{d/2}}}\right).
	\end{align}
	The constants in the big-$O$ notations depend on $\sM, k_0, \phi$ and are uniform for $x$ and $r$. 
	
	Since $\epsilon = o(1)$, when $N$ exceeds certain threshold,
	$ \epsilon <  \min\{  \epsilon_1, \epsilon_2\}$, 
	which is a constant depending on $(\sM, p, k_0, \phi)$ and is uniform for $x$ and $r$, and then \eqref{eq:bound-bias-10}\eqref{eq:bound-var-14} along with the bound for $Y_j$ hold.
	Another large-$N$ threshold is required by the Bernstein inequality to prove \eqref{eq:concen10}, which depends on $(\sM, k_0, \phi)$ and is uniform for $x$ and $r$.
	The overall large-$N$ threshold required by the claim \eqref{eq:bound-k1-sum-D2} depends on $(\sM, p, k_0, \phi)$.

	\vspace{5pt}
	\noindent
	\underline{Bound of $| L_\un f(x)|$.}  
	When $N \ge N_\rho$ and under $E_\rho$, it holds that $\varepsilon_{\rho,k} \le 0.05/L_\phi$ by \eqref{eq:property-Nrho}.
	By Lemma \ref{lemma:lb-ub-E_phi}, $\varepsilon_{\rho,k} \le E_{\phi,k} \le 0.1$. This implies that $\hrho(x) \ge 0.9 \brho_{r_k}(x)$ and $\phi(\hrho(x), \hrho(x_j)) \le 1.1 \phi(\brho_{r_k}(x), \brho_{r_k}(x_j))$ for any $x \in \sM$ and any $j = 1,\dots,N$.
	Therefore, for the fixed $x \in \sM$ in the theorem, we can verify that
	\begin{align} \label{eq:bound-L-un-1}
		| L_\un f(x)| &\le \frac{1}{\frac{m_2}{2}N\epsilon \hrho(x)^2 }\sum_{j=1}^N \epsilon^{-d/2}
		k_0 \left(  \frac{\| x - x_j\|^2}{ \epsilon \phi(\hrho(x) ,  \hrho(x_j))^2 } \right) \left|f(x_j)-f(x) \right| \notag\\
		   &\le \frac{2}{0.9^2m_2} \frac{1}{\epsilon \brho_{r_k}(x)^2} \frac{1}{N} \sum_{j=1}^N \epsilon^{-d/2} a_0[k_0] \exp\left(  -\frac{ a[k_0] \, \| x - x_j\|^2}{ \epsilon \phi(\hrho(x) ,  \hrho(x_j))^2 } \right) \left|f(x_j)-f(x) \right| \notag \\
		   &\le \frac{3}{m_2} \frac{1}{\epsilon \brho_{r_k}(x)^2} \frac{1}{N} \sum_{j=1}^N \epsilon^{-d/2} k_8 \left(  \frac{\| x - x_j\|^2}{ \epsilon \phi(\brho_{r_k}(x) ,  \brho_{r_k}(x_j))^2 } \right)
		    \left|f(x_j)-f(x) \right|, 
	\end{align}
	where
	\begin{align} \label{eq:def-k8}
	k_8(\eta) \coloneqq  a_0[k_0] \exp\left( - \frac{a[k_0]}{1.1^2} \eta\right), \quad \eta \ge 0,
	\end{align}
	and $k_8$ satisfies Assumption \ref{assump:k0-smooth}.
The first inequality is by the definition of $L_\un f$;
	the second inequality is by the fact that $k_0$ satisfies Assumption \ref{assump:k0-smooth}(ii) and that $\hrho(x) \ge 0.9 \brho_{r_k}(x)$;
		and the last inequality holds because $\phi(\hrho(x), \hrho(x_j)) \le 1.1 \phi(\brho_{r_k}(x), \brho_{r_k}(x_j))$ and $\exp(-a[k_0]\eta)$ is monotonic.

	By the same arguments used to obtain the claim in \eqref{eq:bound-k1-sum}, when $N$ is sufficiently large 
	whose threshold depends on $(\sM, p, k_0, \phi)$, for any $r \in [0, \tilde{r}_0]$ and any $x \in \sM$, w.p. $\ge 1-2N^{-10}$, 
	\begin{align}\label{eq:bound-k8-sum}
		\frac{1}{N}\sum_{j=1}^N \epsilon^{-d/2}k_8\left(  \frac{\| x - x_j\|^2}{ \epsilon \phi( \brho_r(x), \brho_r(x_j))^2 } \right) \left|f(x_j)-f(x)\right| =  O\left(\|\nabla f\|_\infty p(x)^{-1/d}\sqrt{\epsilon}\right) + O(\| f \|_\infty \epsilon^{10}).  
	\end{align}
	Here, the threshold for $N$ and the constant in the big-$O$ notation are uniform for both $r$ and $x$. We will apply this claim at $r = r_k$ and $x$ in the theorem, and we denote the good event as $E_5 $.
	Then, using the similar arguments to obtain \eqref{eq:bound-k1-sum-2}, we conclude that under $E_\rho \cap E_5 $ and when $N$ exceeds $N_\rho$ and the threshold required by \eqref{eq:bound-k8-sum}, we have
	\begin{align} \label{eq:bound-L-un}
		\left| L_\un f(x)\right| = O\left(\|\nabla f\|_\infty p(x)^{1/d} \frac{1}{\sqrt{\epsilon}} \right) 
		+ O( \|f\|_\infty p(x)^{2/d}  \epsilon^{9}). 
	\end{align}

	\vspace{5pt}
	\noindent
	\underline{Bound of $| L_\rw f(x) - \sL_pf(x)|$.} 
	
	We first derive upper bounds of  $|D(x) - \E \bar{D}(x)|$ and lower bounds of $D(x)$ and $\E \bar{D}(x)$.
	We use the claims \eqref{eq:bound-D-Dbar} and \eqref{eq:bound-Dbar-1}.
	We denote by $N_1 $ the large-$N$ needed by the two claims;
	the claim in \eqref{eq:bound-D-Dbar} holds under the good event $E_\rho \cap E_3 $
	and the claim in \eqref{eq:bound-Dbar-1} needs the good event $E_4 $.
	Thus, when $N \ge N_1 $, under $E_\rho \cap E_3  \cap E_4 $ (which happens w.p. $\ge 1- 5N^{-10}$),
	\begin{align} \label{eq:bound-D-1}
		\left| D(x) - \E \bar{D}(x)\right|  = O(\varepsilon_{\rho,k})  + O\left( \sqrt{\frac{\log N}{N \epsilon^{d/2}}}\right).
	\end{align}	
Then, we can show that, when $N \ge N_1 $ also exceeds some $N_2 $ (which depends on $(\sM, p, k_0, \phi)$), and under the same good event which ensures that \eqref{eq:bound-D-1} holds, i.e. $E_\rho \cap E_3  \cap E_4 $,
	\begin{align}  \label{eq:lower-bound-D}
		D(x) \ge 1/2 \quad  \text{ and } \quad  \E \bar{D}(x) \ge 1/2.
	\end{align}
	This is because the bound of $\left| D(x) - \E \bar{D}(x)\right| $  in \eqref{eq:bound-D-1} and the bound of $ \left|  \E \bar{D}(x)-1 \right|$ in \eqref{eq:bound-Dbar-1} are $o^{[p]}(1)$: 
	Since $\log N \ll k \ll N $, 
	we have $\varepsilon_{\rho,k} = o^{[p]}(1)$ by \eqref{eq:vareps-rho-bound}.
    We also have $\epsilon^{d/2} = \Omega(\log N/N)$ which is assumed in the theorem condition. 
    These imply that the bound in \eqref{eq:bound-D-1} is $o^{[p]}(1)$.
    Furthermore, we also have $r_k = o(1)$ by the definition of $r_k$ in \eqref{eq:def-rk} and $k \ll N$, and $\epsilon = o(1)$ under the theorem condition, which imply that the bound of $ \left|  \E \bar{D}(x)-1 \right|$ in \eqref{eq:bound-Dbar-1} is $o^{[p]}(1)$.

We are ready to bound the three terms $R_1(x)$, $R_2(x)$, $R_3(x)$
in  \eqref{eq:proof-GL-rw} respectively. 

For $R_1(x)$, by the upper bound of $ \left| L_\un f(x)\right|$ in  \eqref{eq:bound-L-un}, the upper bound of $|D(x) - \E \bar{D}(x)|$ in \eqref{eq:bound-D-1}, and the lower bounds of $D(x), \E \bar{D}(x)$ in \eqref{eq:lower-bound-D}, we have that
\begin{align} \label{eq:bound-circ-1}
	R_1(x) = O\left( \|\nabla f \|_\infty p(x)^{1/d} 
	\left(  \frac{\varepsilon_{\rho,k}}{\sqrt{\epsilon}}    + \sqrt{\frac{\log N}{N\epsilon^{d/2+1}}} \right)  \right) + O^{[f,p]}\left( \epsilon^9 \left( \varepsilon_{\rho,k} + \sqrt{\frac{\log N}{N \epsilon^{d/2}}}\right) \right).
\end{align}
This holds when $N$ exceeds $\max\{N_1 , N_2 \}$, and the large-$N$ threshold required by \eqref{eq:bound-L-un}, denoted by $N_3 $, and under the good event $E_\rho \cap E_3  \cap E_4   \cap E_5 $.

For $R_2(x)$, by the upper bound of $\left| L_\un f(x) - \sL_pf(x) \right|$ in \eqref{eq:fast-rate} in Theorem \ref{thm:conv-un-Laplacian-case1} and the lower bound of $\E \bar{D}(x)$ in \eqref{eq:lower-bound-D}, we have that
\begin{align} \label{eq:bound-circ-2}
	R_2(x) = O^{[f,p]} \left(\epsilon  + r_k \right) +  O\left( \|\nabla f \|_\infty p(x)^{1/d} 
	\left(  \frac{\varepsilon_{\rho,k}}{\sqrt{\epsilon}}    + \sqrt{\frac{\log N}{N\epsilon^{d/2+1}}} \right)  \right).
\end{align}
This additionally needs the large-$N$ threshold required in Theorem \ref{thm:conv-un-Laplacian-case1}, which is denoted by $N_4 $ (which depends on $(\sM, p, k_0, \phi)$ as shown in the proof of the theorem);
the proof of Theorem \ref{thm:conv-un-Laplacian-case1} has defined two good events $E_1 $ and $E_2 $, and their intersection  happens w.p. $\ge 1- 5N^{-10}$.
This bound of $R_2(x)$ holds when $N \ge \max\{N_1 , N_2 , N_4 \}$ and under the good event $E_1  \cap E_2  $.

For $R_3(x)$, by the upper bound of $|\E \bar{D}(x)-1|$ in \eqref{eq:bound-Dbar-1} and the lower bound of $\E \bar{D}(x)$ in \eqref{eq:lower-bound-D}, we have that
\begin{align} \label{eq:bound-circ-3}
	R_3(x) = O^{[p]} \left( \|\sL_pf\|_\infty \ \left(  \epsilon + r_k \right)\right) = O^{[f,p]} \ \left(  \epsilon + r_k \right).
\end{align}
Here,
$ \|\sL_pf \|_\infty$ is an $O^{[f,p]}(1)$ constant that
involves $\|\Delta f \|_\infty$, $\|\nabla f \|_\infty$, $\|\nabla p \|_\infty$ and $p_{\min}$.
This bound of $R_3(x)$ holds when $N \ge N_1 $. 

Combining the bounds of $R_1(x)$,  $R_2(x)$, and $R_3(x)$, we have that when $N \geq  \max\{ N_1 , N_2 , N_3 , N_4 \}$, under the good event $E_\rho \cap E_1  \cap E_2  \cap E_3  \cap E_4  \cap E_5 $ that happens w.p. $\ge 1-11N^{-10}$,
\begin{align*}
	 &| L_\rw f(x) - \sL_pf(x)|
	= O^{[f,p]} \left(\epsilon  + r_k \right) +  O\left( 
	\|\nabla f \|_\infty p(x)^{1/d} 
	\left(  \frac{\varepsilon_{\rho,k}}{\sqrt{\epsilon}}    + \sqrt{\frac{\log N}{N\epsilon^{d/2+1}}} \right)  \right) \\
	&~~~~+ O^{[f,p]}\left( \epsilon^9 \left( \varepsilon_{\rho,k} + \sqrt{\frac{\log N}{N \epsilon^{d/2}}}\right) \right).
\end{align*}
Observe that $O^{[f,p]}( \epsilon^9 ( \varepsilon_{\rho,k} + \sqrt{\frac{\log N}{N \epsilon^{d/2}}}) ) = O^{[f,p]}(\epsilon)$, 
we have proved that the bound in \eqref{eq:fast-rate} holds for $|L_\rw f(x) - \sL_p f(x)|$. 
The overall large-$N$ threshold depends on $(\sM, p, k_0, \phi)$.
\end{proof}

	\begin{proof}[Proof of Lemma \ref{lemma:step1-diff-degree}]

By definitions of $D(x)$ in \eqref{eq:def-D} and $\bar{D}(x)$ in \eqref{eq:def-Dbar},
	 \begin{align} \label{eq:bound-D-Dbar-tri-ineq}
 	\left| D(x) - \bar{D}(x) \right| \le \frac{1}{m_0 N} \sum_{j=1}^N \epsilon^{-d/2}  \left| k_0 \left(  \frac{\| x - x_j\|^2}{ \epsilon \phi(\hrho(x), \hrho(x_j))^2 } \right)- k_0 \left(  \frac{\| x - x_j\|^2}{ \epsilon \phi( \brho_{r_k}(x), \brho_{r_k}(x_j))^2 } \right)\right|. 
 	\end{align}
 	This holds for $k_0$ satisfying Assumption \ref{assump:k0-smooth} 
	and $\phi$ satisfying Assumption \ref{assump:phi-diff}(i)(iii).

Next, since  Lemma \ref{lemma:step1-diff} holds, we recall the bound of $\left| k_0 \left(  \frac{\| x - x_j\|^2}{ \epsilon \phi(\hrho(x), \hrho(x_j))^2 } \right)- k_0 \left(  \frac{\| x - x_j\|^2}{ \epsilon \phi( \brho_{r_k}(x), \brho_{r_k}(x_j))^2 } \right)\right|$ in \eqref{eq:bound-k0-diff-1} in the proof of Lemma \ref{lemma:step1-diff}. Then, by plugging the bound in \eqref{eq:bound-k0-diff-1} into the inequality \eqref{eq:bound-D-Dbar-tri-ineq}, we have that
			\begin{align*}
			\left| D(x) - \bar{D}(x) \right| &\le \frac{1}{m_0 N} \sum_{j=1}^N \epsilon^{-d/2}  6L_\phi \varepsilon_{\rho,k} \, a_1[k_0] \frac{\| x - x_j\|^2}{ \epsilon \phi( \brho_{r_k}(x), \brho_{r_k}(x_j))^2 }    \exp \left(  -\frac{a[k_0]}{2}\frac{\| x - x_j\|^2}{ \epsilon \phi( \brho_{r_k}(x), \brho_{r_k}(x_j))^2 } \right)\\
			&= \frac{6L_\phi}{m_0} \; \varepsilon_{\rho,k} \; \frac{1}{N}
			\sum_{j=1}^N \epsilon^{-d/2}\tilde{k}_1\left(  \frac{\| x - x_j\|^2}{ \epsilon \phi( \brho_{r_k}(x), \brho_{r_k}(x_j))^2 } \right),
		\end{align*}
		where $\tilde{k}_1$ is as defined in \eqref{eq:def-k1-tilde}.
		This proves the lemma.
		\end{proof}

\subsection{Proof of Theorem \ref{thm:conv-un-Laplacian-hrho-norm}}

The proof of Theorem \ref{thm:conv-un-Laplacian-hrho-norm} parallels that of Theorem \ref{thm:conv-un-Laplacian-case1}. 
We first replace $\hrho$ with $\brho_{r_k}$ in the definition of $\widetilde{L}_\un$ in \eqref{eq:def-L-un-tilde}, 
that is, we define
\begin{equation}\label{eq:def-Wbar-GL-norm}
	\bar{L}_\un f(x)   \coloneqq \frac{1}{\frac{m_2}{2}N\epsilon}       \sum_{j=1}^N \epsilon^{-d/2}k_0\left(\frac{\|x-x_j\|^2}{\epsilon \phi( \brho_{r_k}(x), \brho_{r_k}(x_j))^2  }\right)\frac{f(x_j)-f(x)}{  \phi( \brho_{r_k}(x), \brho_{r_k}(x_j))^2 }.  
\end{equation}
The notation $\bar{L}_\un$ here is only used in the proof of Theorem \ref{thm:conv-un-Laplacian-hrho-norm} and differs from \eqref{eq:def-Wbar-GL}.

\begin{proof}[Proof of Theorem \ref{thm:conv-un-Laplacian-hrho-norm}]

	If $\| \nabla f \|_\infty = 0$, then $f$ is a constant function,
	and then the theorem holds trivially since 
 $\widetilde{L}_\un f\equiv 0$ and $\Delta_pf \equiv 0$.
	Below, we consider when $\| \nabla f \|_\infty > 0$.
	Under the condition of the theorem, \eqref{eq:property-Nrho} holds.
	Therefore, when $N \geq N_\rho$, $r_k \leq \tilde{r}_0$ (defined in \eqref{eq:def-r0-tilder0}), which implies that $\brho_{r_k}$ is well-defined by Lemma \ref{lemma:bar-rho-epsilon}(i).
	Then, $\bar{L}_\un$ defined in \eqref{eq:def-Wbar-GL-norm} is also well-defined.

	\vspace{5pt}
	\noindent 
	\underline{Step 1.} to bound $| \widetilde{L}_\un f(x)  - \bar{L}_\un f(x) |$. 
	
	Similar to Lemma \ref{lemma:step1-diff} used in the proof of Proposition \ref{prop:step1-diff}, the following lemma bounds $| \widetilde{L}_\un f(x)-\bar{L}_\un f(x)|$ by $\varepsilon_{\rho,k}$ multiplied with a discrete sum that only involves the population bandwidth function $\brho_{r_k}$.
		The proof of Lemma \ref{lemma:step1-diff-Wtilde} is postponed after we finish Theorem \ref{thm:conv-un-Laplacian-hrho-norm}.

	\begin{lemma}\label{lemma:step1-diff-Wtilde}
	Under the same conditions of Lemma \ref{lemma:step1-diff}, for any $\epsilon > 0$, any $f: \sM \to \R$, and any $x \in \sM$, 
		\begin{align}   \label{eq:lemma-step1-diff-Wtilde}
			\left|\widetilde{L}_\un f(x)-\bar{L}_\un f(x)\right|&\leq 
			\frac{C \varepsilon_{\rho,k}}{\epsilon} \;\frac{1}{N}\sum_{j=1}^N \epsilon^{-d/2}k_1\left(  \frac{\| x - x_j\|^2}{ \epsilon \phi( \brho_{r_k}(x), \brho_{r_k}(x_j))^2 } \right) \frac{\left|f(x_j)-f(x)\right|}{ \phi( \brho_{r_k}(x), \brho_{r_k}(x_j))^2  },  
		\end{align}
		where $C = 18 L_\phi / m_2[k_0]$ is a constant depending on $k_0$ and $\phi$, and $k_1$ is as defined in \eqref{eq:def-k1}.
		
	\end{lemma}

	We claim that for $k_0$ under Assumption \ref{assump:k0-smooth} and $\phi$ under Assumption \ref{assump:phi-diff}(i)(iv),
	when $N$ is sufficiently large with the threshold depending on $(\sM, p, k_0, \phi)$, 
	for any $r \in [0, \tilde{r}_0]$ and any $x \in \sM$, w.p. $\ge 1-2N^{-10}$,
	\begin{align}\label{eq:bound-k1-sum-Wtilde}
		\frac{1}{N}\sum_{j=1}^N \epsilon^{-d/2}k_1\left(  \frac{\| x - x_j\|^2}{ \epsilon \phi( \brho_r(x), \brho_r(x_j))^2 } \right) \frac{\left|f(x_j)-f(x)\right|}{ \phi( \brho_r(x), \brho_r(x_j))^2  } =  O\left(\|\nabla f\|_\infty p(x)^{1/d}\sqrt{\epsilon}\right) + O^{[p]}(\|f\|_\infty \epsilon^{10}).  
	\end{align}
	Here, the threshold for $N$ and the constant in the big-$O$ notation are uniform for both $r$ and $x$. 
	The claim can be applied at $r = r_k$ since $r_k \leq \tilde{r}_0$ when $N \geq N_\rho$.

	We denote by $E_6 $ the good event such that \eqref{eq:bound-k1-sum-Wtilde} holds at $r = r_k$ and $x \in \sM$ in the theorem.
	Suppose the claim holds, then \eqref{eq:bound-k1-sum-Wtilde} gives an upper bound of the r.h.s. of \eqref{eq:lemma-step1-diff-Wtilde}.
		Here, Lemma \ref{lemma:step1-diff-Wtilde} can be applied  because when $N \ge N_\rho$ and under the good event $E_\rho$, the conditions $r_k \leq r_0$ and $\varepsilon_{\rho,k} \leq \min\{ \delta_\phi, 0.05 / L_\phi\}$ of the lemma are satisfied.
	Therefore, we have that
	\[ 	\left|\widetilde{L}_\un f(x)-\bar{L}_\un f(x)\right|\leq C \frac{\varepsilon_{\rho,k}}{\epsilon} \left( O\left(\|\nabla f\|_\infty p(x)^{1/d}\sqrt{\epsilon}\right)  + O^{[p]}(\|f\|_\infty \epsilon^{10})   \right), \]
	which further implies
	\begin{align} \label{eq:bound-Ltilde-Lbar}
		\left|\widetilde{L}_\un f(x)-\bar{L}_\un f(x)\right| = O\left(\|\nabla f\|_\infty p(x)^{1/d}
		 \frac{\varepsilon_{\rho,k}}{\sqrt{\epsilon}} \right) + O^{[p]}( \|f\|_\infty \varepsilon_{\rho,k} \epsilon^{9}  ).
	\end{align}
	This holds when $N$ exceeds both the large-$N$ threshold and $N_\rho$ in \eqref{eq:property-Nrho} and w.p. $\ge 1-3N^{-10}$ (under the good event $E_\rho$ in \eqref{eq:property-Nrho} and $E_6 $).
		
	Below, we prove the claim \eqref{eq:bound-k1-sum-Wtilde}. 
	Recall the definition of $\delta_\epsilon[k_1]$ in \eqref{eq:def-delta-eps}. We first truncate $k_1$ by \eqref{eq:k1-trunc}, and then bound $|f(x_j) - f(x)|$ using Lemma \ref{lemma:f-diff-bound} (applied to $f$) and the metric comparison \eqref{eq:metric-comp2}.
	To apply \eqref{eq:metric-comp2}, we will need $\epsilon$ small enough such that $\delta_\epsilon[k_1] < \delta_1(\sM)$ introduced in Lemma \ref{lemma:M-metric}(iii).
	Recall the definition of $\epsilon_D$ in \eqref{eq:def-epsD}, and let $\epsilon_1 \coloneqq \epsilon_D(k_1, c_{\max}) > 0$. We denote $B^m_ {\delta_\epsilon[k_1]}$ as $B_ {\delta_\epsilon[k_1]}$. Then,
	for $\epsilon < \epsilon_1$, we have $\delta_\epsilon[k_1] < \delta_1(\sM)$.
	By a similar argument to that in the proof of \eqref{eq:bound-k1-sum}, we obtain that when $\epsilon < \epsilon_1$, it holds that 
	\begin{align} \label{eq:k1-trunc-independ-sum-Wtilde}
		& \frac{1}{N}\sum_{j=1}^N \epsilon^{-d/2}k_1\left(  \frac{\| x - x_j\|^2}{ \epsilon \phi( \brho_r(x), \brho_r(x_j))^2 } \right) \frac{\left|f(x_j)-f(x)\right|}{ \phi( \brho_r(x), \brho_r(x_j))^2  }  \notag \\
		&\leq \|\nabla f\|_\infty  1.1  \frac{1}{N}\sum_{j=1}^N \epsilon^{-d/2}k_1\left(  \frac{\| x - x_j\|^2}{ \epsilon \phi( \brho_r(x), \brho_r(x_j))^2 } \right)  \mathbf{1}_{  \{x_j \in B_ {\delta_\epsilon[k_1]}(x) \} } \frac{ \|x-x_j\|  }{ \phi( \brho_r(x), \brho_r(x_j))^2  }   + R_{\epsilon, T}(x,r), \notag \\
		& \qquad\qquad\qquad\qquad\qquad\qquad\qquad\text{where }\sup_{x\in \sM,r \in [0,\tilde{r}_0]} |  R_{\epsilon, T}(x,r)| = O^{[ p]}(\|f\|_\infty\epsilon^{10} ),
	\end{align}
	where the constant in $O^{[ p]}(\|f\|_\infty\epsilon^{10} )$ depends on $\sM$, $p$, and $k_0$ and is uniform for $x$ and $r$.
	Below, we bound the independent sum $\frac{1}{N} \sum_{j=1}^N Y_j$, where
	\begin{align*}
		Y_j \coloneqq \epsilon^{-d/2}k_1\left(  \frac{\| x - x_j\|^2}{ \epsilon \phi( \brho_r(x), \brho_r(x_j))^2 } \right)  \mathbf{1}_{  \{x_j \in B_ {\delta_\epsilon[k_1]}(x) \} } \frac{ \|x-x_j\|  }{ \phi( \brho_r(x), \brho_r(x_j))^2  }.
	\end{align*}

	Then, we bound $\E Y_j$ and $\var(Y_j)$.	
	For $\E Y_j$, we use a similar argument to that which gives \eqref{eq:bound-bias-1}. 
 The integrand of $\E Y_j$ is different from that of the expectation bounded in \eqref{eq:bound-bias-1}. Instead of the upper bound $\phi( \brho_r(x), \brho_r(y)) \leq c_{\max} \max\{\brho_r(x), \brho_r(y)\}$ used to obtain \eqref{eq:bound-bias-1}, we use the following upper bound of $\phi( \brho_r(x), \brho_r(y))^{-1} $, 
	\begin{align} \label{eq:phi^-1-lb}
		\phi( \brho_r(x), \brho_r(y))^{-1} \leq c_{\min}^{-1} \min\{ \brho_r(x), \brho_r(y) \}^{-1} \leq c_{\min}^{-1} ( \brho_r(x)^{-1} +  \brho_r(y)^{-1}  ).
	\end{align}
 Then, we can upper bound $\E Y_j$ as
	\begin{align}
		\E Y_j &=  \int_\sM  \epsilon^{-d/2}k_1\left(  \frac{\| x - y\|^2}{ \epsilon \phi( \brho_r(x), \brho_r(y))^2 } \right)  \frac{ \|x-y\|}{ \phi( \brho_r(x), \brho_r(y))  } \frac{p(y)}{\phi( \brho_r(x), \brho_r(y))} dV(y) \notag \\
		&\leq  \sqrt{\epsilon}  \int_\sM  \epsilon^{-d/2}k_2\left(  \frac{\| x - y\|^2}{ \epsilon \phi( \brho_r(x), \brho_r(y))^2 } \right)  \frac{ p(y)}{ \phi( \brho_r(x), \brho_r(y))  } dV(y) \notag \\
		&\leq \sqrt{\epsilon}  c_{\min}^{-1}  \int_\sM \epsilon^{-d/2}k_2\left( \frac{\| x - y\|^2}{ \epsilon  c_{\max}^2\max\{ \brho_r(x), \brho_r(y) \}^2 } \right) (\brho_r(x)^{-1} + \brho_r(y)^{-1} ) p(y) dV(y), \label{eq:bound-bias-18}
	\end{align}
	where $k_2$ is as defined in \eqref{eq:def-k2} and satisfies $k_1(\eta)\sqrt{\eta} \leq k_2(\eta)$.
	We apply Lemma \ref{lemma:right-operator-degree}(ii) to compute the integral in \eqref{eq:bound-bias-18}
(with $\epsilon$ in the lemma being $\epsilon c_{\max}^2$ and $k_0$ in the lemma being $k_2$),
which holds as long as $\epsilon < \epsilon_2: =\epsilon_D(k_2, 1) / c_{\max}^2$, and have that
\begin{align*} 
	\E Y_j &\leq 2 \sqrt{\epsilon}  c_{\min}^{-1}c_{\max}^d \left( m_0[k_2] p \brho_r^{d-1}(x) + R_{\epsilon,1}(x,r) \right), \quad \sup_{x\in \sM,r \in [0,\tilde{r}_0]}|R_{\epsilon,1}(x,r)| 
	= 
	{O^{[p]}(\sqrt{\epsilon})}, 
\end{align*}
where the constant dependence in 
$
{O^{[p]}(\sqrt{\epsilon})}$ is as in Lemma \ref{lemma:right-operator-degree}(ii) and this 
constant depends on $\phi$ only through $c_{\max}$.
By applying a similar argument that obtains \eqref{eq:bound-bias-1}, 
we can conclude that there is $\epsilon_3 > 0$ such that for $\epsilon <  \epsilon_3$ (this threshold is uniform for $x$ and $r$),
	\begin{align} \label{eq:bound-bias-12}
			\E Y_j \leq C_{E,3} (\sM, k_0, \phi)  p(x)^{1/d} \sqrt{\epsilon},  
	\end{align}
	where $C_{E,3} (\sM, k_0, \phi) =  2^{(d-1)/d+2} c_{\min}^{-1}c_{\max}^{d}  m_0[k_2]$.

 For $\var(Y_j)$, we use a similar argument used to obtain \eqref{eq:bound-var-6}. Here, we use the upper bound of $\phi( \brho_r(x), \brho_r(y))^{-2}$ below 
	\begin{align} \label{eq:phi^-2-lb}
		\phi( \brho_r(x), \brho_r(y))^{-2} \leq c_{\min}^{-2} \min\{ \brho_r(x), \brho_r(y) \}^{-2} \leq c_{\min}^{-2} ( \brho_r(x)^{-2} +  \brho_r(y)^{-2}  )
	\end{align}
	to upper bound $\var(Y_j)$ as
	\begin{align}
		\epsilon^{d/2}\var(Y_j) &\leq \epsilon^{d/2} \E Y_j^2 = \int_\sM  \epsilon^{-d/2}k_1^2\left(  \frac{\| x - y\|^2}{ \epsilon \phi( \brho_r(x), \brho_r(y))^2 } \right)  \frac{ \|x-y\|^2}{ \phi( \brho_r(x), \brho_r(y))^2  } \frac{p(y)}{\phi( \brho_r(x), \brho_r(y))^2} dV(y) \notag \\
		&\leq \epsilon  \int_\sM  \epsilon^{-d/2}k_3\left(  \frac{\| x - y\|^2}{ \epsilon \phi( \brho_r(x), \brho_r(y))^2 } \right)  \frac{p(y)}{\phi( \brho_r(x), \brho_r(y))^2} dV(y) \notag \\
		&\leq \epsilon  c_{\min}^{-2}  \int_\sM \epsilon^{-d/2}k_3\left( \frac{\| x - y\|^2}{ \epsilon  c_{\max}^2\max\{ \brho_r(x), \brho_r(y) \}^2 } \right) (\brho_r(x)^{-2} + \brho_r(y)^{-2} ) p(y) dV(y), \label{eq:bound-var-20}
	\end{align}
	where $k_3$ is as defined in \eqref{eq:def-k3} and satisfies $k_1^2(\eta) \eta \leq k_3(\eta)$.
We apply Lemma \ref{lemma:right-operator-degree}(ii) (with $\epsilon$ in the lemma being $\epsilon c_{\max}^2$ and $k_0$ in the lemma being $k_3$) to compute the integral in \eqref{eq:bound-var-20}, which holds as long as 
$\epsilon <  \epsilon_4 \coloneqq \epsilon_D(k_3, 1) / c_{\max}^2$, 
and have that
\begin{align*} 
	\epsilon^{d/2} \var(Y_j)   &\leq   2\epsilon c_{\min}^{-2}c_{\max}^d \left( m_0[k_3] p \brho_r^{d-2}(x) + R_{\epsilon,2}(x,r) \right), \quad \sup_{x\in \sM,r \in [0,\tilde{r}_0]}|R_{\epsilon,2}(x,r)| 
	= 
	{O^{[p]}(\sqrt{\epsilon})}, 
\end{align*}
where the constant in 
$
{O^{[p]}(\sqrt{\epsilon})}$ depends on $\phi$ only through $c_{\max}$.
By a similar argument used to obtain \eqref{eq:bound-var-6}, we can conclude 
 that there is $\epsilon_5 > 0$ such that for $\epsilon < \epsilon_5$ (this threshold is uniform for $x$ and $r$),
	\begin{align} \label{eq:bound-var-17}
				\var(Y_j) \leq C_{V,5} (\sM, k_0, \phi)  p(x)^{2/d} \epsilon^{1-d/2},  
	\end{align}
	where $C_{V,5} (\sM, k_0, \phi) =  {4\max\{3/2,2^{(d-2)/d}\}} c_{\min}^{-2}c_{\max}^{d}  m_0[k_3]$. 
	
	Moreover, using a similar approach to bound $|Y_j|$ as in the proof of \eqref{eq:bound-k1-sum}, we have
	\begin{align*} 
		 |Y_j|   \leq L_Y \coloneqq C_{B,5} (\sM, p, k_0,\phi) \epsilon^{1/2-d/2},
	\end{align*}
	where  $C_{B,5} (\sM, p, k_0,\phi) = a_0[k_2] c_{\min}^{-1} \rho_{\min}^{-1} $.
	Then, by a similar concentration argument in the proof of \eqref{eq:bound-k1-sum}, 
	when $N$ is sufficiently large whose threshold depends on $(\sM, p, k_0, \phi)$, we have
	$ 40 \frac{\log N}{N}  \leq 9 \frac{C_{V,5} }{(C_{B,5} )^2} p_{\min}^{2/d} \epsilon^{d/2} $ and $\epsilon^{d/2} \geq \log N/N$, and then w.p. $\ge 1-2N^{-10}$,
	\begin{align} \label{eq:concen14}
		0 \leq   \frac{1}{N} \sum_{j = 1}^N Y_j   \leq (C_{E,3}  + \sqrt{40 C_{V,5} }) {p(x)^{1/d}}  \sqrt{\epsilon} = O ({p(x)^{1/d}}  \sqrt{\epsilon}) .
	\end{align}
	The constant in the big-$O$ notation depends on $\sM, k_0, \phi$ through $C_{E,3} $ and $C_{V,5} $ and is uniform for $x$ and $r$. 
Plugging the upper bound of $\frac1N\sum_{j=1}^N Y_j$ in \eqref{eq:concen14} back to the r.h.s. of \eqref{eq:k1-trunc-independ-sum-Wtilde}, \eqref{eq:bound-k1-sum-Wtilde} is proved.
	
	Since $\epsilon = o(1)$, when $N$ exceeds a certain threshold, 
	$\epsilon < \min\{  \epsilon_1, \epsilon_2, \dots, \epsilon_5 \}$,
	which is a constant depending on $(\sM, p, k_0, \phi)$ and is uniform for $x$ and $r$, and then the bounds in \eqref{eq:k1-trunc-independ-sum-Wtilde}\eqref{eq:bound-bias-12}\eqref{eq:bound-var-17} along with the bound of $|Y_j|$ hold.
	Another large-$N$ threshold is required by the Bernstein inequality used to prove \eqref{eq:concen14}.
	The overall large-$N$ threshold required by the claim in \eqref{eq:bound-k1-sum-Wtilde} depends on ($\sM$, $p$, $k_0$, $\phi$).

	\vspace{5pt}
	\noindent
	\underline{Step 2.} to bound  $| \bar{L}_\un f(x) - \Delta_pf(x) |$.
	
	We claim that  when $N$ is sufficiently large with the threshold depending on $(\sM, p, k_0, \phi)$, for any $r \in [0, \tilde{r}_0]$ and any $x \in \sM$, w.p. $\ge 1-2N^{-10}$,
	\begin{align} \label{eq:step2-discrete-sum-diff-Wtilde}
		&\frac{1}{N}\sum_{j=1}^N \epsilon^{-d/2}  k_0 \left(  \frac{\| x - x_j\|^2}{ \epsilon \phi( \brho_r(x), \brho_r(x_j))^2 } \right) \frac{f(x_j)-f(x)}{\phi( \brho_r(x), \brho_r(x_j))^2 }  \notag \\
		&=  \epsilon \frac{m_2}{2}  \Delta_p  f(x) + O^{[f,p]}\left(\epsilon^2  +   \epsilon r \right)  +  O\left(  \|\nabla f\|_\infty p(x)^{1/d}  \sqrt{\frac{\log N}{N\epsilon^{d/2-1}}}\right).  
	\end{align}
	Here, the threshold for $N$ and the constants in the big-$O$ notations are uniform for $x$ and $r$. 
	We apply the claim at $r = r_k$ and $x \in \sM$ as in the theorem, and denote the good event as $E_7 $.
	
	Observe that $\bar{L}_\un$ (defined in \eqref{eq:def-Wbar-GL-norm}) equals the l.h.s. of \eqref{eq:step2-discrete-sum-diff-Wtilde} divided by $\epsilon m_2 / 2$, where $r$ in \eqref{eq:step2-discrete-sum-diff-Wtilde} takes the value of $r_k$.
	Suppose the claim holds, then when $N$ exceeds both $N_\rho$ (in \eqref{eq:property-Nrho}) and the large-$N$ threshold in the claim, under the good event $E_7 $ that happens w.p. $\ge 1-2N^{-10}$,
	we have that
	\begin{align}  \label{eq:bound-Lbar-Deltap}
		|\bar{L}_\un f(x) - \Delta_p  f(x)| = O^{[f,p]}\left(\epsilon  +   r_k\right)  +  O\left(  \|\nabla f\|_\infty p(x)^{1/d}  \sqrt{\frac{\log N}{N\epsilon^{d/2+1}}}\right). 
	\end{align}
	
	Below, we prove the claim.
	Let $\epsilon_1 \coloneqq  \epsilon_D(k_0, c_{\max}) > 0$, then 
	for $\epsilon < \epsilon_1$, we have $\delta_\epsilon[k_0] < \delta_1(\sM)$.
	Note that for any $r \in [0, \tilde{r}_0]$ and any $x, y \in \sM$, by Assumption \ref{assump:phi-diff}(iv) and Lemma \ref{lemma:bar-rho-epsilon}(ii), we have that 
	$\phi( \brho_r(x), \brho_r(y)) \ge 
	c_{\min} \min\{  \brho_r(x), \brho_r(y) \} \ge c_{\min} \rho_{\min}$.
	Then, together the truncation in \eqref{eq:k0-trunc}, we obtain that 
	\begin{align*}
		& \frac{1}{N}\sum_{j=1}^N \epsilon^{-d/2} k_0 \left(  \frac{\| x - x_j\|^2}{ \epsilon \phi( \brho_r(x), \brho_r(x_j))^2 } \right) \frac{f(x_j)-f(x)}{\phi( \brho_r(x), \brho_r(x_j))^2 } \notag \\
		&=  \frac{1}{N}\sum_{j=1}^N \epsilon^{-d/2}  k_0 \left(  \frac{\| x - x_j\|^2}{ \epsilon \phi( \brho_r(x), \brho_r(x_j))^2 } \right)  \mathbf{1}_{  \{x_j \in B_ {\delta_\epsilon[k_0]}(x) \} }  \frac{f(x_j)-f(x)}{\phi( \brho_r(x), \brho_r(x_j))^2 }    + O^{[f,p]}( \epsilon^{10} ),  
	\end{align*}
	where $B_ {\delta_\epsilon[k_0]}$ is $B^m_ {\delta_\epsilon[k_0]}$.
	We analyze the first term, which is an independent sum $\frac{1}{N} \sum_{j=1}^N Y_j$, where
	\[ Y_j \coloneqq  \epsilon^{-d/2}  k_0 \left(  \frac{\| x - x_j\|^2}{ \epsilon \phi( \brho_r(x), \brho_r(x_j))^2 } \right)  
	\mathbf{1}_{ \{x_j \in B_ {\delta_\epsilon[k_0]}(x) \} }  \frac{f(x_j)-f(x)}{\phi( \brho_r(x), \brho_r(x_j))^2 }.   \]

	To proceed, we introduce a counterpart to Lemma \ref{lemma:step2-diff}, to be proved after we finish the theorem.
\begin{lemma} [Bias error using $\widetilde{W}$] \label{lemma:step2-other-cases-tildeW}
Under Assumptions \ref{assump:M}-\ref{assump:p}, 
$\phi$ under Assumption \ref{assump:phi-diff} and $k_0$ under Assumption \ref{assump:k0-smooth},
let $\tilde{r}_0$ be defined in \eqref{eq:def-r0-tilder0}.
When $p\in C^5(\sM)$, $f \in C^4(\sM)$, 
if $\epsilon < \epsilon_D(k_0, c_{\max})$, then for any $r \in [0, \tilde{r}_0]$ and any $x\in \calM$, 
\begin{equation} \label{eq:lemma-step2-other-cases-rw-fast}
	\int_\sM \epsilon^{-d/2}k_0\left(  \frac{\| {x} - y\|^2}{ \epsilon \phi( \brho_r(x), \brho_r(y))^2 } \right)  \frac{f(y)-f(x)} {\phi( \brho_r(x), \brho_r(y))^2} p(y) dV(y) = \epsilon \frac{m_2}{2}  
	\Delta_p  f(x) 
	+ O^{[f,p]}(\epsilon^2 +  \epsilon r).
\end{equation}
The small-$\epsilon$ threshold $\epsilon_D(k_0, c_{\max})$ is defined in \eqref{eq:def-epsD} and same as in Lemma \ref{lemma:step2-diff}.

\end{lemma}

	For $\E Y_j$, 
	note that $\epsilon_1 = \epsilon_D(k_0, c_{\max})$, which equals the small-$\epsilon$ threshold required in Lemma \ref{lemma:step2-other-cases-tildeW}.
	Therefore,
	for $\epsilon < \epsilon_1$, by the truncation argument in \eqref{eq:k0-trunc} and \eqref{eq:lemma-step2-other-cases-rw-fast} in Lemma \ref{lemma:step2-other-cases-tildeW}, 
	we have that  
	\begin{align*}
		\E Y_j	= \epsilon \frac{m_2}{2}  
		\Delta_p  f(x) 
		+ O^{[f,p]}\left(\epsilon^2 +  \epsilon r\right) .    
	\end{align*}
	
	For $\var(Y_j)$, we use a similar argument that gives \eqref{eq:bound-var-7}.
	The difference is that we use \eqref{eq:phi^-2-lb} to upper bound $\var(Y_j)$, and then apply 
	{\eqref{eq:ker-expan-degree-slowI}} in Lemma \ref{lemma:right-operator-degree}(ii) to obtain the expansion for 
	\[  \int_\sM \epsilon^{-d/2}k_4\left( \frac{\| x - y\|^2}{ \epsilon  c_{\max}^2\max\{ \brho_r(x), \brho_r(y) \}^2 } \right) (\brho_r(x)^{-2} + \brho_r(y)^{-2} ) p(y) dV(y). \]
	Here, $k_4$ is as defined in \eqref{eq:def-k4}.
	We can establish that there is $\epsilon_2 > 0$ such that for $\epsilon < \epsilon_2$ (this threshold is uniform for $x$ and $r$),
	\begin{align*}
		\var(Y_j)  \leq \mu_Y(x) \coloneqq C_{V,6} (\sM, k_0, \phi) \| \nabla f\|_\infty^2 p(x)^{2/d} \epsilon^{1-d/2},
	\end{align*}
	where $C_{V,6} (\sM, k_0, \phi)  = {4\max\{3/2,2^{(d-2)/d}\}} 1.1^2 c_{\min}^{-2}c_{\max}^d m_0[k_4] $.

	Moreover, by a similar argument that obtain \eqref{eq:bound-LY}, we have that 
	\begin{align*}
		|Y_j| \leq L_Y \coloneqq  C_{B,6} (\sM, p, k_0,\phi) \|\nabla f \|_\infty  \epsilon^{1/2-d/2}  ,
	\end{align*}
	where  $C_{B,6} (\sM, p, k_0,\phi) = 1.1  c_{\min}^{-1} \rho_{\min}^{-1} a_0[k_0]/\sqrt{2e a[k_0]}.$
	Then, by a similar concentration argument in the proof of \eqref{eq:step2-discrete-sum-diff-W},
	when $N$ is sufficiently large whose threshold depends on $(\sM, p, k_0, \phi)$, we have
	\[
	160 \frac{\log N}{N} \leq 9 \frac{C_{V,6} }{(C_{B,6} )^2} p_{\min}^{2/d} \epsilon^{d/2} \leq 9\frac{\mu_Y(x)}{L_Y^2}.
	\]
	The inequality between the first and last expressions is equivalent, for $\tau \coloneqq \sqrt{40\mu_Y(x)\log N/N}$, to $2\tau L_Y \leq 3\mu_Y(x)$. Since $Y_j$ is signed, $|Y_j-\E Y_j|\leq 2L_Y$. Thus, Bernstein's inequality applies to $Y_j-\E Y_j$, and w.p. $\ge 1-2N^{-10}$, \eqref{eq:step2-discrete-sum-diff-Wtilde} holds.

	Since $\epsilon = o(1)$, when $N$ exceeds some threshold, 
	$\epsilon < \min\{  \epsilon_1, \epsilon_2\}$,
		which is a constant depending on $(\sM, p, k_0, \phi)$ and is uniform for $x$ and $r$, and then the expansion of $\E Y_j$ and the bounds of $\var(Y_j)$ and $|Y_j|$ hold.
	Another threshold is required in the concentration argument when applying Bernstein inequality.
	The overall large-$N$ threshold required by the claim in \eqref{eq:step2-discrete-sum-diff-Wtilde} depends on $(\sM, p, k_0, \phi)$.

	\vspace{5pt}
	\noindent
	\underline{Step 3.} the desired bound follows that of $| \widetilde{L}_\un f(x) - \bar{L}_\un f(x) |$ and $| \bar{L}_\un f(x) - \Delta_pf(x) |$ in the first two steps.
	
	We collect the large-$N$ thresholds from the first two steps, both of which depend on $(\sM, p, k_0, \phi)$. Additionally, we collect the good event $E_\rho \cap E_6 $ from Step 1, which happens w.p. $\ge 1-3N^{-10}$ and ensures \eqref{eq:bound-Ltilde-Lbar} holds, as well as the good event $E_7 $ from Step 2, happening  w.p.  ${\ge 1-2N^{-10}}$, ensuring \eqref{eq:bound-Lbar-Deltap} holds. 
	Thus, when $N$ exceeds the required thresholds and under the good event $E_\rho \cap E_6  \cap E_7 $, happening  w.p. $\ge 1-5N^{-10}$, by the triangle inequality, $| \widetilde{L}_\un f(x) - \Delta_pf(x) |$ is bounded by 
	the sum of \eqref{eq:bound-Ltilde-Lbar} and \eqref{eq:bound-Lbar-Deltap}, which gives
	\[
	O\left(\|\nabla f \|_\infty p(x)^{1/d} \frac{\varepsilon_{\rho,k}}{\sqrt{\epsilon}}\right)
	+ O^{[p]}( \|f\|_\infty\varepsilon_{\rho,k} \epsilon^{9}) + O^{[f,p]}\left(\epsilon  + r_k\right)  +  O\left(  \|\nabla f\|_\infty p(x)^{1/d}  \sqrt{\frac{\log N}{N\epsilon^{d/2+1}}}\right).
	\]
	Observe that $O^{[p]}( \|f\|_\infty\varepsilon_{\rho,k} \epsilon^{9}) 
	= O^{[f,p]}(\epsilon)$, 
	this bound becomes \eqref{eq:fast-rate}.
	The overall large-$N$ threshold depends on $(\sM, p, k_0, \phi)$. 
\end{proof}

\begin{proof}[Proof of Lemma \ref{lemma:step1-diff-Wtilde}]
By the definitions of $\widetilde{L}_\un$ in \eqref{eq:def-L-un-tilde} and $\bar{L}_\un$ in \eqref{eq:def-Wbar-GL-norm} and the triangle inequality,
we have
	\begin{align}  \label{eq:bound-L-Lbar-Wtilde}
	\left|\widetilde{L}_\un f(x)-\bar{L}_\un f(x)\right|\leq 
	\widetilde{M}_1(x) + \widetilde{M}_2(x),
\end{align}
where
\begin{align*}
	\widetilde{M}_1(x)  &\coloneqq    \frac{1}{\frac{m_2}{2}N\epsilon }\sum_{j=1}^N 
	\epsilon^{-d/2}  \left| k_0 \left(  \frac{\| x - x_j\|^2}{ \epsilon \phi(\hrho(x), \hrho(x_j))^2 } \right)- k_0 \left(  \frac{\| x - x_j\|^2}{ \epsilon \phi( \brho_{r_k}(x), \brho_{r_k}(x_j))^2 } \right)\right|   \frac{\left|f(x_j)-f(x)\right|}{\phi(\hrho(x), \hrho(x_j))^2} , \\
	\widetilde{M}_2(x)  &\coloneqq    \frac{1}{\frac{m_2}{2}N\epsilon  }  \sum_{j=1}^N 
	\epsilon^{-d/2}   k_0 \left(  \frac{\| x - x_j\|^2}{ \epsilon \phi( \brho_{r_k}(x), \brho_{r_k}(x_j))^2 } \right)    \left |  \frac{1}{\phi(\hrho(x), \hrho(x_j))^2} - \frac{1}{\phi( \brho_{r_k}(x), \brho_{r_k}(x_j))^2}  \right| \\
	&\qquad \left|f(x_j)-f(x)\right| .
\end{align*}
	Since $r_k \leq r_0$, $\brho_{r_k}(x)$ for all $x\in \calM$  is well-defined by Lemma \ref{lemma:bar-rho-epsilon}(i), and consequently,  $\bar{L}_\un$ is also well-defined.
By \eqref{eq:bound-L-Lbar-Wtilde},
	it suffices to bound the two terms $\widetilde{M}_1(x)$ and $\widetilde{M}_2(x)$ on the r.h.s.

	To bound $\widetilde{M}_1(x)$, note that 
	since $\phi$ satisfies Assumption \ref{assump:phi-diff}(i)(iii) and $\varepsilon_{\rho,k} \leq \min\{ \delta_\phi, 0.05 / L_\phi\}$, Lemma \ref{lemma:lb-ub-E_phi} holds. Therefore, by \eqref{eq:bound-phi-rela-0.1}, $E_{\phi,k} \leq 0.1$. Consequently, 
	\begin{align} \label{eq:bound-phi--2}
		 \frac{1}{\phi(\hrho(x), \hrho(x_j))^2} \leq  \frac{1}{0.9^2\phi( \brho_{r_k}(x), \brho_{r_k}(x_j))^2}.
	\end{align}
Under the condition of the current lemma, Lemma \ref{lemma:step1-diff} holds, and thus we have  \eqref{eq:bound-k0-diff-1} in the proof of Lemma \ref{lemma:step1-diff}.
Combining \eqref{eq:bound-phi--2} with \eqref{eq:bound-k0-diff-1}, we have that
	\[ \widetilde{M}_1(x) \leq  \frac{12L_\phi}{0.9^2m_2}\frac{ \varepsilon_{\rho,k}}{N\epsilon }\sum_{j=1}^N \epsilon^{-d/2}  \tilde{k}_1 \left( \frac{\| x - x_j\|^2}{ \epsilon \phi( \brho_{r_k}(x), \brho_{r_k}(x_j))^2 } \right)  \frac{\left|f(x_j)-f(x)\right|}{\phi( \brho_{r_k}(x), \brho_{r_k}(x_j))^2}, \]
	where $\tilde{k}_1$ is as defined in	\eqref{eq:def-k1-tilde}.
	
	For $\widetilde{M}_2(x)$, note that $k_0(\eta) \leq a_0[k_0] \exp(-a[k_0]\eta) 
	\le a_0[k_0] \exp(- \frac{a[k_0]}{2}\eta)$, where the first inequality is by Assumption \ref{assump:k0-smooth}(ii). 
	Combined with \eqref{eq:bound-1/phi-hrho-rela} in the proof of Lemma \ref{lemma:step1-diff}, which applies under the conditions of the current lemma, we have
	\[ \widetilde{M}_2(x) \leq \frac{12L_\phi}{m_2}\frac{ \varepsilon_{\rho,k}}{N\epsilon }\sum_{j=1}^N \epsilon^{-d/2}  a_0[k_0]  \exp \left( - \frac{a[k_0] }{2}
	\frac{  \| x - x_j\|^2}{ \epsilon \phi( \brho_{r_k}(x), \brho_{r_k}(x_j))^2 } \right)  \frac{\left|f(x_j)-f(x)\right|}{\phi( \brho_{r_k}(x), \brho_{r_k}(x_j))^2} . 
 \]
 By definitions of $k_1$ in \eqref{eq:def-k1} and $\tilde k_1 $ in \eqref{eq:def-k1-tilde}, we have $k_1(\eta) = \tilde{k}_1(\eta) + a_0[k_0] \exp(-a[k_0]\eta/2)$.
Together with that $\frac{12L_\phi}{0.9^2m_2} \leq C$ and the bounds for $\widetilde{M}_1(x)$ and $\widetilde{M}_2(x)$, we prove \eqref{eq:lemma-step1-diff-Wtilde}.
\end{proof}

\begin{proof}[Proof of Lemma \ref{lemma:step2-other-cases-tildeW}]
	The proof essentially follows the same approach of Lemma \ref{lemma:step2-diff}.
	We inherit the notations therein, including the projected and normal coordinates $u$ and $s$, and the functions $\tilde{f}$, $\tilde{p}$, $\beta_x$, and $\tilde{\beta}_x$. 
	The difference is that we now handle the integral
	\begin{align*}
		\brho_r(x)^{-2} \epsilon^{-d/2} \int_{\sM} 
		{  k_0 \left(  \frac{\| x - y\|^2}{  \epsilon  \phi( \brho_r(x), \brho_r(y) )^2  } \right)}   (f(y)-f(x)) p(y) \left( \frac{\brho_r(x)}{  \phi( \brho_r(x), \brho_r(y) ) } \right)^2 dV(y),
	\end{align*}
	which includes an additional term $\brho_r(x)^2 / \phi( \brho_r(x), \brho_r(y) )^2$ in the integrand.
	The ratio $\brho_r(x) / \phi( \brho_r(x), \brho_r(y) )$ has been analyzed in the proof of Lemma \ref{lemma:step2-diff} by a local expansion, which we will utilize here.

	Following the first three steps as in the proof of Lemma \ref{lemma:step2-diff} and using the expansion of $\brho_r(x) / \phi( \brho_r(x), \brho_r(y) )$ in \eqref{eq:phi/rho-taylor}, we have
	\begin{align*}
&		\brho_r(x)^{-2} \epsilon^{-d/2} \int_{\sM} 
		{  k_0 \left(  \frac{\| x - y\|^2}{  \epsilon  \phi( \brho_r(x), \brho_r(y) )^2  } \right)}   (f(y)-f(x)) p(y) \left( \frac{\brho_r(x)}{  \phi( \brho_r(x), \brho_r(y) ) } \right)^2 dV(y) \\
  & = \circled{3}  + O^{[f,p]}(\epsilon^2),
	\end{align*}
	where 
	\begin{align}
		\circled{3} &\coloneqq  \brho_r(x)^{-2}  \epsilon^{-d/2} \int_{ \R^d } 
		\Bigg( \left( k_0 \left(  \frac{\|u\|^2}{ \epsilon \brho_r(x)^2 } \right)	\right) +  \left( -  \|u\|^2 \nabla \tilde\beta_x(0) \cdot u + Q_{x,4}^{(\rho)}(u)  \right) \frac{1}{\epsilon \brho_r(x)^2 } k_0' \left(  \frac{\|u\|^2}{ \epsilon \brho_r(x)^2 } \right)	\notag \\
		&\quad +  {\frac{1}{2}} \frac{\|u\|^4  (\nabla \tilde\beta_x(0) \cdot u)^2 }{\epsilon^2 \brho_r(x)^4 } k_0'' \left(  \frac{\|u\|^2}{ \epsilon \brho_r(x)^2 } \right)  \Bigg) \left(\nabla \tilde{f}(0) \cdot u + \frac{1}{2} u^T \Hess_{\tilde{f}}(0) u +  Q_{x,3}^{(f)}(u)  \right) \notag \\
		&\quad  \left(\tilde{p}(0)  +  \nabla \tilde{p}(0) \cdot u +  Q_{x,2}^{(p)}(u) \right) \left( 1 - \nabla \tilde\beta_x(0) \cdot u + Q_{x,2}^{(\rho)}(u) \right)  \left( 1 + Q_{x,2}^{(V)}(u)\right)  du. \notag
	\end{align}
	We collect the leading terms in $\circled{3}$
	and have that
	\begin{align*}
		\circled{3} = \circled{4}_1 + \circled{4}_2 + \circled{4}_3 + \circled{4}_4 + O^{[f,p]}(\epsilon^2),
	\end{align*} 
	where $\circled{4}_1$, $\circled{4}_2$ and $\circled{4}_3$ have the same definitions 
	as in \eqref{eq:lemma-bias-fast-step4-2}, \eqref{eq:lemma-bias-fast-step4-3}, and \eqref{eq:lemma-bias-fast-step4-4}, and 
	\begin{align*}
		\circled{4}_4  = - \tilde{p}(0) \; \brho_r(x)^{-2}  \epsilon^{-d/2} \int_{ \R^d }   k_0 \left(  \frac{\|u\|^2}{ \epsilon \brho_r(x)^2 } \right)   (\nabla \tilde{f}(0) \cdot u )  (\nabla \tilde\beta_x(0) \cdot u) du.
	\end{align*}
	In Step 4 in the proof of Lemma \ref{lemma:step2-diff}, we have already shown that $\circled{4}_1$, $\circled{4}_2$, and $\circled{4}_3$ satisfy \eqref{eq:lemma-bias-fast-step4-5}, \eqref{eq:lemma-bias-fast-step4-6}, and \eqref{eq:lemma-bias-fast-step4-7}, respectively.
	By similar arguments to those in Step 4 of the proof of Lemma \ref{lemma:step2-diff}, we have
	\begin{align*}
		\circled{4}_4  = \frac{1}{d} \epsilon m_2[k_0]  \nabla f(x) \cdot \frac{\nabla  p(x)}{p(x)} + O^{[f,p]}(\epsilon r).
	\end{align*}
	Inserting these back to \circled{3}, by the definition of $\Delta_p$ in \eqref{eq:def-Delta-p},
	we have
	\begin{align*}
		\circled{3} = \epsilon \frac{m_2[k_0]}{2} \Delta_pf(x) + O^{[f,p]}(\epsilon^2 + \epsilon r).
	\end{align*}
	Then, we have proved \eqref{eq:lemma-step2-other-cases-rw-fast}. 
	The threshold for $\epsilon$ is $\epsilon <  \epsilon_D(k_0, c_{\max})$ (defined in \eqref{eq:def-epsD}), the same as in Lemma \ref{lemma:step2-diff}.
	The constant in the big-$O$ notation of the final error bound depends on $f$ and $p$ and is uniform for $x$ and $r$.
\end{proof}

\subsection{Random-walk graph Laplacian with density correction}\label{app-subsec:tildeLrw}

We consider the random-walk graph Laplacian associated with $\widetilde{W}$ defined as 
\begin{equation}  \label{eq:def-L-rw-tilde-matrix}
 \widetilde{L}_\rw \coloneqq  - \frac{1}{\frac{m_2}{2m_0}} \diag\left(  \left\{ \frac{1}{ \epsilon \hrho(x_i)^2 }  \right\}_i \right) (I  - \widetilde{D}^{-1} \widetilde{W}) \in \R^{N\times N}, 
\end{equation}
 and the operator $\widetilde{L}_\rw$ applied to $f$ is given by
\begin{equation}\label{eq:def-L-rw-tilde}
	 \widetilde{L}_\rw f(x) =  \frac{1}{\frac{m_2}{2m_0}\epsilon\hrho(x)^2} \left(  \frac{  \sum_{j=1}^N k_0 \left(  \frac{\| x - x_j\|^2}{ \epsilon \phi(\hrho(x), \hrho(x_j))^2 } \right) \frac{f(x_j)}{\phi(\hrho(x), \hrho(x_j))^2}  }{ \sum_{j=1}^N k_0 \left(  \frac{\| x - x_j\|^2}{ \epsilon   \phi(\hrho(x), \hrho(x_j))^2   } \right) \frac{1}{\phi(\hrho(x), \hrho(x_j))^2} }  -f(x) \right) , \quad x \in \sM. 
\end{equation}

\section{Technical Lemmas} \label{app:tech-lemmas}

We always assume Assumption \ref{assump:M} regarding the manifold $\sM$.

\subsection{Lemmas used in Section \ref{sec:concen-hrho} }

Theorem \ref{thm:consist-hrho} uses the following two lemmas. 
Lemma \ref{lemma:Lip-cont-hat-R} provides the Lipschitz continuity and regularity of $\hat{R}$ in $\R^m$, which was originally proved in \cite{cheng2022convergence}.
This lemma is used in the proof of Theorem \ref{thm:consist-hrho} to establish the Lipschitz continuity of the bandwidth function $\hrho$ defined in \eqref{eq:def-hat-rho}.
Lemma \ref{lemma:covering} upper bounds the covering number of $\sM$ using Euclidean balls in $\R^m$,
which is a well-known estimate, see, e.g., Lemma A.2  in  \cite{cheng2022convergence}.

\begin{lemma}[Lemma 2.1 in \cite{cheng2022convergence}]
\label{lemma:Lip-cont-hat-R}
Suppose $X$ has distinct data points $\{x_i\}_{i=1}^N$ 
and  $1< k < N$. 
Then $\hat{R}$ defined in 
\eqref{eq:def-hat-R}
is  Lipschitz
continuous on $\R^m$ with  $\text{Lip}_{\R^m}(\hat{R}) \le 1$.
Moreover, $\hat{R}$ is $C^\infty$ on $\R^m \backslash E$, 
where $E$ is a finite union of ($m$-1)-hyperplanes
(finitely many points when $m=1$).
\end{lemma}

\begin{lemma}[Lemma A.2 in \cite{cheng2022convergence}]
\label{lemma:covering}
For any $0< r \leq \delta_1$, 
where $\delta_1$ is introduced in Lemma \ref{lemma:M-metric}(iii), 
${\cal N}( {\sM},  \| \cdot \|_{\R^m}, r ) \leq  V({\sM}) {r^{-d}}$,
where ${\cal N}( {\sM},  \| \cdot \|_{\R^m}, r ) $ is the smallest cardinality of an $r$-net of $\sM$, and $V(\sM)$ equals an $O(1)$ constant times the Riemannian volume of ${\sM}$.
\end{lemma}

\subsection{Differential geometry lemmas} \label{app:diff-geo-lemmas}

Consider the local neighborhood at $x\in \calM$ and the tangent space $T_x \calM$.
We denote the normal coordinate of $y \in \calM$, denoted as $s$, 
and the projected coordinate of $y$, denoted by $u$. 
Let $\varphi_x$ be the orthogonal projection of $y$ to the tangent space $T_x \sM$ (embedded in $\R^m$),
then $u =  \varphi_x(y)$.
Under Assumption \ref{assump:M},  there is  $\delta_0({\sM}) > 0$ such that for any $\delta < \delta_0$ and any $x \in \sM$, 
both $s$ and $u$ are well-defined for $y \in \sM \cap B^m_{\delta}(x)$.
Such $\delta_0$ can be chosen as $\min_x \{  inj(x) , reach (\calM, x)\}/2$, where the minimum is positive by compactness of $\calM$.

The next lemma, Lemma \ref{lemma:M-coord-expansion}, shows the comparison between the normal and projected coordinates,
and it is reproduced from Lemma 6 in \cite{coifman2006diffusion}.

\begin{lemma} [Normal and projected coordinates comparison] \label{lemma:M-coord-expansion}
	When $\delta < \delta_0(\sM)$, for $x\in\sM$ and $y\in B^m_{\delta}(x) \cap {\sM}$ ,  let $s$ and $u$ denote the normal and projected coordinates of $y$, respectively.
	Let $(e_1, \dots, e_d)$ be an orthonormal basis of $T_x \sM$.
	With respect to this basis, $s$ can be expanded as $s= \sum_{i=1}^d s_i e_i$, and the entries of $u$ are given by $u_i = \langle d\iota_x(e_i), \iota(y) - \iota(x) \rangle$.
	Then,
	\begin{equation} \label{eq:coord-comp}
		s_i = u_i +Q_{x,3}^{(N)}(u) + O(\|u\|^4), \quad i = 1,\dots, d,    
	\end{equation}
	where $Q_{x,3}^{(N)}(u)$ denotes a homogeneous polynomial of degree $3$ of $u$, whose coefficients depend on $x$. The superscript $^{(N)}$ stands for ``normal coordinate''.
	
\end{lemma}

The following lemma is adapted from Lemma A.1 in \cite{cheng2022convergence} and Lemma 7 in \cite{coifman2006diffusion}, and it demonstrates the volume and metric comparison on $\sM$.

\begin{lemma}[Volume and metric comparison]
\label{lemma:M-metric}
When $\delta < \delta_0(\sM)$, for $x\in\sM$ and $y\in B^m_{\delta}(x) \cap {\sM}$,  let $s$ and $u$ denote the normal and projected coordinates of $y$, respectively. Then,

(i) The volume comparison:
\begin{equation}\label{eq:volume-comp}
\left| \det\left(\frac{dV(y)}{du}\right) \right| = 1 + Q^{(V)}_{x,2}(u) +  O(\|u\|^3), 
\end{equation}
where $Q_{x,2}^{(V)}(u)$ denotes a homogeneous polynomial of degree $2$ of $u$, whose coefficients depend on $x$. The superscript $^{(V)}$ stands for ``volume''.

(ii)  The metric comparison:
\begin{equation}\label{eq:metric-comp1}
	\| x-y  \|_{\R^m}^2 =  \|u\|^2 + Q_{x,4}^{(D)}(u) + O(\| u \|^5),  
\end{equation}
where $Q_{x,4}^{(D)}(u)$ denotes a homogeneous polynomial of degree $4$ of $u$, whose coefficients depend on $x$. The superscript $^{(D)}$ stands for ``distance''.

(iii) There exists some $0< \delta_1(\sM)  \leq  \delta_0(\sM)$ such that,
when $\delta < \delta_1(\sM)$,
 the following hold for any $x \in \sM$ and $y\in B^m_{\delta}(x) \cap {\sM}$: 
\begin{equation}\label{eq:metric-comp2}
 \| x-y  \|_{\R^m} \leq d_{\sM}(x, y) \leq 1.1 \| x-y  \|_{\R^m}, \quad \|u\| \leq  \| x-y  \|_{\R^m} \leq 1.1 \|u\|, \quad 0.9 \leq \left| \det\left(\frac{dV(y)}{du}\right) \right|  \leq 1.1,
\end{equation}
where $d_{\sM}$ denotes the manifold geodesic distance. 
\end{lemma}

\begin{proof}[Proof of Lemma \ref{lemma:M-metric}]
	The proof of (i)(ii) directly follows from that of Lemma 7 in \cite{coifman2006diffusion}. 
	To prove the first inequality in (iii), we use (i) and Lemma \ref{lemma:M-coord-expansion}, which imply that $d_{\sM}(x, y) =  \| x-y  \|_{\R^m} + O( \| x-y  \|_{\R^m}^3)$. This expansion establishes the first inequality.
	The other two inequalities in (iii) directly follow from (ii) and (i), respectively.
\end{proof}

The following lemma provides the Lipschitz continuity 
of a $C^1$ function $f$ on $\sM$ with respect to the geodesic distance.
The lemma applies when $x,y$ are arbitrary two points on $\calM$, 
while we only use when $y$ lies inside the injectivity radius of $x$ in proving our main result. 
Recall the definitions of $\nabla f$ and $\| \nabla f \|_\infty$ from Definition \ref{def:Dkf-norm}.

\begin{lemma}[Intrinsic Lipschitz continuity of $f$ on $\calM$] \label{lemma:f-diff-bound}
    Under Assumption \ref{assump:M}, for any $f \in C^1(\sM)$, and any two points $x, y \in \sM$, we have that
    \begin{equation}  \label{eq:f-diff-bound}
        \left| f(x) - f(y) \right|  \leq \| \nabla f \|_\infty d_\sM(x, y).
    \end{equation}
\end{lemma}

\begin{proof}[Proof of Lemma \ref{lemma:f-diff-bound}]
Since $\sM$ is compact and connected, 
by Hopf-Rinow Theorem,
 for any two points $x,y \in \sM$, there exists a length minimizing geodesic $\gamma$ joining $x$ and $y$ such that $\gamma(0) = x, \gamma(T) = y, \|\gamma'\| = 1$, where $T = d_\sM(x,y)$. 
	
	The function $f \circ \gamma: [0, T] \to \R$ is $C^1$, and then $f(\gamma(T)) - f(\gamma(0)) = \int_0^T \frac{\rd}{\rd t} f(\gamma(t)) \rd t $.
	For the first-order covariant derivative, we have
 $\frac{\rd}{\rd t} f(\gamma(t)) = \nabla_{\gamma'(t)} f(\gamma(t))$, and then 
\begin{align*}
		&\quad \left| f(x) - f(y) \right| = \left| f(\gamma(T)) - f(\gamma(0)) \right| = \left| \int_0^T \frac{\rd}{\rd t} f(\gamma(t)) \rd t \right|  = \left| \int_0^T \nabla_{\gamma'(t)} f(\gamma(t)) \rd t \right| \\
		&\leq  \int_0^T \left| \nabla_{\gamma'(t)} f(\gamma(t))  \right| \rd t  \leq \int_0^T \| \nabla f \|_\infty \rd t = \| \nabla f \|_\infty T = \| \nabla f \|_\infty d_\sM(x,y).
	\end{align*}
\end{proof}

\subsection{Kernel integral expansion lemmas} \label{app:kernel-expan-lemmas}

\subsubsection{Small-$\epsilon$ threshold}

We define several notations for the thresholds of $\epsilon$, which are used both in Lemmas \ref{lemma:step2-diff}, 
\ref{lemma:step2-other-cases-tildeW}, \ref{lemma:right-operator-degree} and frequently throughout our proofs when we require kernel bandwidth $\epsilon $ to be small.
These thresholds of $\epsilon$ are typically dependent on $\sM$, $p$, and $\phi$,  but are uniform for all $x\in \calM$.

We define the thresholds for $h: [0, \infty) \to [0, \infty) $ under Assumption \ref{assump:k0-smooth}.
The following notations are defined assuming that $\phi$ meets Assumption \ref{assump:phi-diff}(iv).

	We first define $\delta_\epsilon[h]$ as follows,
	\begin{equation}  \label{eq:def-delta-eps}
		\delta_\epsilon[h]  \coloneqq 1.1 c_{\max} \rho_{\max} \sqrt{\frac{10+d}{a[h]}\epsilon\log\frac{1}{\epsilon}},     
	\end{equation}
	where $a[h]$ represents the decay parameter constant such that $|h^{(l)}(\eta)| \leq a_l[h] \exp(-a[h] \eta)$ holds, $c_{\max}$ is introduced in Assumption \ref{assump:phi-diff}(iv), and $\rho_{\max}$ is defined in \eqref{eq:def-rhomin-rhomax}. 
	We will consider $\epsilon$ small enough s.t. $\delta_\epsilon[h] < \delta_1(\calM)$, where $\delta_1(\sM)$ is introduced in Lemma \ref{lemma:M-metric}(iii).
	Because we will use such small-$\epsilon$ thresholds for different values of $a[h]$ and $c_{\rm max}$, 
	we  choose any sufficiently small $\epsilon_D(h, c_{\max})\in(0,e^{-1})$ such that
	\begin{align} \label{eq:def-epsD}
		\forall \epsilon < \epsilon_D(h, c_{\max}), \quad \delta_\epsilon[h] < \delta_1(\sM).
	\end{align}
	Such a choice exists because $\delta_\epsilon[h]\to 0$ as $\epsilon \to 0$.
	$\epsilon_D(\cdot,\cdot)$ depends on $\sM$, $p$, $h$, and $\phi$.
	The subscript $_D$ is to emphasize that $h$ is differentiable here.

\subsubsection{Integrals of the form $\int_\sM K_\epsilon(x,y)f(y) dV(y)$ with fixed bandwidth}
\label{app:kernel-expan-lemmas-fixed-band}

We consider  fixed bandwidth kernels, where $K_\epsilon(x,y)$ takes the form $\epsilon^{-d/2} h( \|x-y\|^2 / \epsilon )$ for some $h : [0, \infty) \to [0, \infty)$. 
Recall the definition of  $\| \nabla^l  f \|_\infty$  in Section \ref{sec:prelim}.

When $h$ is $C^2$ and exponentially decaying, Lemma A.3 in \cite{cheng2022convergence} (also reproduced from Lemma 8 in \cite{coifman2006diffusion}) showed that 
	\[
	\int_{\sM} h\left( \frac{ \| x-y \|^2}{\epsilon} \right) f(y) dV(y)= \epsilon^{d/2} \left( m_0[h] f(x) +  \epsilon \frac{m_2[h]}{2} ( \Delta f(x) + \omega(x) f(x)) 
	+ O^{ [ f^{(\le 4)} ]}(\epsilon^2) \right) ,
	\]
	where $\omega(x)$ is a function involving local derivatives of the extrinsic manifold coordinates at $x$. 
	The residual term denoted by big-$O$ with superscript $f^{(\le 4)}$
	means that the constant 
	involves up to the 4-th derivative of $f$ on ${\sM}$.

The next lemma, Lemma \ref{lemma:G-expansion-h-indicator} derives the same expansion for indicator function $h$. 
It can be deduced from the Lemmas SI.1 and SI.3 in \cite{wu2018think}.

\begin{lemma}
\label{lemma:G-expansion-h-indicator}
Under Assumption \ref{assump:M}, 
let $h(\eta) = \mathbf{1}_{[0,1]}(\eta)$, 
when $\epsilon$ is sufficiently small, 
for any $f\in C^3(\sM)$ and $x \in {\sM}$,
\begin{align} \label{eq:G-expansion-h-indicator}
 \int_{\sM} h \left(  \frac{ \| x- y\|^2 }{\epsilon}\right)  f(y) dV(y) 
&=  \epsilon^{d/2} \left(  m_0[h] f(x) + \epsilon\frac{m_2[h]}{2}(\Delta f(x) + \omega(x) f(x)) +  O^{ [ f^{(\le 3)} ]}( \epsilon^{3/2} )  \right),    
\end{align}
where 
\begin{align} \label{eq:def-omega}
	\omega(x) \coloneqq \frac{1}{12m_2[h]} \int_{S^{d-1}} \| \Pi_x(\theta, \theta) \|^2 d\theta - \frac{1}{3} s(x), 
\end{align}
here, $\Pi_x$ is the second fundamental form at $x$, and $s(x)$ is the scalar curvature at $x$.
The residual term denoted by big-$O$ with superscript $f^{(\le 3)}$
 means that the constant 
involves up to the 3rd derivative of $f$ on ${\sM}$.
Specifically, if the residual term is denoted as $r_{f, \epsilon}(x)$, it satisfies 
$ \sup_{x \in {\sM}} |r_{f, \epsilon}(x)| \le C(f) \epsilon^{3/2}$, where
$C(f) = c({\sM}, h) (1+ \sum_{l=0}^3 \| \nabla^l f \|_\infty)$. 
\end{lemma}

\subsubsection{Integrals of the form $\int_\sM K_\epsilon(x,y) \brho_r(y)^i p(y) dV(y)$ } \label{app:kernel-expan-degree}

Lemma \ref{lemma:right-operator-degree} below is used in our proofs in two ways:

\begin{enumerate}

    \item Analyze the bias error of the independent sums $\bar{D}(x) $ in the convergence analysis of the random-walk graph Laplacians $L_\rw$,
    where $\bar{D}(x)$ is defined as in  \eqref{eq:def-Dbar}. 
The expectation of the sum (the population counterpart of $\bar D$) is
\[\epsilon^{-d/2} \int_{\sM} 
{  k_0 \left(  \frac{\| x - y\|^2}{  \epsilon  \phi( \brho_r(x), \brho_r(y) )^2  } \right)}  p(y)  dV(y),
\]
and Lemma \ref{lemma:right-operator-degree}(i) expands the integral up to $O(\epsilon)$ term.
This bounds the bias error. 
Using Bernstein’s inequality and the upper bound of the variance of the independent random variables, which can also be derived from Lemma \ref{lemma:right-operator-degree}, the variance error can also be bounded.

\item  Provide upper bounds for the expectations and variances of independent sums of random variables.
Specifically, these expectations and variances can be upper bounded by the kernel integrals in the form of
\[\epsilon^{-d/2} \int_{\sM} 
k_0 \left(  \frac{\| x - y\|^2}{  \epsilon  \max\{ \brho_r(x), \brho_r(y) \}^2  } \right)  \brho_r(y)^{i} p(y)  dV(y),\quad  i = -4, -2, -1, \dots,2.\] 
See \eqref{eq:bound-bias-13} and \eqref{eq:bound-var2} in the proof of Proposition \ref{prop:step1-diff} for an example.
\end{enumerate}

\begin{lemma} [Integrals of the form $\int_\sM K(x,y) \brho_r(y)^i p(y) dV(y)$]   \label{lemma:right-operator-degree}
	Under Assumptions \ref{assump:M}-\ref{assump:p}, let $\tilde{r}_0$ be defined in \eqref{eq:def-r0-tilder0}.
	Suppose $p \in C^4(\sM)$.
	
	(i) Suppose $k_0$ and $\phi$ satisfy Assumptions \ref{assump:k0-smooth} and \ref{assump:phi-diff} respectively. When $\epsilon$ is sufficiently small,  for any $r \in [0, \tilde{r}_0]$, any $x \in \sM$, and $ i = -2, -1, \dots, 2$, 
	\begin{align}
		&\epsilon^{-d/2} \int_{\sM} 
		{  k_0 \left(  \frac{\| x - y\|^2}{  \epsilon  \phi (\brho_r(x), \brho_r(y) )^2  } \right)}  \brho_r(y)^{i} p(y)  dV(y) 
		= m_0[k_0]  p \brho_r^{d+i}(x) + O^{[p]}( \epsilon ),   	\label{eq:ker-expan-degree-fast} 		\\ 
		&\epsilon^{-d/2} \int_{\sM} 
		k_0 \left(  \frac{\| x - y\|^2}{  \epsilon  \phi (\brho_r(x), \brho_r(y) )^2  } \right)   \frac{\brho_r(x)^2}{   \phi( \brho_r(x), \brho_r(y) )^2  }  p(y)  dV(y) 
		= m_0[k_0]  p \brho_r^d(x) + O^{[p]}( \epsilon ).
		\label{eq:ker-expan-degree-Wtilde-fast}
	\end{align}	
	The constants in $O^{[p]} (\epsilon)$ depend on
	$\sM, p, k_0, \phi$ and are uniform for $x$, $r$, and $i$. Specifically, they depend on $p$ through $p_{\min}$, $p_{\max}$, $\|\nabla^l p\|_\infty, l \leq 4$, and
	on $\phi$ through $c_{\min}, c_{\max}$, and $C_{\phi, l}$ as in \eqref{eq:def-C-phi-l} for $l \leq 2$.
	The threshold for $\epsilon$ is $\epsilon_D(k_0, c_{\max})$ defined in \eqref{eq:def-epsD}, the same as that in Lemma \ref{lemma:step2-diff}.

	(ii) Suppose $k_0$ satisfies Assumption \ref{assump:k0-smooth} and $\phi(u,v) = \max\{u,v\}$. When $\epsilon$ is sufficiently small,  for any $r \in [0, \tilde{r}_0]$, any $x \in \sM$, and $ i = -4, -2, -1, \dots, 2$, 
	\begin{align}
		&\epsilon^{-d/2} \int_{\sM} 
		{  k_0 \left(  \frac{\| x - y\|^2}{  \epsilon  \phi (\brho_r(x), \brho_r(y) )^2  } \right)}  \brho_r(y)^{i} p(y)  dV(y) 
		= m_0[k_0]  p \brho_r^{d+i}(x) 
		+ 
		{O^{[p]}( \sqrt{\epsilon} )},  
		\label{eq:ker-expan-degree-slowI} \\
		&\epsilon^{-d/2} \int_{\sM} 
		k_0 \left(  \frac{\| x - y\|^2}{  \epsilon  \phi (\brho_r(x), \brho_r(y) )^2  } \right)   \frac{\brho_r(x)^2}{   \phi( \brho_r(x), \brho_r(y) )^2 }  p(y)  dV(y) 
		= m_0[k_0]  p \brho_r^d(x) 
		+ 
		{O^{[p]}( \sqrt{\epsilon} )}.
		\label{eq:ker-expan-degree-Wtilde-slowI}
	\end{align}
	The constants in 
	${O^{[p]} (\sqrt{\epsilon})}$ depend on
	$\sM, p, k_0$ and are uniform for $x$, $r$, and $i$.  Specifically, they depend on $p$ through $p_{\min}$, $p_{\max}$, $\|\nabla^l p\|_\infty, l \leq 4$.
	The threshold for $\epsilon$ is $\epsilon_D(k_0, 1)$ defined in \eqref{eq:def-epsD}.
	
\end{lemma}

\begin{proof}[Proof of Lemma \ref{lemma:right-operator-degree}]

We utilize the same notations as in the proof of Lemma \ref{lemma:step2-diff}, including the tangent and normal coordinates $u$ and $s$, and the functions $\tilde{p}$, $\beta_x$, and $\tilde{\beta}_x$.

\vspace{0.5em}
\noindent $\bullet$
Proof of (i): 
To prove \eqref{eq:ker-expan-degree-fast}, by the definition of $\beta_x$ in \eqref{eq:def-beta}, 
it suffices to prove that for $\epsilon$ sufficiently small, for $i = -2, -1, \dots, 2$,
\begin{align}
	\epsilon^{-d/2} \int_{\sM} 
	{  k_0 \left(  \frac{\| x - y\|^2}{  \epsilon  \phi (\brho_r(x), \brho_r(y) )^2  } \right)}  \beta_x(y)^{i} p(y)  dV(y) 
	&= m_0[k_0]  p \brho_r^d(x) + O^{[p]}( \epsilon ).   	\label{eq:ker-expan-degree-fast-1} 
\end{align}

\noindent\underline{Proof of \eqref{eq:ker-expan-degree-fast-1}: }
The proof follows similarly from the proof of Lemma \ref{lemma:step2-diff}. 
In Step 1, the local truncation to $B^m_{\delta_\epsilon}(x) \cap \sM$ is the same as Step 1 of the proof of Lemma \ref{lemma:step2-diff}. 
In Step 2, we apply Taylor expansion to $p$ up to the second degree, i.e. $p(y) = \tilde{p}(0) + \nabla \tilde{p}(0) \cdot u + O^{[p]}(\|u\|^2)$, 
to $\beta_x^i$ up to the second order, i.e.
$	\beta_x(y)^i = 1 + i \nabla \tilde{\beta}_x(0) \cdot u +  O^{[p]}(\|u\|^2)$.
Moreover, we use the expansion of the quantity inside $k_0$ as given in \eqref{eq:phi-taylor}
and apply Taylor expansion to $k_0$ to the second order to expand the kernel.
Then, by these expansions along with the volume comparison in \eqref{eq:volume-comp}, we can change the variable to $u$ as in the argument used to obtain \eqref{eq:lemma-bias-fast-step2}. 
In Step 3, we drop the residual terms in the expansions using \eqref{eq:bound-integrals-eps} and change the integration domain to $\R^d$ using \eqref{eq:bound-integrals-eps-change-Rd} as in the argument used to obtain \eqref{eq:lemma-bias-fast-step2}, with an error of $O^{[p]}(\epsilon)$. 
Then, in Step 4, we collect the leading term, $m_0[k_0]  p \brho_r^d(x)$, and the error remains $O^{[p]}(\epsilon)$.
The threshold for $\epsilon$ is the same as Lemma \ref{lemma:step2-diff}, which can be found in Step 5 in the proof of Lemma \ref{lemma:step2-diff}.
This finishes the proof of \eqref{eq:ker-expan-degree-fast-1}.

\vspace{0.5em}
\noindent\underline{Proof of \eqref{eq:ker-expan-degree-Wtilde-fast}: }
The proof follows the same steps as the proof of \eqref{eq:ker-expan-degree-fast-1} above, with the only difference being the use of the expansion of $\brho_r(x) /  \phi( \brho_r(x), \brho_r(y) )$ as given in \eqref{eq:phi/rho-taylor}.

\vspace{0.5em}
\noindent $\bullet$
Proof of (ii): 
Again, by definition of $\beta_x$ in \eqref{eq:def-beta}, to prove \eqref{eq:ker-expan-degree-slowI} it suffices to prove, for $i = -4, -2, -1, \dots, 2$,
\begin{align}
	\epsilon^{-d/2} \int_{\sM} 
	{  k_0 \left(  \frac{\| x - y\|^2}{  \epsilon  \phi (\brho_r(x), \brho_r(y) )^2  } \right)}  \beta_x(y)^{i} p(y)  dV(y) 
	&= m_0[k_0]  p \brho_r^d(x) +  O^{[p]}( \sqrt{\epsilon} ) .   
		\label{eq:ker-expan-degree-slowI-1} 
\end{align}
After that we will prove \eqref{eq:ker-expan-degree-Wtilde-slowI}.

\vspace{0.5em}
\noindent\underline{Proof of \eqref{eq:ker-expan-degree-slowI-1}:}
The proof follows similarly from the proof of Lemma \ref{lemma:step2-diff}. 
In Step 1, the local truncation on $\sM$ is the same as Step 1 of the proof of Lemma \ref{lemma:step2-diff}, except that the threshold for $\epsilon$ is 
$
{\epsilon_D(k_0,1)}$.
In Step 2, the expansion of $p$ and $\beta_x^i$ are the same as in the proof for case (i).
Then, combined with the fact that $\phi$ is the maximum function,
\begin{align*}
	&\quad \phi(\brho_r(x), \brho_r(y)) =  \brho_r(x) \max\{ 1, \beta_x(y) \}  \\
	&= \brho_r(x)  \left( \max\{ 1, 1 + \nabla \tilde\beta_x(0) \cdot u  + O^{[p]}(\|u\|^2) \} \right)  \\
	&= \brho_r(x)  \left( 1 + [ \nabla \tilde\beta_x(0) \cdot u]_+  + O^{[p]}(\|u\|^2) \right),
\end{align*}
where $a_+ := \max\{a,0\}$.
Consequently, 
\begin{align} \label{eq:phi/rho-taylor-min}
	\frac{\brho_{r}(x)}{ \phi(\brho_r(x), \brho_r(y) ) } = 1 - [ \nabla \tilde\beta_x(0) \cdot u]_+  + O^{[p]}(\|u\|^2).
\end{align}	Combining \eqref{eq:phi/rho-taylor-min} with \eqref{eq:metric-comp1}, we have
\begin{equation} \label{eq:phi-taylor-min}
	\frac{ \|x-y\|^2 }{ (\phi(\brho_r(x), \brho_r(y)) / \brho_r(x))^2 } = \|u\|^2 - 2  \|u\|^2 [ \nabla \tilde\beta_x(0) \cdot u]_+ + O^{[p]}(\|u\|^4).
\end{equation}
For the expansion of the kernel, we use the expansion of the quantity inside $k_0$ in \eqref{eq:phi-taylor-min} and then still apply Taylor expansion to $k_0$ to the second order.
By these expansions and the volume expansion in \eqref{eq:volume-comp}, we change the variable to $u$ as in the argument used to obtain the change of variable.
{The first-order term generated by \eqref{eq:phi-taylor-min} is proportional to $[\nabla \tilde\beta_x(0)\cdot u]_+$ and does not vanish by symmetry. After the change of variables $u=\sqrt{\epsilon}\,\brho_r(x)v$, its integral is $O^{[p]}(\sqrt{\epsilon})$, uniformly in $x$ and $r\in[0,\tilde r_0]$. The remaining first-order terms that are linear in $u$ vanish by symmetry, while the second-order and residual terms, as well as the error from changing the integration domain to $\R^d$, are $O^{[p]}(\epsilon)$. Thus, collecting the leading term gives $m_0[k_0]p\brho_r^d(x)+O^{[p]}(\sqrt{\epsilon})$.}
The threshold for $\epsilon$ is 
$
{\epsilon_D(k_0,1)}$. 
This proves \eqref{eq:ker-expan-degree-slowI-1}.

\vspace{0.5em}
\noindent\underline{Proof of \eqref{eq:ker-expan-degree-Wtilde-slowI}: }
{The proof follows the same steps as the proof of \eqref{eq:ker-expan-degree-slowI-1} for case (ii), using in addition the expansion of $\brho_r(x) /  \phi( \brho_r(x), \brho_r(y) )$ in \eqref{eq:phi/rho-taylor-min}. Its first-order positive-part term contributes $O^{[p]}(\sqrt{\epsilon})$ after the same change of variables, while all remaining terms are $O^{[p]}(\epsilon)$. This proves \eqref{eq:ker-expan-degree-Wtilde-slowI}.}
\end{proof}

\subsection{Other auxiliary lemmas}

\begin{lemma}[Bernstein inequality]
	Let $\{Y_j\}_{j=1}^N$ be i.i.d. bounded random variables.
	Suppose $\E Y_j = 0$, $|Y_j| \leq L$, and $\E Y_j^2 \leq \mu$ for positive constants $L$ and $\mu$. 
	Then, for any $\tau > 0$,
	\[ 
 \Pr \left( \frac{1}{N} \sum_{j=1}^N Y_j  > \tau \right),  
 \Pr \left( \frac{1}{N} \sum_{j=1}^N Y_j  < -\tau \right) \leq \exp\left( - \frac{N \tau^2}{ 2( \mu + \frac{\tau L}{3}) } \right).   \]
	In particular, when $\tau L \leq 3 \mu$, both the tail probabilities are bounded by $\exp(- \frac{1}{4} \frac{N \tau^2}{\mu})$.
\end{lemma}

\begin{lemma}\label{lemma:c1=1/2}
Suppose $\phi \in C^2(\R_+ \times \R_+)$, $\forall u,v \in \R_+,  \phi(u,v)=\phi(v,u)$, $\phi(u,u) = u$, and $\partial_1 \phi(u,u) = c$ for some constant $c$.
Then $c = 1/2$.
\end{lemma}

\begin{proof}[Proof of Lemma \ref{lemma:c1=1/2}]
Since $\phi$ is symmetric, $\partial_1 \phi(u,u) = \partial_2 \phi(u,u) = c$, $\forall u \in \R_+$. Then for small $h$, we have
  \begin{align*}
  \phi( u, u+ h) 
    & = \phi(u,u) + c h + O(h^2)
    = u + ch + O(h^2), \\
  \phi( u+h, u) 
   & = \phi(u+h,u+h) - c h + O(h^2) 
   = u+h - ch + O(h^2).
  \end{align*}
Since $  \phi( u, u+ h)  = \phi( u+h, u) $, comparing the $O(h)$ term gives that $ c = 1/2$.
\end{proof}

\begin{lemma} \label{lemma:rela-err-uv-0.1}
	
	For $u, v \in \R_+$, if $\frac{|u-v|}{u} \leq 0.1$, then we have
	\begin{align} \label{eq:rela-err-uv-0.1}
		\left| \frac{1}{u^2} - \frac{1}{v^2}  \right| \leq \frac{3}{u^2} \frac{ |u-v| }{ u }. 
	\end{align}
	
\end{lemma}

\begin{proof}[Proof of Lemma \ref{lemma:rela-err-uv-0.1}]
	
	If $\frac{|u-v|}{u} \leq 0.1$, then $v \geq 0.9 u$. Therefore,
	\begin{align*}
		\left|  1 - \frac{u^2}{v^2}  \right| &= \frac{u^2}{v^2} \left|  1 - \frac{v^2}{u^2}  \right| =   \frac{u^2}{v^2} \left|  1 + \frac{v}{u}  \right|   \left|  1 - \frac{v}{u}  \right|  \leq \frac{1}{0.9^2} \left(  \left|  1 - \frac{v}{u}  \right|  + 2  \right)   \left|  1 - \frac{v}{u}  \right|  \\
		&\leq \frac{2.1}{0.9^2}  \left|  1 - \frac{v}{u}  \right|  <  3  \left|  1 - \frac{v}{u}  \right|. 
	\end{align*}
\end{proof}

\begin{lemma} \label{lemma:fast-rate}
	For the bound in \eqref{eq:fast-rate1}, the optimal choices of $\epsilon$ and $k$ with respect to $N$ are $\epsilon \sim N^{-2/(d+6)}$ and $k \sim N^{6/(d+6)}$, respectively. The resulting overall rate is ${O( N^{-2/(d+6)} \sqrt{\log N})}$, omitting the constant dependence in the big-$O$ notation.
\end{lemma}

\begin{proof} [Proof of Lemma \ref{lemma:fast-rate}]
	
	By the definition of $r_k$ in \eqref{eq:def-rk}, $r_k$ and $k$ have a one-to-one correspondence. Therefore, it is equivalent to find the optimal choices of $\epsilon$ and $r_k$ for the bound below:
	\begin{eqnarray} \label{eq:lema-fast-rate-1}
		O\left(\epsilon  +   \frac{1}{\sqrt{\epsilon}} r_k^{3/2}   +   r_k  +   \sqrt{\frac{1}{N \epsilon r_k^{d/2}}} +   \;   \sqrt{\frac{1}{N\epsilon^{d/2+1}}}    \right),
	\end{eqnarray}
	which reformulates \eqref{eq:fast-rate1} in terms of $r_k$ instead of $k$.
	Here, we omit the constant dependence in the big-$O$ notations and the $\sqrt{\log N}$ factors.
	
	Assume $\epsilon \sim N^{-\alpha}, r_k \sim N^{-\beta}$, where $\alpha, \beta > 0$. Then, the error bound in \eqref{eq:lema-fast-rate-1} becomes
	\[ O\left(N^{-\alpha}  +  N^{- \frac{3}{2} \beta + \frac{1}{2} \alpha}   +   N^{-\beta}  +  N^{- \frac{1}{2} + \frac{1}{2} \alpha + \frac{d}{4}\beta } +   N^{- \frac{1}{2} + \frac{d+2}{4}\alpha}   \right), \]
	which equals
	\begin{align} \label{eq:def-R-alpha-beta-fast}
		O(N^{ - R(\alpha, \beta) }), \quad  R(\alpha, \beta) \coloneqq \min  \left\{  \alpha,  \; \frac{3}{2} \beta - \frac{1}{2} \alpha, \; \beta, \;  \frac{1}{2}  - \frac{d}{4}\beta - \frac{1}{2} \alpha , \;  \frac{1}{2} - \frac{d+2}{4}\alpha  \right\}.  
	\end{align}
	Therefore,  finding the optimal choices of $\epsilon$ and $r_k$ is equivalent to finding the values of $\alpha, \beta$ that achieve
	\[ \max_{\alpha,  \beta}  R(\alpha, \beta) = \max_{\alpha,  \beta}  \; \min  \left\{  \alpha,  \; \frac{3}{2} \beta - \frac{1}{2} \alpha, \; \beta, \;  \frac{1}{2} - \frac{d}{4}\beta  - \frac{1}{2} \alpha , \;  \frac{1}{2} - \frac{d+2}{4}\alpha  \right\}.  \]
	{For any $\alpha,\beta>0$,}
	\[
		{R(\alpha,\beta)
		\leq \min\left\{\alpha,\frac{1}{2}-\frac{d+2}{4}\alpha\right\}
		\leq \frac{2}{d+6}.}
	\]
	{The last upper bound is attained only when $\alpha=2/(d+6)$. Taking
	$\alpha=\beta=2/(d+6)$, all five quantities in the minimum defining
	$R(\alpha,\beta)$ equal $2/(d+6)$, so the upper bound is attained.
	Moreover, if $R(\alpha,\beta)=2/(d+6)$, the first and fifth terms force
	$\alpha=2/(d+6)$; with this value of $\alpha$, the second and fourth terms
	force $\beta\geq 2/(d+6)$ and $\beta\leq 2/(d+6)$, respectively.
	Thus the maximizer is unique.}
	Namely, the optimal choices of $\epsilon$ and $r_k$ in terms of the order of $N$ for the error bound in \eqref{eq:fast-rate1} are $\epsilon \sim N^{-2/(d+6)}$ and  $r_k \sim N^{-2/(d+6)}$ (with the corresponding choice of $k$ being $k \sim N^{6/(d+6)}$).
\end{proof}

\section{Additional experimental details}

\subsection{Simulated 1D data} \label{app:exp-data}

\begin{figure}
\centering
\subfigure[]{\includegraphics[width=0.24\textwidth]{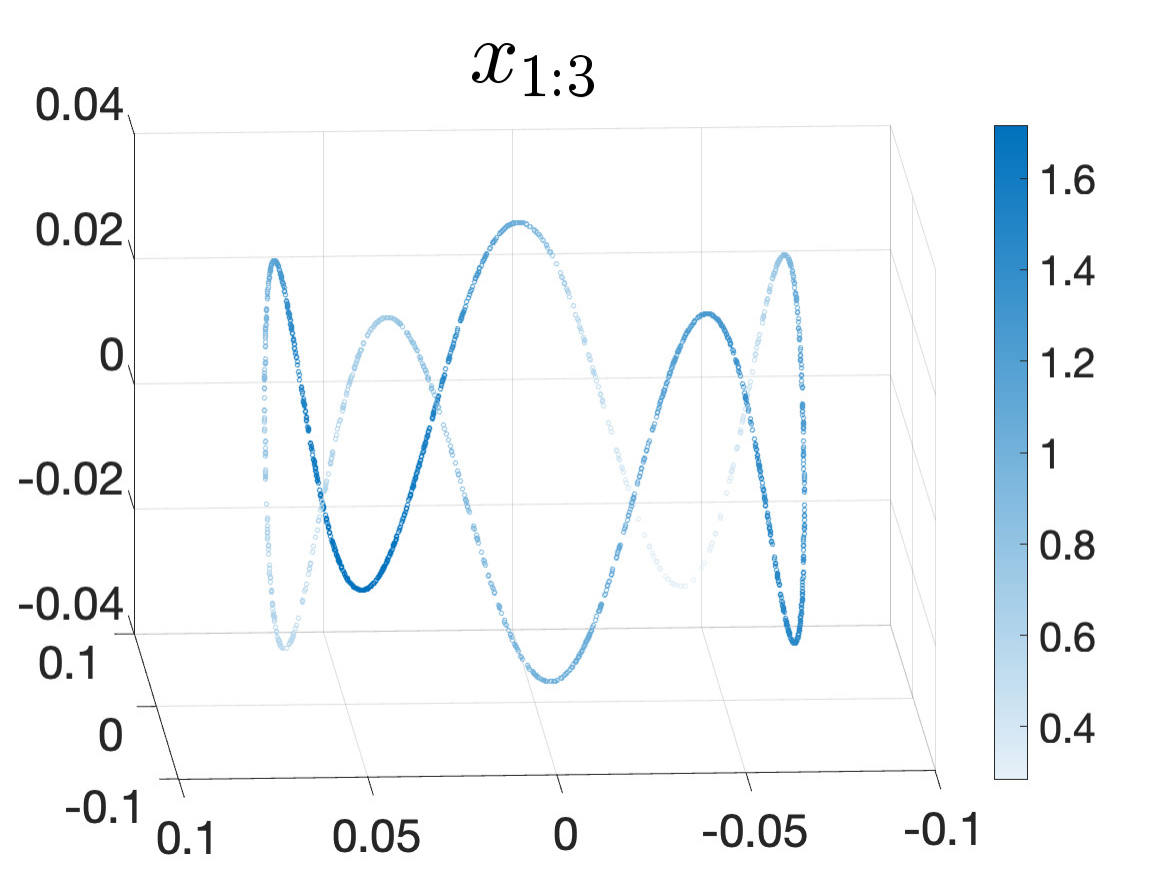}} 
\subfigure[]{\includegraphics[width=0.24\textwidth]{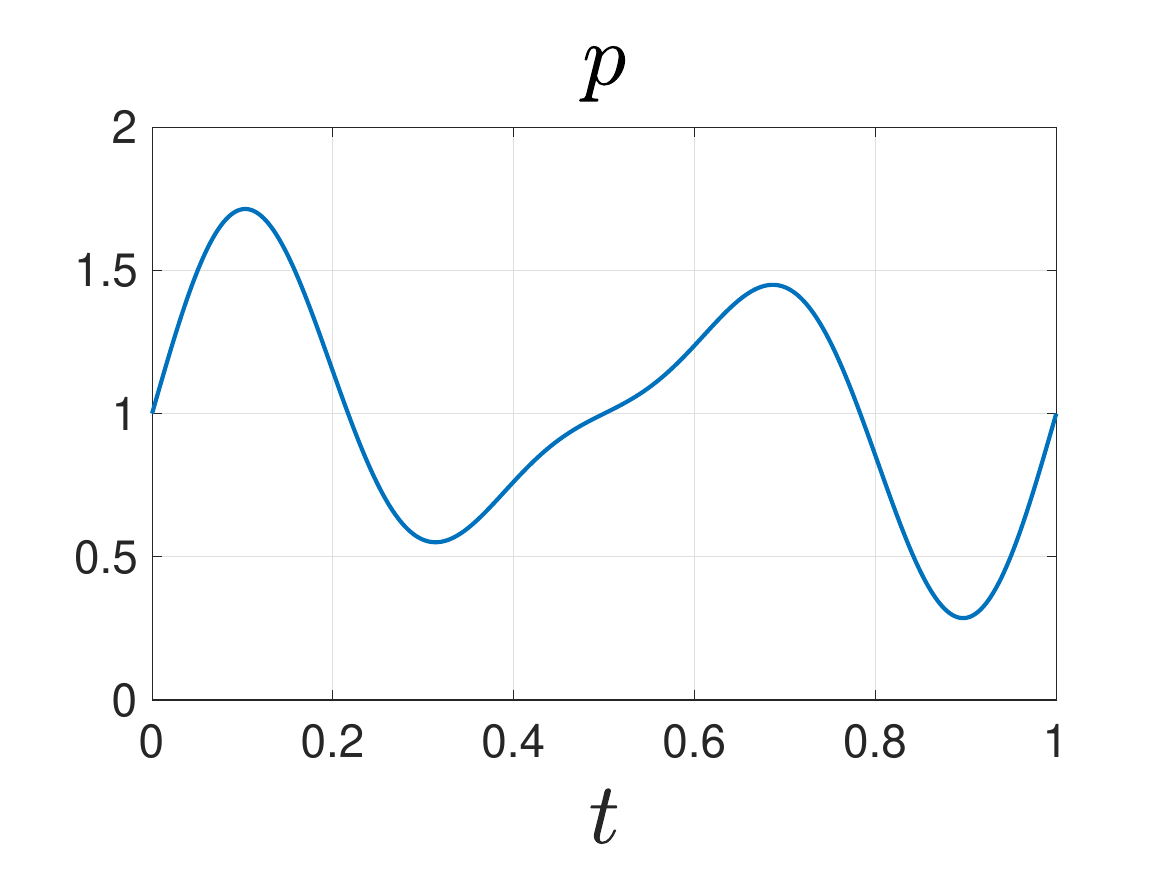}} 
\subfigure[]{\includegraphics[width=0.24\textwidth]{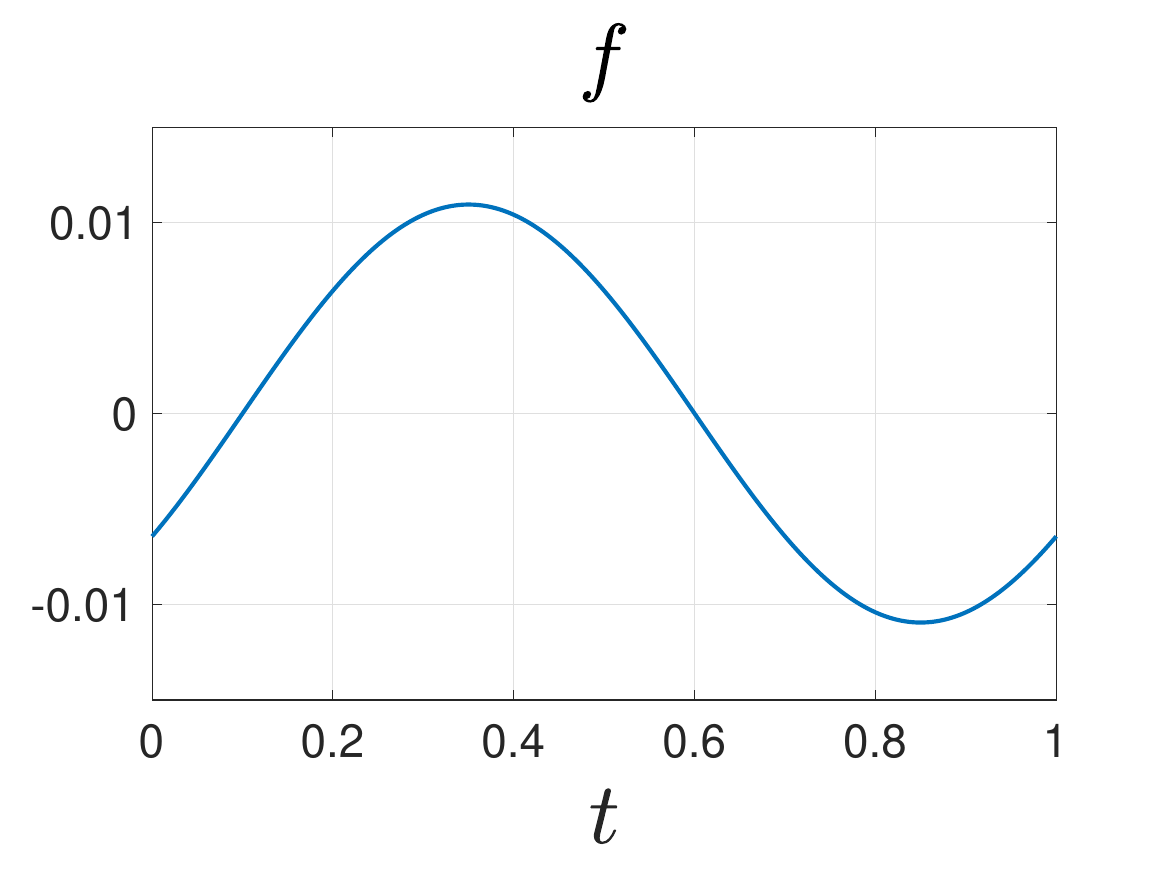}}
\subfigure[]{\includegraphics[width=0.24\textwidth]{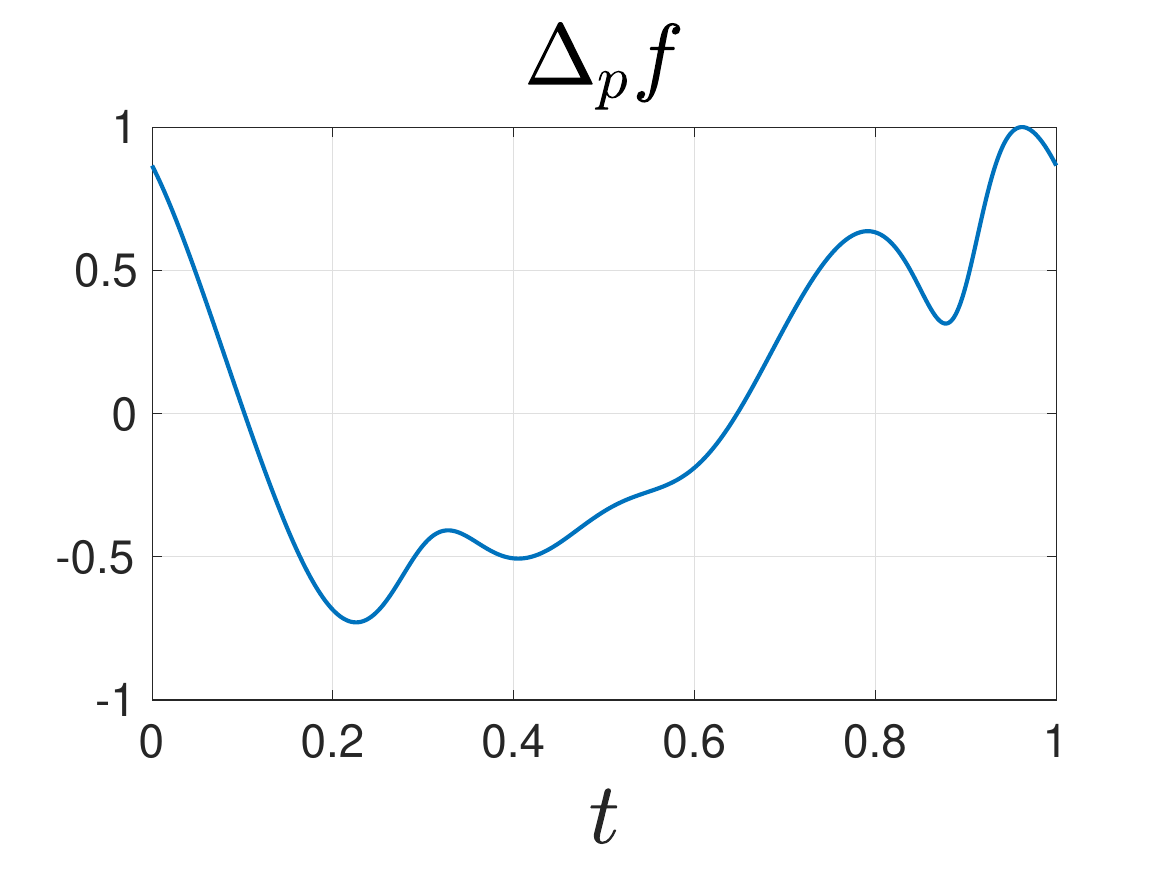}}
\vspace{-5pt}
\caption{
	\small 
	The simulated data on a closed curve in $\R^4$.
	(a) The first three coordinates of 2,000 samples, where the color depth represents the density function $p$.
	(b) The density function $p$.
	(c) Test function $f$. 
	(d) The function $\Delta_p f$ defined as in \eqref{eq:def-Delta-p}. 
	In (b)(c)(d), the functions are plotted against the intrinsic coordinate $ t$ of the curve.
}
\label{fig:data-1d}
\end{figure}

We give details of the simulated 1D manifold data used in Sections \ref{sec:nume-exp-knn} and \ref{sec:nume-exp-GL}. 

\paragraph{Manifold and density.}
The data samples, illustrated in Figure \ref{fig:data-1d}(a), are drawn from a distribution $p$ on a (closed) one-dimensional curve embedded in $\R^4$.
The curve can be parametrized by its intrinsic coordinate $t$, which is the arclength.
Specifically, for $t \in [0,1]$, the embedding mapping is given by 
\begin{equation} \label{eq:1d-embedding}
	\iota(t)  = \frac{1}{2\pi \sqrt{5}} 
	\left(  \cos(2\pi t),  \sin(2\pi t),   
    \frac{2}{5} \cos(10\pi t),    
    \frac{2}{5} \sin(10\pi t)   \right).
\end{equation}
The normalizing constant $\frac{1}{2\pi \sqrt{5}} $ is included such that the embedding $\iota: [0,1] \to \R^4$ is isometric. 

The data density $p$, expressed as a function of the intrinsic coordinate, is specified as 
$
p(t)  =   1  +  \frac{1}{2}  \sin(4\pi  t) +  \frac{1}{4} \sin(6\pi t)$.
The density function $p$ varies from $0.28$ to $1.72$ on the curve, see Figure \ref{fig:data-1d}(b).

\paragraph{Test function $f$ and $\Delta_p f$.}
We choose the test function to be
$f(t) =  \sin(2\pi (t - 0.1))$.
The value of $\Delta_p f$ can be analytically computed according to \eqref{eq:def-Delta-p}, and can be expressed in the intrinsic coordinate $t$ as 
$\Delta_p f(t) = f''(t) + f'(t) p'(t) / p(t)$.
The function $\Delta_p f$ is plotted in Figure \ref{fig:data-1d}(d).

\subsection{Computation of $\brho_{r}$ on 1D data}  \label{app:comp-brho-r}

Recall that $\brho_{r}$ is defined as the solution to a polynomial equation \eqref{eq:def-bar-rho-epsilon}, which involves the correction function $Q$ on the manifold.
In the 1D manifold example given in Appendix \ref{app:exp-data}, the expression of $p$ is already provided, with everything parametrized using the intrinsic coordinate. To obtain the expression of $Q$, it remains to identify the function $\omega$, which is specified in Lemma \ref{lemma:G-expansion-h-indicator}.

We will show that for the embedding given in \eqref{eq:1d-embedding}, 
the function $\omega(x)$ appearing in the definition of $\brho_r$ in \eqref{eq:def-bar-rho-epsilon} and defined in Lemma \ref{lemma:G-expansion-h-indicator} equals a constant $\frac{ 101 \pi^2}{5}$ for any $x \in [0,1]$, where $x$ is the intrinsic coordinate on the 1D closed curve.

To derive the explicit expression of $\omega$, we utilize lemma \ref{lemma:G-expansion-h-indicator} with $d=1$ and $\iota$ given in \eqref{eq:1d-embedding}.
Specifically, letting $f=1$, then Lemma \ref{lemma:G-expansion-h-indicator} gives 
\begin{equation} \label{eq:brho-compute1}
	\epsilon^{-1/2}  \int_0^1 h \left(  \frac{ \| \iota(x)- \iota(y)\|^2 }{\epsilon}\right)  dy
	=  m_0[h] + \epsilon\frac{m_2[h]}{2} \omega(x) +  O( \epsilon^{3/2} )  ,  
\end{equation}
where $m_0[h] = 2$, $m_2[h] = \frac{2}{3}$. 
Here, $x,y \in [0,1]$ are intrinsic coordinates on the 1D closed curve. Since $f \equiv 1$, the constant in the big-$O$ notation only depends on $\sM$.

We use the explicit expression of $\iota$ to expand the l.h.s. of \eqref{eq:brho-compute1} in terms of $\epsilon$. By comparing the $O(\epsilon)$ term on both sides of \eqref{eq:brho-compute1}, the expression of $\omega(x)$ can be derived.

By the expression of $\iota$ in \eqref{eq:1d-embedding} and the Taylor expansion $\iota(x+r) = \iota(x) + r \iota'(x) + \frac{r^2}{2} \iota''(x) + \frac{r^3}{6} \iota'''(x) + O(r^4)$, we have that 
\[ \|\iota(x+r) - \iota(x) \|^2 = r^2 +  r^4 \left(  \frac{1}{4} \| \iota''(x) \|^2 + \frac{1}{3} \iota'(x) \cdot \iota'''(x)  \right) + O(r^5) =r^2 - \frac{101\pi^2}{15}  r^4 + O(r^5), \]
which further implies
\[  \|\iota(x+r) - \iota(x) \| = r - \frac{101\pi^2}{30}  r^3 + O(r^4).\]
Therefore, 
\[ \{  y \in [0,1] \mid  \|\iota(x) - \iota(y)\| \leq \sqrt{\epsilon} \} =  \{ y \in  [ x - r_-, x + r_+] \mod 1\}, \]
where 
\[ r_- = \sqrt{\epsilon} +   \frac{101\pi^2}{30}  \epsilon^{3/2} + O(\epsilon^2), \quad   r_+ = \sqrt{\epsilon} +   \frac{101\pi^2}{30}  \epsilon^{3/2} + O(\epsilon^2). \]
Thus, the l.h.s. of \eqref{eq:brho-compute1} equals
\begin{align*}
	\epsilon^{-1/2}\int_0^1 h \left(  \frac{ \| \iota(x)- \iota(y)\|^2 }{\epsilon}\right)  dy = \epsilon^{-1/2} (r_+ + r_-) = 2 + \epsilon  \; \frac{1}{3} \; \frac{ 101 \pi^2}{5} + O(\epsilon^{3/2}).
\end{align*}
Comparing the $O(\epsilon)$ terms, we conclude that $\omega(x) = \frac{ 101 \pi^2}{5}$ holds for any $x \in [0,1]$.

\begin{figure}[tb]
	\centering{
		\includegraphics[width=0.9\linewidth]{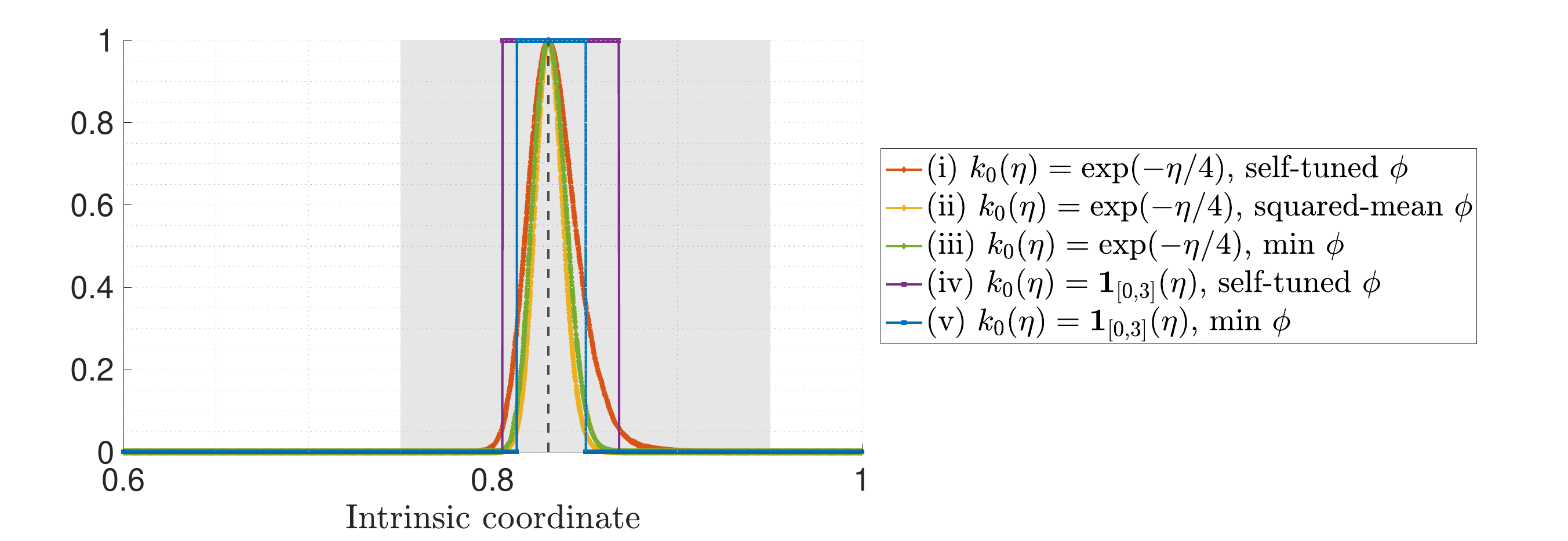}
	}
	\vspace{-5pt}
	\caption{
		\small
		Values of $k_0( \|x_0 - x_j \|^2 / ( \sigma_0^2 \phi(\hat{R}(x_0), \hat{R}(x_j))^2  ) )$ plotted against $x_j$ in its intrinsic coordinate on $[0,1]$ for a fixed $x_0$.
		The local neighborhood around $x_0$ used in the experiments is colored in grey.
		For each of the five types of affinities, 
		we use the value of $\sigma_0$ that achieves the minimum Err in Figure \ref{fig:converg-rates}.
	}
	\label{fig:kernel}
\end{figure}

\subsection{Computation on local neighborhood around $x_0$} \label{app:choice-neighhood}

\paragraph{Choice of the local neighborhood in practice.}

For the fixed $x_0 \in \sM$, where $\sM$ is the 1D closed curve embedded in $\R^4$ as given in \eqref{eq:1d-embedding},  the local neighborhood is selected so that, for all tested $\sigma_0$ and $k$, the kernel $k_0( \| \iota(x_0) - \iota(y) \|^2 / ( \sigma_0^2 \phi(\hat R(x_0), \hat R(y))^2 )  )$ remains below a pre-chosen threshold for $y$ outside this neighborhood. 
In the experiments, we set this threshold to be $10^{-3}$.
After choosing the neighborhood, the graph Laplacian $\widetilde{L}_\rw f(x_0)$ is computed only using the data points within this neighborhood, as described below. In practice, we use approximately $4,800$ points within this neighborhood.

\paragraph{Computation of $\widetilde{L}_\rw f$ on the local neighborhood.}

In Section \ref{sec:nume-exp-GL}, we compute $\widetilde{L}_\rw f(x_0)$ using only data points within this local neighborhood around $x_0$. Specifically, in each run, we draw data points from the distribution $p$ described in Appendix \ref{app:exp-data} within this neighborhood and denote these points as $\{x_i\}_{i=1}^N$.
Then, 
to compute $\widetilde{L}_\rw f(x_0)$, we first construct the vector $\widetilde{W}_{\mathrm{vec}} \in \R^N$ as follows (where the pair $(k_0,\phi)$ is one of the five types (i)-(v) listed in Section \ref{sec:nume-exp-GL}):
\[  (\widetilde{W}_{\mathrm{vec}} )_j = k_0 \left(  \frac{\| \iota(x_0) - \iota(x_j)\|^2 }{ \sigma_0^2 {\phi( \hat R(x_0), \hat R(x_j)  )^2} }  \right)  \frac{1}{\sigma_0^2 \phi( \hat R(x_0), \hat R(x_j) )^2}. \]
Then, by the definition of $\widetilde{L}_\rw f(x_0)$ in \eqref{eq:def-L-rw-tilde}, it can be computed as
\[\widetilde{L}_\rw f(x_0) = \frac{1}{ \frac{m_2}{2m_0} \sigma_0^2 \hat R(x_0)^2 } \left(  \frac{1}{ \widetilde{W}_{\mathrm{vec}}^T \mathbf{1}_N}  \widetilde{W}_{\mathrm{vec}}^T \rho_X f - f(x_0) \right), \]
where $\mathbf{1}_N = (1,1,\dots,1) \in \R^N$ and $\rho_X f \in \R^N$ is as in \eqref{eq:def-rhoX}.

\end{appendices}

\end{document}